\documentclass{article}
\usepackage{iclr2027_conference,times}
\iclrfinalcopy
\usepackage{amsmath,amssymb,amsthm}
\usepackage{graphicx,booktabs,longtable,array}
\usepackage{microtype,enumitem,placeins,needspace}
\usepackage{hyperref,url}
\hypersetup{hidelinks,pdftitle={Why Constrained Updates Still Forget: A Functional-Path Analysis of Continual SFT},pdfauthor={Anonymous authors}}
\newcommand{\method}{\textsc{RetainSFT}}
\newcommand{\R}{\mathbb R}

\newtheorem{proposition}{Proposition}
\newtheorem{corollary}{Corollary}

\title{Why Constrained Updates Still Forget:\\
A Functional-Path Analysis of Continual SFT}
\author{
  Qingsen Ma \\
  The Chinese University of Hong Kong \\
  \texttt{qings.ma@gmail.com} \\
  \And
  Dianyun Wang \\
  Beijing University of Posts and Telecommunications \\
  \texttt{dianyunwang@bupt.edu.cn} \\
  \AND
  Jiahong Liu \\
  The Chinese University of Hong Kong \\
  \texttt{jiahong.liu21@gmail.com} \\
  \And
  Irwin King \\
  The Chinese University of Hong Kong \\
  \texttt{king@cse.cuhk.edu.hk}
}
\begin{document}
\maketitle
\begin{abstract}
Supervised fine-tuning adapts language models to specialized tasks but can disrupt previously acquired behavior. Existing approaches constrain parameters, gradients, or replay objectives, yet the connection between the protected object, the intervening trajectory, and retained outputs remains unclear. We develop a functional-path account that decomposes old-loss changes along realized updates into the controlled direction, gradient staleness, probe mismatch, and finite-step effects. The account yields a conditional refresh budget and identifies the restoring direction supplied by informative replay. Matched QLoRA interventions separate constraint granularity from placement before or after Adam: placement strongly affects retention across the studied code-label audits, whereas the granularity interaction depends on the endpoint. Checkpointed generation separates endpoint recovery from higher sampled-path accuracy, and matched replay schedules reveal loss exposure concealed by final scores. A sequential-task comparison extends this distinction to backward transfer: subspace regularization and projection improve average retention while trading off different old tasks; supervised replay protects both more consistently in this setting. Additional response interfaces distinguish content gains from exact-response retention. These findings connect update-level protection to task-wise behavior and motivate mitigation through explicit functional anchors, trajectory evaluation, and joint old/new auditing.
\end{abstract}

\section{Introduction}
\label{sec:intro}

Supervised fine-tuning (SFT) adapts language models to new tasks, but improvements on the fine-tuning distribution can come at the expense of behaviors the model previously performed correctly \citep{kotha2024understanding,zheng2025spurious}. This problem also arises in parameter-efficient adaptation: LoRA and QLoRA leave the pretrained weights frozen while changing the outputs of the adapted model \citep{hu2022lora,dettmers2023qlora}. Existing approaches seek to reduce forgetting by restricting parameter changes, constraining gradients, or revisiting old examples \citep{kirkpatrick2017overcoming,lopez2017gradient,chaudhry2018efficient,buzzega2020dark,pan2020continual}. These approaches protect different quantities during optimization. Understanding their effects therefore requires connecting what an update controls to which old behaviors remain available as training proceeds.

This study investigates this connection through controlled QLoRA adaptation experiments targeting behaviors verified prior to fine-tuning. Our initial observations revealed that severe forgetting can occur even when established update constraints are satisfied. In code-label adaptation scenarios, projecting actual optimizer updates onto the gradient of the "old loss" satisfies preset constraints, yet previously correct responses often disappear in large numbers. We analyze this discrepancy through a "functional-path account," precisely decomposing the change in old loss—induced by actual parameter updates—into a cached constraint term, probe mismatch, gradient staleness, and a finite-step residual. This decomposition distinguishes between quantities directly controlled by projection and changes accumulated during training. By deriving a conditional refresh-length bound and constructing cases that reach this limit, we reveal how staleness in constraints and finite-step updates drive increases in old loss. To explore which design choices alter this outcome, we independently varied constraint granularity and the point of optimizer application. Keeping probes, refresh intervals, solvers, and training configurations constant, we compared a single aggregated constraint against independent constraints for specific old-behavior categories, and examined applying constraints either to pre-Adam gradients or post-Adam parameter updates. In code-related experiments, regardless of the response evaluation metric, projecting onto pre-Adam gradients preserved more old-behavior responses than projecting onto post-Adam updates.

We also examined when old behaviors are lost and when they are recovered. Analysis of generated checkpoints revealed that methods achieving similar final accuracy can follow vastly different behavioral trajectories: one might undergo an early performance drop followed by recovery, while another maintains consistently high accuracy across observed checkpoints. Replay offers a complementary intervention for these trajectories by repeatedly presenting old-task objectives during adaptation. A theoretical result regarding conditional descent elucidates how informative replay counteracts increases in old-task loss. Empirical results demonstrate that supervised replay effectively preserves old-behavior responses during code-label adaptation. Keeping the replay samples, sample order, and total budget constant, we further discovered that even when final generation accuracy is comparable, delaying the presentation of replay samples causes the loss on old tasks to exceed baseline levels.

Finally, we extended our evaluation beyond code-label adaptation to include natural language sentiment analysis, structured information extraction, and a continual learning scenario encompassing three tasks: news classification, sentiment classification, and information extraction. These distinct response interfaces differentiate between "preserving answer content" and "precisely reproducing the required response." Sequential experiments directly measured the extent of forgetting regarding newly acquired tasks using metrics for backward transfer and per-task performance. Compared to standard supervised fine-tuning (SFT), CLoRA and OPLoRA improved average retention rates, though the distribution of gains varied; each method prioritized the retention of different early-stage tasks. In contrast, supervised replay improved retention across all early tasks in the sequence. These results establish a link between local protection mechanisms and final behavioral performance, while also revealing why aggregate metrics alone fail to accurately reflect which acquired behaviors are preserved during subsequent adaptation.We develop a functional-path understanding of catastrophic forgetting in SFT, explaining why constrained updates still forget and how replay protects previously learned behavior throughout adaptation.
\section{Problem Setting: What Forgets during SFT?}
\label{sec:problem}

We study \emph{retention of acquired behavior during new-task adaptation}. Starting from a pretrained checkpoint, we establish Old performance, apply New-task SFT, and re-evaluate the same Old prompts and response interface. Old replay comes from a separate training split. For accuracy $A$, forgetting and mitigation benefit are
\begin{equation}
 F_m=A_{\rm Old}(\theta_0)-A_{\rm Old}(\theta_T^m),\qquad
 B_m=F_{\rm plain}-F_m.
 \label{eq:forgetting-setting}
\end{equation}
Positive $F_m$ measures lost availability; positive $B_m$ measures mitigation, separately from an absolute retention budget and the New-task score. For example, in the matched Mistral Code study, Legacy accuracy falls from 100\% to 5.43\% under plain SFT, versus 97.83\% with replay; New content accuracy reaches 92.40\% and 93.80\%, respectively (Table~\ref{tab:structure-location}). Independent strict-response banks start at 100\% and finish at 3.65\% and 73.96\%.

We also study sequentially acquired tasks. With $R_{ij}$ denoting task-$j$ accuracy after learning task $i$, backward transfer (BWT), $\frac{1}{K-1}\sum_{j<K}(R_{Kj}-R_{jj})$, compares final with just-learned scores; higher is better. Our three-task continuation supplies this matrix (Table~\ref{tab:continual-bridge}; Appendix~\ref{app:continual-bridge}). Independent adaptations isolate update design and behavioral paths; sequential boundaries reveal which newly learned tasks subsequently regress. 

\section{Functional-Path Analysis}
\label{sec:method}
\label{sec:theory}

The functional-path analysis connects the applied update, the old-loss trajectory, and generated behavior (Figure~\ref{fig:functional-path}).

\begin{figure}[t]
\centering
\includegraphics[width=0.85\linewidth]{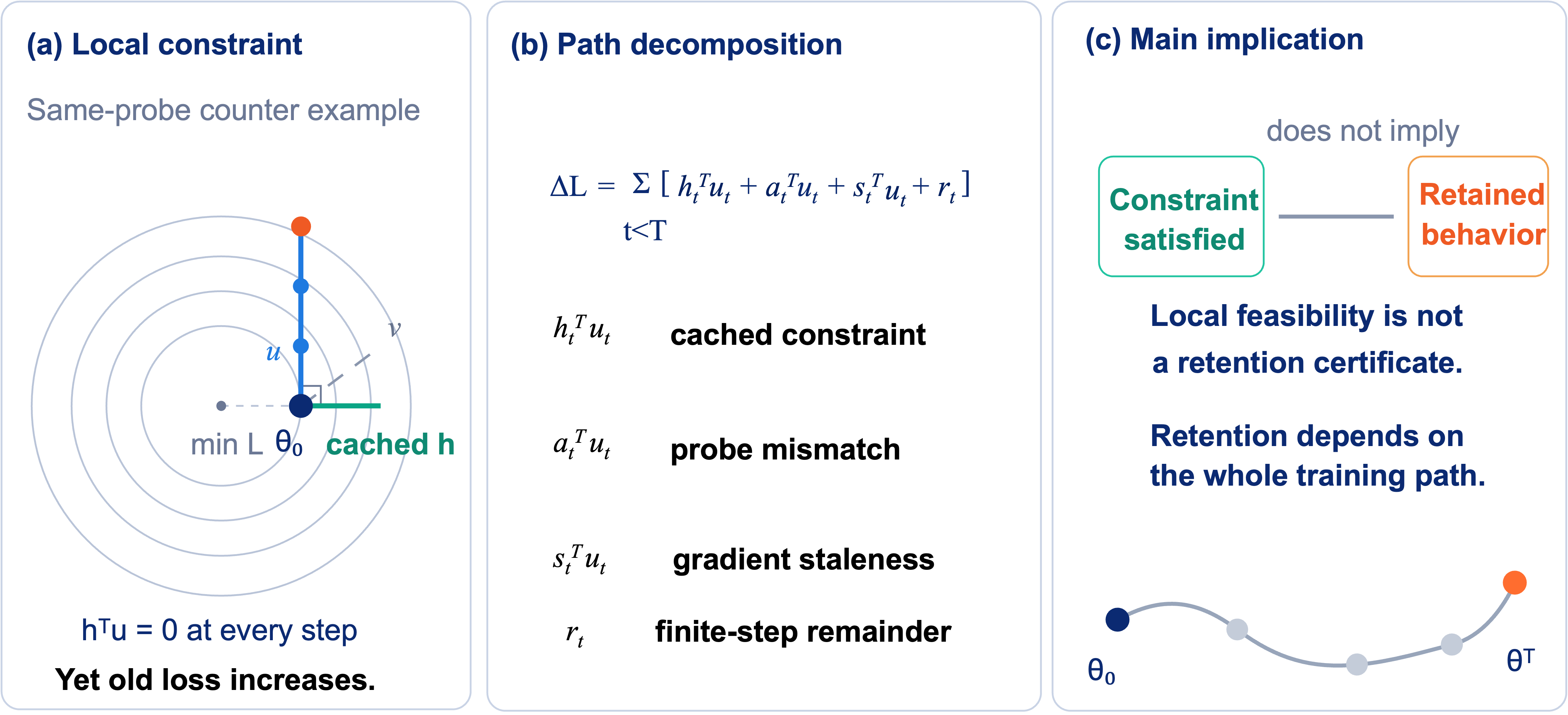}
\caption{\textbf{From local feasibility to functional retention.}
(a) A same-probe quadratic ($L=Q$) has $h^\top u_t=0$ but increasing old loss.
(b) Realized loss change decomposes into four terms (Proposition~\ref{prop:path}); the halfspace controls only the cached term. Orange terms have unrestricted signs.
(c) Loss paths and generated behavior need separate audits. Appendix~\ref{app:proofs} supplies the quadratic and a cross-entropy counterexample that flips a correct greedy decision.}
\label{fig:functional-path}
\end{figure}

\subsection{The retained object and the realized update}

Let $L(\theta)$ be a fixed old evaluation loss, $Q(\theta)$ a finite control-probe loss, and $L_1(\theta)$ the new objective. Old behavior should be demonstrated before adaptation and assessed through the same frozen interface; recovery under other prompts or relearning is a separate question. In QLoRA, $\theta$ denotes adapter factors, not the full effective weights.

Write $u_t=\theta_{t+1}-\theta_t$ for the \emph{realized} update, including optimizer state, clipping, weight decay, and any post-step projection. For $g_t=\nabla L(\theta_t)$, Taylor's theorem gives
\begin{equation}
 L(\theta_T)-L(\theta_0)=\sum_{t<T}g_t^\top u_t+\sum_{t<T}r_t,
 \qquad
 r_t=\int_0^1(1-s)u_t^\top\nabla^2L(\theta_t+s u_t)u_t\,ds.
 \label{eq:path}
\end{equation}
For twice continuously differentiable $L$, the remainder is the displayed Hessian integral. Raw gradient cosine omits update magnitude, optimizer transformation, and subsequent states; final parameter distance also leaves the path integral unspecified.

\subsection{A cached halfspace controls one term}

Suppose the old probe gradient was last refreshed at state $\theta_{\tau(t)}$, and let $h_t=\nabla Q(\theta_{\tau(t)})$. An ideal Euclidean projection maps a candidate update $v_t$ to
\begin{equation}
 u_t=v_t-\frac{[h_t^\top v_t]_+}{\|h_t\|^2}h_t,
 \qquad h_t\ne0,
 \label{eq:project}
\end{equation}
so that $h_t^\top u_t\leq0$. The parameter state used to compute $h_t$ matters even if the probe examples never change.

\Needspace{10\baselineskip}
\begin{proposition}[Functional-path decomposition]
\label{prop:path}
Let $L,Q$ be twice continuously differentiable on an open convex neighborhood of the trajectory. Define $a_t=\nabla L(\theta_t)-\nabla Q(\theta_t)$ and $s_t=\nabla Q(\theta_t)-h_t$. Then
\begin{equation}
 L(\theta_T)-L(\theta_0)=\sum_{t<T}
 \left[\underbrace{h_t^\top u_t}_{\text{cached constraint}}
 +\underbrace{a_t^\top u_t}_{\text{probe mismatch}}
 +\underbrace{s_t^\top u_t}_{\text{staleness}}+\underbrace{r_t}_{\substack{\text{finite-step}\\\text{remainder}}}\right].
 \label{eq:decomp}
\end{equation}
If $\nabla L,\nabla Q$ are Lipschitz with constants $\beta_L,\beta_Q$ and $h_t^\top u_t\leq\varepsilon_t$, then
\begin{equation}
 L(\theta_T)-L(\theta_0)\leq\sum_{t<T}\left[
 \varepsilon_t+\|a_t\|\|u_t\|
 +\beta_Q\|\theta_t-\theta_{\tau(t)}\|\|u_t\|
 +\frac{\beta_L}{2}\|u_t\|^2\right].
 \label{eq:bound}
\end{equation}
\end{proposition}

Adding and subtracting current and cached probe gradients proves the identity (all proofs: Appendix~\ref{app:proofs}). Taking $L=Q$ removes mismatch. Equation~\ref{eq:bound} separates the quantities that determine a conditional pathwise loss budget. The measurement identity in Appendix~\ref{app:measurement} makes explicit the numerical errors included in the logged remainder.

\begin{corollary}[Refresh-length budget]
\label{cor:refresh}
Under the hypotheses of Proposition~\ref{prop:path}, suppose refresh blocks contain at most $r$ updates, $\tau(t)$ is the first state of the block containing update $t$, and $\|a_t\|\leq D$. Then
\begin{equation}
 L(\theta_T)-L(\theta_0)\leq
 \sum_{t<T}\varepsilon_t+D\sum_{t<T}\|u_t\|
 +\frac{\beta_L+(r-1)\beta_Q}{2}\sum_{t<T}\|u_t\|^2.
 \label{eq:block}
\end{equation}
For identical probe and target loss with smoothness $\beta$ and exact halfspaces, this reduces to $\beta r\sum_{t<T}\|u_t\|^2/2$.
\end{corollary}

Refreshing a probe changes its cached state; enlarging it changes the protected objective. We test both in the refresh--coverage factorial and isolate one-step refresh by holding the state and candidate update fixed. 

\paragraph{A tight counterexample.}
Set $L=Q=\|\theta\|^2/2$ in $\R^2$. At the start of a block cache $h=\theta_{kr}\ne0$. Let $J$ rotate by $90$ degrees, and repeatedly propose $v=bh+aJh$ with $a,b>0$. Projection is active and gives $u=aJh$ at every step, so every cached dot is exactly zero. At cache age $j$, however,
\begin{equation}
 \nabla L(\theta_{kr+j})^\top u=j a^2\|h\|^2,
 \qquad L(\theta_{(k+1)r})=(1+r^2a^2)L(\theta_{kr}).
 \label{eq:counter}
\end{equation}
Loss grows geometrically and attains Corollary~\ref{cor:refresh}'s same-probe bound. The construction tests the halfspace property, not experimental AdamW dynamics. The appendix's cross-entropy example flips a correct greedy decision while every cached dot remains zero.

\subsection{Why informative replay can supply a restoring direction}

Replay repeatedly evaluates an old-target objective at current parameters. For an interpretable local comparison, hold the state and a symmetric positive-definite preconditioner $P$ fixed. Let $g_0=\nabla L(\theta)\ne0$, $g_1=\nabla L_1(\theta)$, and $\lambda\geq0$, and compare $u_0=-\eta Pg_1$ to $u_\lambda=-\eta P(g_1+\lambda g_0)$. Their first-order old-loss difference is
\begin{equation}
 g_0^\top(u_\lambda-u_0)=-\eta\lambda g_0^\top Pg_0.
 \label{eq:replaydirection}
\end{equation}

\begin{proposition}[Conditional replay descent]
\label{prop:replay}
Let $L$ have $\beta$-Lipschitz gradient, $\beta>0$, on the update segment, and define $A_\lambda=g_0^\top Pg_1+\lambda g_0^\top Pg_0$. If $A_\lambda>0$ and
\begin{equation}
 0<\eta<\frac{2A_\lambda}{\beta\|P(g_1+\lambda g_0)\|^2},
 \label{eq:replaystep}
\end{equation}
then $L(\theta+u_\lambda)<L(\theta)$.
\end{proposition}

This is a sufficient descent condition, not a statement that more replay is always better: increasing $\lambda$ also changes the curvature penalty and the new-task trade-off. If replay supplies $g_R=g_0+e$ instead, the first-order improvement is $-\eta\lambda(g_0^\top Pg_0+g_0^\top Pe)$; target and probe coverage therefore remain relevant. Teacher KL constrains its own output-distribution loss and is not algebraically interchangeable with hard-label NLL.

Replay changes AdamW moments and future states \citep{loshchilov2019adamw}; separate replay steps differ from combined gradients. Proposition~\ref{prop:replay} motivates testing old-loss paths, not an AdamW guarantee or improvement of every decomposition term.

We therefore cross granularity with optimizer placement: pre-Adam constrains the stored descent gradient; post-Adam constrains realized displacement. Fixed probes and solvers isolate these design factors, including their induced optimizer-state trajectories rather than moments alone.

\subsection{From loss to generation and retention}

Generation is a sequence of discrete decisions. If every baseline-correct prefix has margin $\gamma_{ij}>0$, changing its logit vector by strictly less than $\gamma_{ij}/2$ in infinity norm suffices to preserve greedy decoding, including stopping under the same rule based on tokens and length (Appendix~\ref{app:margin}). Mean NLL, mean margin, and a single average-gradient halfspace do not enforce this per-prefix condition. The retained loss is defined by its prompts and supervised tokens; a response contract additionally specifies output form and stopping. We therefore distinguish loss, scored content, and exact terminated responses at fixed checkpoints and the endpoint.

This distinction organizes our experiments: compare explicit protection mechanisms, measure their paths, and jointly audit Old and New at fixed-budget endpoints. Appendix~\ref{app:implementation} describes the supporting \method{} implementation and acceptance contract.

\section{Experiments}
\label{sec:experiments}

\subsection{Experimental setup}
\label{sec:setup}

\paragraph{Models and acquired behaviors.}
The scale study has 576 rank-16 NF4 QLoRA runs on Mistral-7B-Instruct-v0.3, Qwen2.5-7B-Instruct, and OLMo-2-1124-7B-Instruct \citep{dettmers2023qlora,qwen2024qwen25,jiang2023mistral7b,olmo20242}; a separately seeded 256-run supplement tests 16 methods on Mistral/Qwen. Both adapt sentiment to codes (4,775/628/597 train/development/audit records). Qualified Old subsets have 23/17/26 records, mainly Mistral instruction format and Qwen/OLMo arithmetic: distinct functions, not an architecture ranking (Appendix~\ref{app:banks}).

A 64-run Mistral study compares six Code methods, plus plain/replay on natural-language sentiment and extraction over four seeds. Natural sentiment changes the interface, not the semantic task. Two independent Old banks each add 60 qualified records. New training supervises EOS; Old masks remain unchanged (Appendix~\ref{app:structure-location}).

\paragraph{Matched updates and replay mass.}
Within the refresh--coverage and structure--location protocols, paired methods share 2,048 ordered New exposures and 128 AdamW updates, 16 New records per update. Each record's loss averages its supervised tokens:
\begin{equation}
 \ell_t=\frac{1}{16}\sum_{i\in B_t^{\rm new}}\ell_i+
         \frac{1}{16}\sum_{j\in B_t^{\rm old}}\ell_j .
 \label{eq:mass}
\end{equation}
Nominal replay fractions 2/5/10/20\% add 42/108/228/512 Old exposures without replacing New records; the structure--location study uses 10\%. Exposure matching is not hardware-FLOP matching; Appendix~\ref{app:structure-location} specifies precision and optimization.

\paragraph{Projection interventions.}
The original study crosses refresh $r\in\{1,8\}$ with probes $n\in\{12,48\}$; diagnostic $E$ has 12 non-protection records. Cached/fresh branches share state and displacement. The structure--location factorial fixes 48 probes, every-step refresh, and one solver: pre checks the stored descent gradient; post checks displacement. Branches share initial state, moments, batch, and RNG, not the resulting Adam candidate. Validation includes $E$; hidden audit is separate.

\paragraph{Timing, targets, and baselines.}
Uniform/front/back schedules share the ordered 228 Old records and New sequence. Hybrid replaces half the old CE mass by teacher KL at 2/10\% replay. I-GEM, PEGP, and ARR adapt gradient constraints, feature-subspace protection, and KL/open-web replay to shared QLoRA budgets \citep{tekmen2026igem,qiao2024pegp,riemer2025arr}; calibration precedes hidden-audit access (Appendix~\ref{app:causal-methods}).

\paragraph{Endpoints and inference.}
All methods report final fixed-budget endpoints. Original/legacy generation uses \emph{reference-length greedy exact}; the structure--location banks use \emph{fixed-cap strict-EOS exact} within 64 tokens, alongside rule-based content accuracy. These interfaces use separate calls with different batches; Appendix~\ref{app:structure-score-audit} reports token-level checks. Structure--location generation is saved at updates 0/16/32/64/96/128. Signed AUC integrates baseline-centered validation NLL; $\mathrm{AUC}_{+}=\int_0^1[L(s)-L(0)]_+ds$ separates above-baseline exposure (Appendix~\ref{app:positive-exposure}). The sequential study instead evaluates each task boundary, deriving content and strict-EOS scores from the same singleton 64-token output (Appendix~\ref{app:continual-bridge}).

Paired retention changes and New scores measure benefit; all evaluable runs remain included regardless of contract status. \texttt{FEASIBLE} requires New acquisition and the full retention contract. Supporting New NI instead requires the paired 95\% $t$ lower bound above $-.03$ on the complete New audit. Eight-seed sign flips use frozen Holm families: original factorial AUC, baseline generation, schedule AUC, and hybrid AUC; the structure--location factorial uses legacy generation. Its supporting replay contrast uses equal-weight strict-bank changes. Four-seed tasks and the two-seed sequence are descriptive; intervals concern fixed banks.

\subsection{Constraint placement and behavioral retention}
\label{sec:results}

The original two-model refresh--coverage study supplies 64 projection runs and \textbf{8,192 actual updates}, all satisfying the registered residual tolerance; all 640 same-state branches also pass. Yet its Mistral projection conditions retain 3.80--4.35\% of the acquired reference-length answers, versus 100\% before SFT. 

\begin{table}[t]
\centering
\caption{Mistral retention under Code adaptation (left, eight paired seeds)
and sequential learning (right, two paired seeds; descriptive).
The sequence is AG News $\rightarrow$ SST-2 $\rightarrow$ JSON.
Accuracies are percentages; BWT is in percentage points.}
\label{tab:structure-location}
\begingroup
\small
\setlength{\tabcolsep}{2.5pt}
\renewcommand{\arraystretch}{1.03}
\begin{minipage}[t]{0.53\linewidth}
\vspace{0pt}
\centering
\textbf{(a) Code adaptation ($n=8$)}\par\smallskip
\begin{tabular*}{\linewidth}{@{\extracolsep{\fill}}lrrrr@{}}
\toprule
Method & Ref. & Strict & New & F/8\\
\midrule
Plain      & 5.43 & 3.65 & 92.40 & 0\\
Agg./pre   & 97.83 & 54.58 & 92.00 & 0\\
Agg./post  & 7.07 & 3.65 & 93.20 & 0\\
Str./pre   & 96.74 & 77.40 & 89.01 & 3\\
Str./post  & 70.65 & 24.27 & 93.55 & 0\\
Replay 10\% & 97.83 & 73.96 & 93.80 & 1\\
\bottomrule
\end{tabular*}
\end{minipage}\hfill
\begin{minipage}[t]{0.44\linewidth}
\vspace{0pt}
\centering
\textbf{(b) Sequential learning ($n=2$)}\par\smallskip
\begin{tabular*}{\linewidth}{@{\extracolsep{\fill}}lrrr@{}}
\toprule
Method & ACC & BWT & New\\
\midrule
Plain       & 77.34 & $-21.88$ & 98.44\\
Replay 10\% & \textbf{90.89} & $\mathbf{-1.17}$ & 99.22\\
Task-GEM    & 82.03 & $-15.62$ & 100.00\\
Adapter-EWC & 77.08 & $-21.88$ & 99.22\\
CLoRA       & 80.47 & $-17.58$ & 100.00\\
OPLoRA      & 81.25 & $-17.19$ & 99.22\\
\bottomrule
\end{tabular*}
\end{minipage}
\par\smallskip
\begin{minipage}{\linewidth}
\footnotesize
Left: Ref. is Legacy reference-length greedy exact; Strict averages two
independent 60-record strict-EOS banks (64-token cap), both initially 100\%.
New uses the full content audit; F counts full-contract passes, not beneficial
experiments. All evaluable runs are included.
Agg./Str.: aggregate/stratum-wise; pre/post: before/after Adam.
Right: ACC averages final task accuracies; BWT measures change from each
method's own acquisition scores; New is final JSON accuracy.
\end{minipage}
\endgroup
\end{table}

The matched factorial separates two design choices confounded in the earlier I-GEM comparison. On Legacy reference exact, the post-minus-pre effect is $-58.42$ points (95\% CI $[-68.82,-48.03]$, Holm $p=.02344$), with a granularity--placement interaction of $+64.67$ points ($[38.72,90.63]$, Holm $p=.03125$). Placement has the same direction on independent strict-EOS accuracy: $-52.03$ points ($[-76.84,-27.22]$, descriptive). The interaction is $-2.19$ points ($[-42.32,37.94]$) on that endpoint and $-1.67$ ($[-42.19,38.86]$) on independent reference exact. Thus, the Legacy interaction is endpoint-dependent, while the placement effect is consistent across these Code audits. Appendix~\ref{app:structure-location} reports all effects and same-state branches. A same-output prefix sensitivity preserves the Legacy effect directions (Table~\ref{tab:structure-prefix-sensitivity}).

\paragraph{Same-state intervention.}
The matched study records no numerical failures across 4,096
projected training updates and 160 same-state branches, checked
against each policy's declared constraint object. At five states
per seed, the four policies start from identical parameters,
optimizer state, and New batches. Stratum-wise minus aggregate
protection before Adam reduces diagnostic NLL change by .005475
(95\% CI $[-.009880,-.001071]$). The corresponding local interaction
is $+.004515$ ($[-.001298,.010328]$). These branches identify a
local benefit of stratum-wise pre-Adam protection. Local diagnostic effects and accumulated behavioral
effects therefore provide distinct evidence about the protection
mechanism.

\paragraph{Observed behavioral paths and audit coverage}

Aggregate/pre and replay both finish at 97.83\% legacy reference exact, but aggregate/pre falls to 29.89\% at update 16 and remains at 36.96\% at update 64 before recovering. Replay records 91.30\% and 95.11\% at those checkpoints (Figure~\ref{fig:structure-paths}). Their matching legacy endpoints also conceal different independent strict-EOS retention: 54.58\% versus 73.96\%. 

\begin{figure}[t]
\centering
\includegraphics[width=0.9\linewidth]{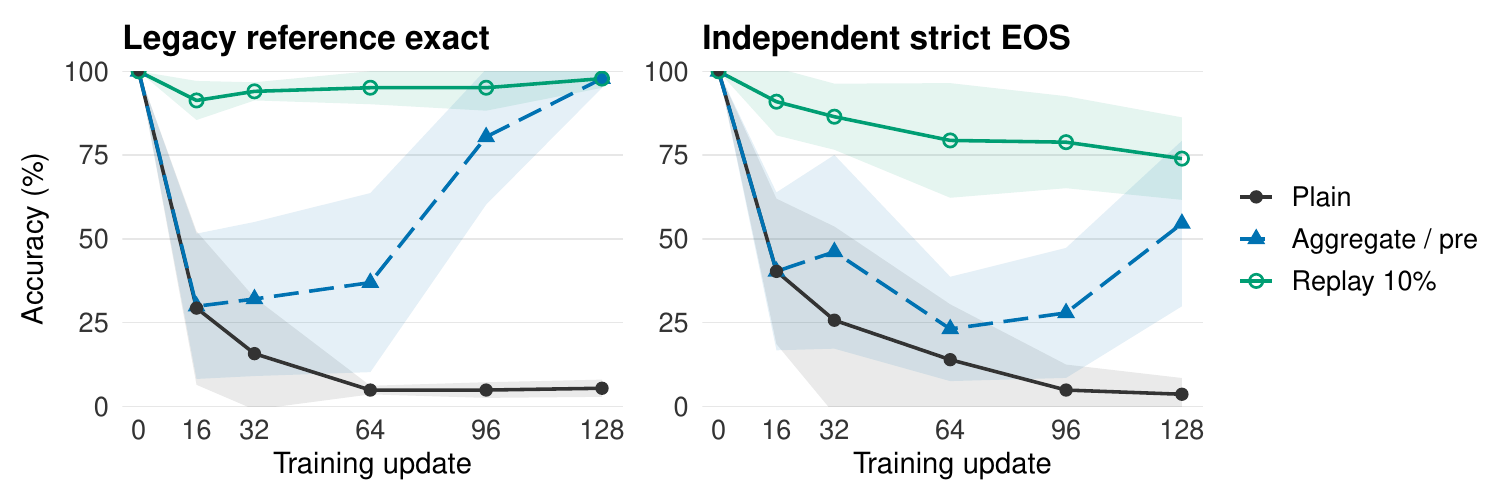}
\caption{Observed code-adaptation paths. Aggregate/pre and replay share the final legacy mean but follow different intermediate paths; independent strict-EOS audits expose further retention differences. Points and bands summarize eight seeds (pointwise 95\% $t$ intervals; portions outside 0--100\% not shown). Appendix~\ref{app:structure-location} reports all methods and separate banks.}
\label{fig:structure-paths}
\end{figure}

Code replay improves independent strict-bank retention by $70.31$ points over plain (CI $[55.97,84.66]$, supporting exact $p=.0078125$; all eight effects positive). Its full-New content difference is $+1.40$ points ($[-.90,3.71]$), satisfying the supporting $-3$-point NI rule. Its zero validation-NLL AUC+ and 73.96\% final independent strict accuracy also distinguish loss exposure from response retention.

\begin{table}[t]
\centering\small
\caption{Additional adaptation tasks, four paired seeds per method (descriptive). Percentages use the same old-bank definitions as Table~\ref{tab:structure-location}; New Content uses each complete audit. Positive legacy replay effects coexist with lower independent-bank strict accuracy.}
\label{tab:structure-transfer}
\begin{tabular}{llrrr}
\toprule
Task & Method & Legacy Ref. & Independent Strict & New Content\\
\midrule
Natural sentiment & Plain & 47.83 & 99.38 & 94.43\\
& Replay & 100.00 & 90.42 & 93.68\\
Structured extraction & Plain & 72.83 & 97.92 & 100.00\\
& Replay & 94.57 & 75.21 & 100.00\\
\bottomrule
\end{tabular}
\end{table}

Additional tasks distinguish scored content from exact response retention (Table~\ref{tab:structure-transfer}). Replay improves Legacy reference exact but changes independent strict means by $-8.96$ points on natural sentiment and $-22.71$ on extraction . On Extraction bank A, replay preserves 99.58\% content accuracy versus plain's 100\%, while strict accuracy is 82.08\% versus 100\%; all 42 content-correct/strict-incorrect replay records contain EOS. Bank B also loses content accuracy, so these effects are not uniformly stopping or response-form failures (Appendix~\ref{app:structure-outputs}). Both methods reach 100\% New extraction content; natural sentiment is high-baseline adaptation (82.24\% initial content). Code outputs show another pattern: the frozen code-label-intrusion rule flags 886/960 plain responses on the independent banks versus 23/960 with replay. These repeated outputs describe behavior, not independent trials or internal knowledge erasure.

\paragraph{Gradient freshness and probe coverage}

Refreshing the original same-state cache improves Mistral's immediate diagnostic loss at both probe sizes; Qwen's intervals include zero (Figure~\ref{fig:localrefresh}). Increasing coverage at $r8$ improves Qwen validation-NLL AUC by $.01315$ (Holm $p=.02344$). 

The decomposition also distinguishes net improvement from uniform component reduction. Replay's favorable Mistral diagnostic path coexists with larger positive staleness and residual sums than plain (Table~\ref{tab:causal-paths}). This agrees with the conditional restoring-direction account. 

\begin{figure}[t]
\centering
\includegraphics[width=0.9\linewidth]{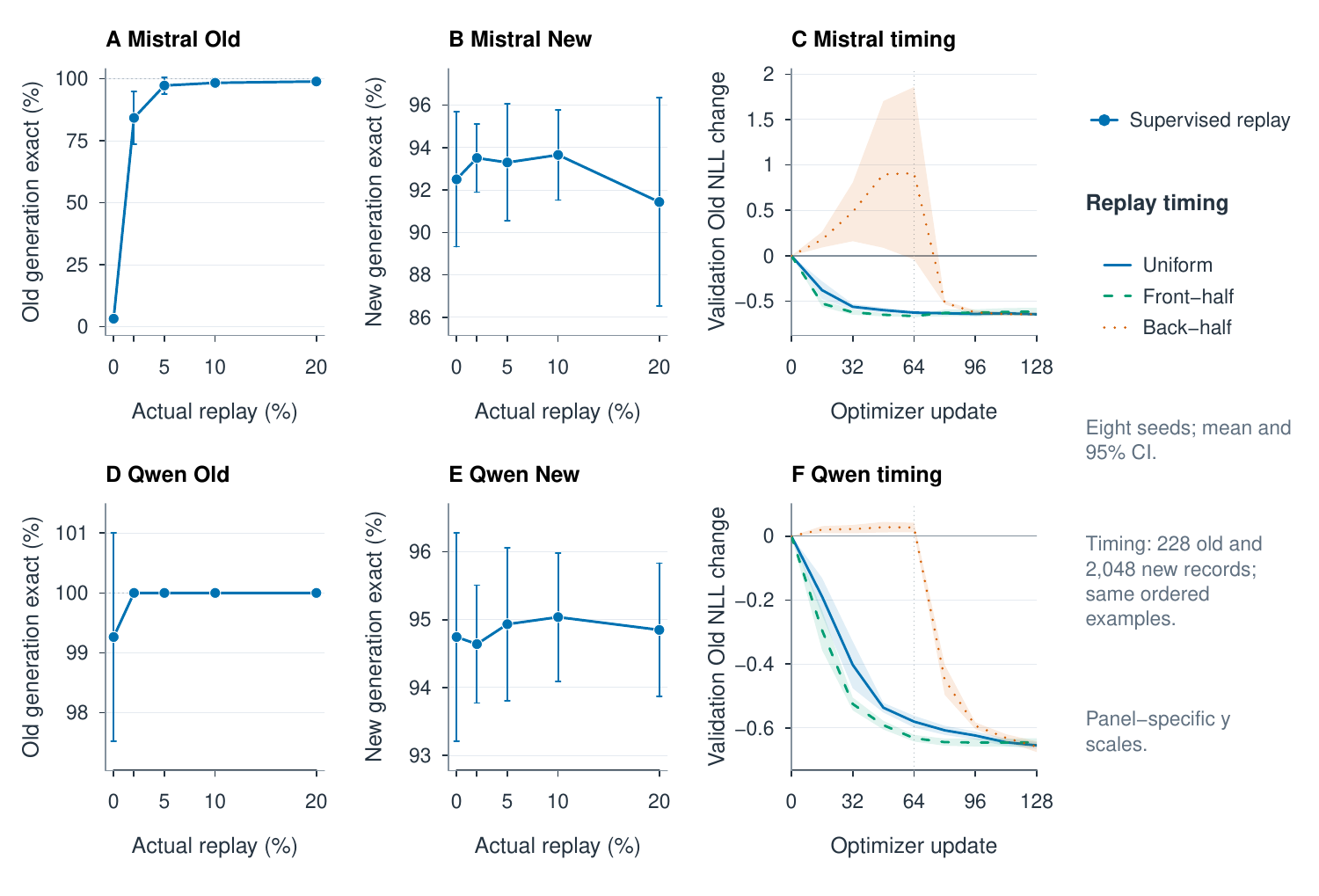}
\caption{Two-model replay operating points. A/B and D/E: final Old/New generation versus the exact old-example fraction. C/F: uniform/front/back trajectories with identical Old counts and ordered examples. The budget curves are descriptive; matched schedule interventions identify timing effects. Similar or identical final generation can accompany different accumulated old-loss exposure. Endpoint and pointwise trajectory intervals summarize eight seeds; panels use separate $y$ scales.}
\label{fig:replaybudget}
\end{figure}

\begin{table}[t]
\centering
\small
\caption{Task-wise retention in the
AG News $\rightarrow$ SST-2 $\rightarrow$ JSON sequence.
Entries show each method's just-learned $\rightarrow$ final
accuracy (\%), averaged over two paired seeds.
Final evaluation follows JSON training.
Content and strict-EOS scores agree; comparisons are descriptive.}
\label{tab:continual-bridge}
\begin{tabular}{lcc}
\toprule
Method
& AG News
& SST-2 \\
\midrule
Plain
& $82.03 \rightarrow 55.47$
& $95.31 \rightarrow 78.12$ \\
Replay 10\%
& $82.03 \rightarrow 81.25$
& $93.75 \rightarrow 92.19$ \\
Task-GEM
& $82.03 \rightarrow 68.75$
& $95.31 \rightarrow 77.34$ \\
Adapter-EWC
& $82.03 \rightarrow 53.91$
& $93.75 \rightarrow 78.12$ \\
CLoRA
& $82.81 \rightarrow 50.00$
& $93.75 \rightarrow 91.41$ \\
OPLoRA
& $85.94 \rightarrow 74.22$
& $92.97 \rightarrow 70.31$ \\
\bottomrule
\end{tabular}
\end{table}
\subsection{Replay timing: loss exposure beyond final accuracy}
\label{sec:practical}

Matched schedules share the New/Old sequences, optimizer steps, and total old objective mass, including exactly 228 Old records. Mistral uniform/back replay finish at 98.37/97.83\% Old generation but have AUCs $-.5494/+.0451$: delaying targets increases exposure by $.59443$ (CI $[.32867,.86020]$, Holm $p=.015625$). Front-loading lowers AUC by $.03107$ (Holm $p=.015625$), with five feasible endpoints versus uniform's eight.

All Qwen schedules finish at 100\% Old generation with eight feasible endpoints. Yet front-loading lowers AUC by $.04865$ and back-loading raises it by $.25071$ (both Holm $p=.015625$; full CIs in Table~\ref{tab:causal-registered}). Target timing therefore changes the loss path even when final generation does not distinguish schedules.

A post-hoc positive-part analysis separates exposure that later improvement cannot cancel. Back-minus-uniform AUC+ is $.29160$ on Mistral (CI $[.03536,.54783]$) and $.01067$ on Qwen ($[.00542,.01591]$), each with Holm $p=.03125$ in the four-contrast supporting family. Front and uniform both have zero positive-part area on the recorded grid. Front-loading's signed-AUC advantage thus differs from back-loading's above-baseline exposure (Appendix~\ref{app:positive-exposure}).

These interventions distinguish target timing from target count. All Qwen schedules establish supporting New NI; on Mistral, uniform establishes it while front/back do not. Neither the lowest AUC nor highest final generation alone determines the joint operating point.

\paragraph{Replay budget and objective allocation}

At 2/5/10/20\% replay, Mistral Old generation is 84.24/97.28/98.37/98.91\%. Ten-percent replay establishes supporting New NI ($+1.15$ points, CI $[-2.88,5.19]$); 5/20\% do not. Qwen retains 100\% generation and supporting NI at every budget, despite different loss paths.

Replacing half the supervised objective mass by base-teacher KL increases Old-NLL AUC by $.08276/.08725$ on Mistral at 2/10\% replay and $.09736/.11283$ on Qwen (all model-specific Holm $p=.02344$; full CIs in Table~\ref{tab:causal-registered}). This compares allocations of a fixed old-target budget. Earlier Pythia interventions support both labels and teacher targets over input-only replay (Appendix~\ref{app:historical}).


I-GEM retains 96.20\% Mistral generation versus replay's 98.37\%. Our factorial is not finite-iteration I-GEM (Appendix~\ref{app:causal}). Matching model-token equivalents (MTE, a compute proxy, not measured FLOPs) preserves Mistral's 94.02-point replay benefit. Longer training yields 24/24 feasible replay endpoints and New NI for all models.
\subsection{Sequential retention: protected subspaces versus acquired tasks}
\label{sec:continual-bridge}

One Mistral adapter and Adam state learn AG News, SST-2, then JSON extraction: 32 updates per task, six policies, two paired seeds. Plain reaches 98.44\% extraction, but news/sentiment decline from 82.03/95.31\% when acquired to 55.47/78.12\% finally (Table~\ref{tab:continual-bridge}).

CLoRA/OPLoRA improve mean final accuracy by 3.12/3.91 points. CLoRA improves sentiment but reduces news accuracy; OPLoRA shows the opposite pattern. Replay improves both. Verified OPLoRA subspace protection coexists with learned-task regression: initial-weight subspaces and subsequently acquired behavior are different protection objects. Appendix~\ref{app:continual-bridge} gives all matrices and numerical checks for this descriptive comparison, including eight archived controls and four new chains.

These regressions involve learned task answers rather than only
response formatting: 83 of Plain's 85 final old-task errors are
valid but incorrect classification labels.

Meanwhile, OPLoRA's largest final protected-leakage energy ratio
is $1.91\times10^{-14}$ for FP32 effective factors
($3.72\times10^{-7}$ after BF16 factor rounding).
In this fixed two-seed sequence, numerical protection of selected
initial-weight subspaces therefore coexists with the loss of
subsequently acquired task answers.
\section{Related Work}
\label{sec:related-work}

\paragraph{Functional targets and replay.}
Functional regularization anchors predictions \citep{pan2020continual}; ARR combines KL with corpus replay and MSSR schedules memory strength \citep{riemer2025arr,lu2026mssr}. Our matched-target interventions isolate timing effects beyond final accuracy (Appendix~\ref{app:causal}).

\paragraph{Gradient constraints and forgetting theory.}
GEM/A-GEM constrain episodic gradients \citep{lopez2017gradient}; I-GEM constrains LoRA gradients and PEGP protects feature subspaces \citep{tekmen2026igem,qiao2024pegp}. Beyond linear-model expected-forgetting analyses \citep{lin2023theory}, our nonlinear account separates direction, staleness, mismatch, and finite-step residuals, with interventions on granularity and placement.

\paragraph{Parameter and subspace protection.}
EWC uses Fisher-weighted penalties \citep{kirkpatrick2017overcoming}; Adam-NSCL projects Adam updates into historical-feature null spaces \citep{wang2021adamnscl}. CLoRA regularizes LoRA factors against fixed subspaces; OPLoRA protects weight singular subspaces \citep{lu2025clora,xiong2026oplora}. Our EWC/CLoRA/OPLoRA comparisons measure acquired-task retention (Appendix~\ref{app:continual-bridge}).

\section{Conclusion}
\label{sec:conclusion}

The functional-path account connects realized updates to retained behavior. Controlled interventions identify placement effects and endpoint-dependent interactions; checkpointed outputs and matched schedules distinguish endpoints from trajectories. Sequential tasks show that average subspace-protection gains coexist with old-task trade-offs. Together with replay's substantial but response-dependent gains, these findings support explicit functional anchors and task-wise old/new evaluation.
\label{sec:main-end}

\clearpage
\subsection*{AI use statement}

Generative AI tools were used to assist with code development, text editing, and translation, and contributed to the drafting of relevant sections of the appendix.
\subsection*{Ethics statement}

The study evaluates retention on declared, model-specific capabilities under controlled adaptation tasks. Preserving a measured capability does not establish the safety or suitability of a deployed model: retention can also preserve undesirable behavior. Retention banks and deployment audits should therefore be selected for the intended application, with separate evaluation of safety, bias and privacy. The manuscript limits its behavioral claims to the measured interfaces and reports complete outcomes and declared audit statuses.

\subsection*{Reproducibility statement}

Section~\ref{sec:method} states the assumptions and Appendix~\ref{app:proofs} provides the proofs and constructive examples. Section~\ref{sec:experiments} specifies the intervention logic and primary comparisons. Appendix~\ref{app:causal} gives the original supplement protocol, adapted baselines, complete groups, statistical families and measurement audit. Appendices~\ref{app:scale-results}--\ref{app:historical} retain the scale and historical evidence; Appendix~\ref{app:implementation} gives contract details. Appendix~\ref{app:structure-location} specifies the separate structure--location study, task and decoding interfaces, complete outcomes, sampled outputs and retained checkpoints. Appendix~\ref{app:structure-score-audit} reports an offline scoring audit and targeted GPU recheck, distinct from new training. Appendix~\ref{app:continual-bridge} supplies the sequential data, baseline adaptations, all twelve task-boundary matrices, and numerical subspace checks. Statistics were recomputed from seed-level outcomes; generation, teacher-forced loss, numerical feasibility and new-task non-inferiority are distinct measurements.

\bibliography{iclr2027_conference}

@article{chaudhry2018efficient,
  title={Efficient lifelong learning with a-gem},
  author={Chaudhry, Arslan and Ranzato, Marc'Aurelio and Rohrbach, Marcus and Elhoseiny, Mohamed},
  journal={arXiv preprint arXiv:1812.00420},
  year={2018}
}

@article{titsias2019functional,
  title={Functional regularisation for continual learning with gaussian processes},
  author={Titsias, Michalis K and Schwarz, Jonathan and Matthews, Alexander G de G and Pascanu, Razvan and Teh, Yee Whye},
  journal={arXiv preprint arXiv:1901.11356},
  year={2019}
}

@article{pan2020continual,
  title={Continual deep learning by functional regularisation of memorable past},
  author={Pan, Pingbo and Swaroop, Siddharth and Immer, Alexander and Eschenhagen, Runa and Turner, Richard and Khan, Mohammad Emtiyaz},
  journal={Advances in neural information processing systems},
  volume={33},
  pages={4453--4464},
  year={2020}
}

@misc{jiang2023mistral7b,
      title={Mistral 7B}, 
      author={Albert Q. Jiang and Alexandre Sablayrolles and Arthur Mensch and Chris Bamford and Devendra Singh Chaplot and Diego de las Casas and Florian Bressand and Gianna Lengyel and Guillaume Lample and Lucile Saulnier and Lélio Renard Lavaud and Marie-Anne Lachaux and Pierre Stock and Teven Le Scao and Thibaut Lavril and Thomas Wang and Timothée Lacroix and William El Sayed},
      year={2023},
      eprint={2310.06825},
      archivePrefix={arXiv},
      primaryClass={cs.CL},
      url={https://arxiv.org/abs/2310.06825}, 
}

@article{olmo20242,
  title={2 OLMo 2 Furious},
  author={OLMo, Team and Walsh, Pete and Soldaini, Luca and Groeneveld, Dirk and Lo, Kyle and Arora, Shane and Bhagia, Akshita and Gu, Yuling and Huang, Shengyi and Jordan, Matt and others},
  journal={arXiv preprint arXiv:2501.00656},
  year={2024}
}

@article{lu2026mssr,
  title={MSSR: Memory-Aware Adaptive Replay for Continual LLM Fine-Tuning},
  author={Lu, Yiyang and He, Yu and Chen, Jianlong and Zha, Hongyuan},
  journal={arXiv preprint arXiv:2603.09892},
  year={2026}
}

@article{kirkpatrick2017overcoming,
  title={Overcoming catastrophic forgetting in neural networks},
  author={Kirkpatrick, James and Pascanu, Razvan and Rabinowitz, Neil and Veness, Joel and Desjardins, Guillaume and Rusu, Andrei A and Milan, Kieran and Quan, John and Ramalho, Tiago and Grabska-Barwinska, Agnieszka and others},
  journal={Proceedings of the national academy of sciences},
  volume={114},
  number={13},
  pages={3521--3526},
  year={2017},
  publisher={National Academy of Sciences}
}

@article{lopez2017gradient,
  title={Gradient episodic memory for continual learning},
  author={Lopez-Paz, David and Ranzato, Marc'Aurelio},
  journal={Advances in neural information processing systems},
  volume={30},
  year={2017}
}

@article{li2017learning,
  title={Learning without forgetting},
  author={Li, Zhizhong and Hoiem, Derek},
  journal={IEEE transactions on pattern analysis and machine intelligence},
  volume={40},
  number={12},
  pages={2935--2947},
  year={2017},
  publisher={IEEE}
}

@article{buzzega2020dark,
  title={Dark experience for general continual learning: a strong, simple baseline},
  author={Buzzega, Pietro and Boschini, Matteo and Porrello, Angelo and Abati, Davide and Calderara, Simone},
  journal={Advances in neural information processing systems},
  volume={33},
  pages={15920--15930},
  year={2020}
}

@inproceedings{kotha2024understanding,
  title={Understanding catastrophic forgetting in language models via implicit inference},
  author={Kotha, Suhas and Springer, Jacob and Raghunathan, Aditi},
  booktitle={International Conference on Learning Representations},
  volume={2024},
  pages={24110--24139},
  year={2024}
}

@article{tao2023can,
  title={Can bert refrain from forgetting on sequential tasks? a probing study},
  author={Tao, Mingxu and Feng, Yansong and Zhao, Dongyan},
  journal={arXiv preprint arXiv:2303.01081},
  year={2023}
}

@inproceedings{zheng2025spurious,
  title={Spurious forgetting in continual learning of language models},
  author={Zheng, Junhao and Cai, Xidi and Qiu, Shengjie and Ma, Qianli},
  booktitle={International Conference on Learning Representations},
  volume={2025},
  pages={67081--67146},
  year={2025}
}

@article{qiao2024pecl,
  title={Learn more, but bother less: parameter efficient continual learning},
  author={Qiao, Fuli and Mahdavi, Mehrdad},
  journal={Advances in Neural Information Processing Systems},
  volume={37},
  pages={97476--97498},
  year={2024}
}

@inproceedings{lin2023theory,
  title={Theory on forgetting and generalization of continual learning},
  author={Lin, Sen and Ju, Peizhong and Liang, Yingbin and Shroff, Ness},
  booktitle={International Conference on Machine Learning},
  pages={21078--21100},
  year={2023},
  organization={PMLR}
}

@article{rolnick2019replay,
  title={Experience replay for continual learning},
  author={Rolnick, David and Ahuja, Arun and Schwarz, Jonathan and Lillicrap, Timothy and Wayne, Gregory},
  journal={Advances in neural information processing systems},
  volume={32},
  year={2019}
}

@inproceedings{wang2024inscl,
  title={Inscl: A data-efficient continual learning paradigm for fine-tuning large language models with instructions},
  author={Wang, Yifan and Liu, Yafei and Shi, Chufan and Li, Haoling and Chen, Chen and Lu, Haonan and Yang, Yujiu},
  booktitle={Proceedings of the 2024 Conference of the North American Chapter of the Association for Computational Linguistics: Human Language Technologies (Volume 1: Long Papers)},
  pages={663--677},
  year={2024}
}

@inproceedings{he2024seekr,
  title={Seekr: Selective attention-guided knowledge retention for continual learning of large language models},
  author={He, Jinghan and Guo, Haiyun and Zhu, Kuan and Zhao, Zihan and Tang, Ming and Wang, Jinqiao},
  booktitle={Proceedings of the 2024 Conference on Empirical Methods in Natural Language Processing},
  pages={3254--3266},
  year={2024}
}

@inproceedings{huang2024ssr,
  title={Mitigating catastrophic forgetting in large language models with self-synthesized rehearsal},
  author={Huang, Jianheng and Cui, Leyang and Wang, Ante and Yang, Chengyi and Liao, Xinting and Song, Linfeng and Yao, Junfeng and Su, Jinsong},
  booktitle={Proceedings of the 62nd Annual Meeting of the Association for Computational Linguistics (Volume 1: Long Papers)},
  pages={1416--1428},
  year={2024}
}

@article{ren2024ilora,
  title={Analyzing and reducing catastrophic forgetting in parameter efficient tuning},
  author={Ren, Weijieying and Li, Xinlong and Wang, Lei and Zhao, Tianxiang and Qin, Wei},
  journal={arXiv preprint arXiv:2402.18865},
  year={2024}
}

@article{hu2022lora,
  title={Lora: Low-rank adaptation of large language models},
  author={Hu, Edward J and Shen, Yelong and Wallis, Phillip and Allen-Zhu, Zeyuan and Li, Yuanzhi and Wang, Shean and Wang, Lu and Chen, Weizhu},
  journal={arXiv preprint arXiv:2106.09685},
  year={2021}
}

@article{dettmers2023qlora,
  title={Qlora: Efficient finetuning of quantized llms},
  author={Dettmers, Tim and Pagnoni, Artidoro and Holtzman, Ari and Zettlemoyer, Luke},
  journal={Advances in neural information processing systems},
  volume={36},
  pages={10088--10115},
  year={2023}
}

@article{loshchilov2019adamw,
  title={Decoupled weight decay regularization},
  author={Loshchilov, Ilya and Hutter, Frank},
  journal={arXiv preprint arXiv:1711.05101},
  year={2017}
}

@inproceedings{kornblith2019cka,
  title={Similarity of neural network representations revisited},
  author={Kornblith, Simon and Norouzi, Mohammad and Lee, Honglak and Hinton, Geoffrey},
  booktitle={International conference on machine learning},
  pages={3519--3529},
  year={2019},
  organization={PMlR}
}

@inproceedings{biderman2023pythia,
  title={Pythia: A suite for analyzing large language models across training and scaling},
  author={Biderman, Stella and Schoelkopf, Hailey and Anthony, Quentin Gregory and Bradley, Herbie and O’Brien, Kyle and Hallahan, Eric and Khan, Mohammad Aflah and Purohit, Shivanshu and Prashanth, USVSN Sai and Raff, Edward and others},
  booktitle={International conference on machine learning},
  pages={2397--2430},
  year={2023},
  organization={PMLR}
}

@article{qwen2024qwen25,
  title={Qwen2. 5-coder technical report},
  author={Hui, Binyuan and Yang, Jian and Cui, Zeyu and Yang, Jiaxi and Liu, Dayiheng and Zhang, Lei and Liu, Tianyu and Zhang, Jiajun and Yu, Bowen and Lu, Keming and others},
  journal={arXiv preprint arXiv:2409.12186},
  year={2024}
}

@article{zhang2024tinyllama,
  title={Tinyllama: An open-source small language model},
  author={Zhang, Peiyuan and Zeng, Guangtao and Wang, Tianduo and Lu, Wei},
  journal={arXiv preprint arXiv:2401.02385},
  year={2024}
}

@inproceedings{tekmen2026igem,
  title={GEM-Style Constraints for PEFT with Dual Gradient Projection in LoRA},
  author={Tekmen, Brian and Yin, Jason and Tong, Qianqian},
  booktitle={2025 IEEE International Conference on Data Mining Workshops (ICDMW)},
  pages={2736--2743},
  year={2025},
  organization={IEEE}
}

@article{qiao2024pegp,
  title={Gradient projection for continual parameter-efficient tuning},
  author={Qiao, Jingyang and Zhang, Zhizhong and Tan, Xin and Qu, Yanyun and Zhang, Wensheng and Han, Zhi and Xie, Yuan},
  journal={IEEE Transactions on Pattern Analysis and Machine Intelligence},
  volume={47},
  number={10},
  pages={9316--9329},
  year={2025},
  publisher={IEEE}
}

@article{riemer2025arr,
  title={The effectiveness of approximate regularized replay for efficient supervised fine-tuning of large language models},
  author={Riemer, Matthew and Miehling, Erik and Liu, Miao and Bouneffouf, Djallel and Campbell, Murray},
  journal={arXiv preprint arXiv:2512.22337},
  year={2025}
}

@article{zhang2015character,
  title={Character-level convolutional networks for text classification},
  author={Zhang, Xiang and Zhao, Junbo and LeCun, Yann},
  journal={Advances in neural information processing systems},
  volume={28},
  year={2015}
}

@inproceedings{socher2013recursive,
  title={Recursive deep models for semantic compositionality over a sentiment treebank},
  author={Socher, Richard and Perelygin, Alex and Wu, Jean and Chuang, Jason and Manning, Christopher D and Ng, Andrew Y and Potts, Christopher},
  booktitle={Proceedings of the 2013 conference on empirical methods in natural language processing},
  pages={1631--1642},
  year={2013}
}

@inproceedings{wang2021adamnscl,
  title={Training networks in null space of feature covariance for continual learning},
  author={Wang, Shipeng and Li, Xiaorong and Sun, Jian and Xu, Zongben},
  booktitle={2021 IEEE/CVF Conference on Computer Vision and Pattern Recognition (CVPR)},
  pages={184--193},
  year={2021},
  organization={IEEE}
}

@inproceedings{lu2025clora,
  title={Controlled low-rank adaptation with subspace regularization for continued training on large language models},
  author={Lu, Yuheng and Qian, Bingshuo and Yuan, Caixia and Jiang, Huixing and Wang, Xiaojie},
  booktitle={Proceedings of the 63rd Annual Meeting of the Association for Computational Linguistics (Volume 1: Long Papers)},
  pages={19165--19181},
  year={2025}
}

@inproceedings{xiong2026oplora,
  title={Oplora: Orthogonal projection lora prevents catastrophic forgetting during parameter-efficient fine-tuning},
  author={Xiong, Yifeng and Xie, Xiaohui},
  booktitle={Proceedings of the AAAI Conference on Artificial Intelligence},
  volume={40},
  number={40},
  pages={34088--34096},
  year={2026}
}
\bibliographystyle{iclr2027_conference}
\appendix
\section{Proofs and Constructive Examples}
\label{app:proofs}

\subsection{Proof of Proposition~\ref{prop:path}}

Let $f_t(s)=L(\theta_t+s u_t)$. Applying the fundamental theorem of calculus twice gives
\begin{align}
 f_t(1)-f_t(0)
 &=f_t'(0)+\int_0^1\int_0^s f_t''(v)\,dv\,ds\\
 &=\nabla L(\theta_t)^\top u_t+
 \int_0^1(1-v)u_t^\top\nabla^2L(\theta_t+v u_t)u_t\,dv.
\end{align}
Substitute $\nabla L(\theta_t)=h_t+a_t+s_t$ and sum over $t$; the left side telescopes, yielding Equation~\ref{eq:decomp}. Lipschitz continuity gives $\|s_t\|\leq\beta_Q\|\theta_t-\theta_{\tau(t)}\|$. For a twice differentiable function, a $\beta_L$-Lipschitz gradient implies $\|\nabla^2L(\theta)\|_{\mathrm{op}}\leq\beta_L$, by differentiating the gradient along each unit direction. Thus $|r_t|\leq\beta_L\|u_t\|^2/2$. Cauchy--Schwarz for the mismatch and staleness terms, together with the residual assumption, proves Equation~\ref{eq:bound}.\hfill$\square$

\paragraph{Finite precision and gradient error.}
If the computed cache is $\widetilde h_t=h_t+e_t^h$ and $\widetilde h_t^\top u_t\leq\widetilde\varepsilon_t$, write $h_t^\top u_t=\widetilde h_t^\top u_t-(e_t^h)^\top u_t$. The bound uses $\widetilde\varepsilon_t$ and acquires $\|e_t^h\|\|u_t\|$. The logged projection residual checks the computed constraint, while the analytical bound retains the gradient-error norm and smoothness constants as explicit assumptions.

\paragraph{Empirical accounting identity.}
\label{app:measurement}
Let the logged pre/post losses be $\widetilde L_t^-=L(\theta_t)+\delta_t^-$ and $\widetilde L_t^+=L(\theta_{t+1})+\delta_t^+$, the logged evaluation gradient be $\widetilde g_t=g_t+e_t^g$, and the recorded displacement be $\widetilde u_t=u_t+e_t^u$. Interpreting these recorded values in real arithmetic, the measured increment and finite-step residual satisfy
\begin{align}
 \widetilde\Delta_t&=\widetilde L_t^+-\widetilde L_t^-,\notag\\
 \widetilde r_t&=\widetilde\Delta_t-\widetilde g_t^\top\widetilde u_t\notag\\
 &=r_t+\delta_t^+-\delta_t^--(e_t^g)^\top u_t-\widetilde g_t^\top e_t^u.
 \label{eq:measured-remainder}
\end{align}
Thus the residual includes the smooth finite-step term and numerical loss, gradient, and displacement errors, with additional scalar roundoff when reductions are evaluated numerically. It is not a direct measurement of the Hessian integral. Adding and subtracting logged fresh and cached gradient dots reconstructs $\widetilde\Delta_t$ algebraically; near-zero closure error checks that reconstruction, not smoothness or the size of these errors.

Consecutive evaluations are checked independently. With $\xi_t=\widetilde L_{t+1}^- -\widetilde L_t^+$, their endpoint accounting is
\begin{equation}
 \widetilde L_{T-1}^+-\widetilde L_0^-
 =\sum_{t<T}\widetilde\Delta_t+\sum_{t<T-1}\xi_t.
 \label{eq:measurement-gaps}
\end{equation}
Reporting both the increment sum and the adjacent-evaluation gaps separates within-update accounting from continuity across the logged trajectory.

\subsection{Proof of Corollary~\ref{cor:refresh}}

For a block $B=\{b,\ldots,b+m-1\}$, $m\leq r$, set $v_t=\|u_t\|$. Then $\|\theta_t-\theta_b\|\leq\sum_{k=b}^{t-1}v_k$, and
\begin{align}
 \sum_{t\in B}\|\theta_t-\theta_b\|v_t
 &\leq\sum_{b\leq k<t<b+m}v_kv_t\\
 &=\frac12\left[\left(\sum_{t\in B}v_t\right)^2-\sum_{t\in B}v_t^2\right]
 \leq\frac{r-1}{2}\sum_{t\in B}v_t^2.
\end{align}
The last inequality follows from Cauchy--Schwarz and $m\leq r$. Summing over blocks and substituting into Equation~\ref{eq:bound} proves Equation~\ref{eq:block}. For $L=Q$, set $D=0$ and $\beta_L=\beta_Q=\beta$.\hfill$\square$

The same bound applies to every prefix of the trajectory after truncating its last refresh block. Bounding one final loss does not, by itself, bound the largest intermediate loss. Conversely, uniform bounds on every partial right-hand-side sum would provide a pathwise sufficient budget.

\subsection{Tight quadratic construction}
\label{app:quadratic}

Take $L=Q=\|\theta\|^2/2$ on $\R^2$, $\theta_0\ne0$, and
\begin{equation}
 J=\begin{pmatrix}0&-1\\1&0\end{pmatrix},\qquad J^\top=-J,\quad J^\top J=I.
\end{equation}
In block $k$ cache $h=\theta_{kr}$. At each of its $r$ steps, use candidate $v=bh+aJh$ for $a,b>0$. Since $h^\top v=b\|h\|^2>0$, projection is active and yields $u=aJh$. Every cached gradient is nonzero and every post-projection dot is exactly zero. Within the block $\theta_{kr+j}=h+jaJh$, so
\begin{align}
 \nabla L(\theta_{kr+j})^\top u&=ja^2\|h\|^2,\\
 L(\theta_{kr+j+1})-L(\theta_{kr+j})&=(j+1/2)a^2\|h\|^2>0.
\end{align}
Orthogonality gives $L(\theta_{(k+1)r})=(1+r^2a^2)L(\theta_{kr})$, hence $L(\theta_{Kr})=(1+r^2a^2)^K L(\theta_0)$ diverges as $K$ grows. The block loss increase is $r^2a^2\|h\|^2/2$ and the squared-update sum is $r a^2\|h\|^2$, attaining $r\sum\|u_t\|^2/2$ exactly. For $r=1$, the fresh dot is zero but curvature still increases the loss; for $r>1$, each positive-age step also has a positive fresh first-order term.\hfill$\square$

\subsection{A cross-entropy and greedy-output construction}

For one binary old example with correct class 1, define $z_1(\theta)=c-\|\theta\|^2/2$, $z_0(\theta)=0$, and $L=Q=\log(1+\exp(-z_1))$. Start at $\theta_0=(1,0)$ with $c>1/2$, giving a correct greedy answer. The cached gradient is $h=\sigma(1/2-c)(1,0)\ne0$, where $\sigma$ is the logistic function. Project the candidate $(b,a)$, $a,b>0$, against this fixed cache for $r$ steps. Every projection is active and yields $u=(0,a)$, so $\theta_j=(1,ja)$ and $h^\top u=0$. Nevertheless,
\begin{equation}
 \nabla L(\theta_j)^\top u
 =ja^2\sigma\left((1+j^2a^2)/2-c\right)>0\quad(j>0).
\end{equation}
The loss increases at every step and the final greedy answer is wrong if $r^2a^2>2c-1$. For $r=8$, $a=1/5$, and $c=1$, the margin changes from $1/2$ to $-39/50$. For arbitrarily small $a>0$, choosing $1/2<c<(1+r^2a^2)/2$ gives a positive initial margin and the same failure. This analytic example has identical probe and evaluation data and a nonzero control gradient. It is not a measured language-model trajectory.

\subsection{Proof of Proposition~\ref{prop:replay}}

Let $v_\lambda=P(g_1+\lambda g_0)$ and $q=g_0^\top Pg_0>0$. The smoothness inequality gives
\begin{equation}
 L(\theta-\eta v_\lambda)-L(\theta)
 \leq-\eta A_\lambda+\frac{\beta\eta^2}{2}\|v_\lambda\|^2.
\end{equation}
If $A_\lambda>0$, then $v_\lambda\ne0$. Equation~\ref{eq:replaystep} makes this upper bound strictly negative; allowing equality in its upper step-size bound gives non-increase. Because $A_\lambda=g_0^\top Pg_1+\lambda q$ is affine with positive slope, one sufficient choice is $\lambda>\max\{0,-g_0^\top Pg_1/q\}$. Subtracting the two common-state candidates gives Equation~\ref{eq:replaydirection}.\hfill$\square$

If $P$ is positive semidefinite rather than definite, the same reasoning requires $q>0$. If $g_0=0$, no strictly restoring first-order term is available. With a sampled replay gradient $g_R=g_0+e$, the $P$-norm inequality
\begin{equation}
 |g_0^\top Pe|\leq\|g_0\|_P\|e\|_P,
 \qquad\|w\|_P=(w^\top Pw)^{1/2},
\end{equation}
follows by applying Cauchy--Schwarz to $P^{1/2}g_0$ and $P^{1/2}e$. Thus $\|e\|_P<\|g_0\|_P$ is sufficient for a strictly beneficial first-order replay term when $\eta,\lambda>0$. We do not measure this uniform mismatch condition for a hidden audit.

\paragraph{AdamW and teacher targets.}
Adding a replay gradient generally changes both Adam's numerator and diagonal preconditioner, even when the preceding optimizer state is shared. Clipping and interleaved separate updates add further differences. A frozen-teacher KL gradient is the gradient of a different objective and vanishes when student and teacher distributions agree exactly. Applying the proposition to that KL establishes a statement about that KL, not automatically about hard-label accuracy. These facts explain why the conditional calculation and the empirical long-horizon effects are reported as different forms of evidence.

\subsection{A sufficient condition for greedy generation retention}
\label{app:margin}

\begin{proposition}[Preservation of acquired greedy decisions]
\label{prop:margin}
Fix a finite prompt bank and deterministic greedy decoding with a fixed stopping rule depending only on generated tokens and their count. Suppose the baseline generates $y_{i,1},\ldots,y_{i,m_i}$ on each acquired item, including a termination token when the rule uses one. On every correct prefix $p_{ij}=(x_i,y_{i,<j})$, define the margin
\begin{equation}
 \gamma_{ij}=z_{\theta_0}(y_{ij}\mid p_{ij})-
 \max_{v\ne y_{ij}}z_{\theta_0}(v\mid p_{ij})>0.
\end{equation}
If
\begin{equation}
 \|z_{\theta_T}(\cdot\mid p_{ij})-z_{\theta_0}(\cdot\mid p_{ij})\|_\infty
 <\gamma_{ij}/2\qquad\text{for every }i,j,
\end{equation}
then the final model produces the same greedy sequences and stopping positions.
\end{proposition}
\begin{proof}
For any competitor $v\ne y_{ij}$, the final correct-token logit gap is bounded below by
\begin{equation}
 z_{\theta_T}(y_{ij}\mid p_{ij})-z_{\theta_T}(v\mid p_{ij})
 \geq\gamma_{ij}-2\|z_{\theta_T}(\cdot\mid p_{ij})-z_{\theta_0}(\cdot\mid p_{ij})\|_\infty>0.
\end{equation}
The first prefixes coincide. Inductively, equality of previous tokens supplies the same prefix and the strict inequality forces the same next token, including the stopping decision.
\end{proof}
If every token logit's parameter gradient along segment $t$ at prefix $p_{ij}$ has norm at most $K_{ij,t}$, the fundamental theorem of calculus and the triangle inequality give a sufficient path condition $\sum_{t<T}K_{ij,t}\|u_t\|<\gamma_{ij}/2$. A mean answer margin does not establish these per-prefix minima, and these Lipschitz constants were not certified for the experimental models.

\paragraph{Parameterization and representation.}
LoRA deploys $W_{\mathrm{base}}+sBA$. For nonzero $c$, the transformation $(B,A)\mapsto(cB,A/c)$ preserves this effective weight while generally changing factor coordinates and their norms. A measured factor-distance statistic is therefore not an invariant function distance. Likewise, a function-relevant output margin may change while broad representation similarity remains high. These observations explain why geometric diagnostics and functional audits answer different questions; they do not imply that all representational change is irrelevant to forgetting.

\section{Causal Supplement: Protocol, Complete Families, and Measurement Audit}
\label{app:causal}

\subsection{Snapshot and statistical design}

This appendix reports the original 256-run refresh--coverage supplement, verified on the server on September 13, 2026, at 12:52 Beijing time. It contains all 32 smoke, 36 calibration, and \textbf{256 formal outcomes}: 128/128 Mistral and 128/128 Qwen. The formal statuses are 171 feasible, 84 retention-audit failures, and one New-acquisition failure. The latter is Mistral $r1,n48$, seed 51008; it remains in all assigned endpoints. A completed scientific failure is an observation, not a missing process or a reason to drop a seed. The separately analyzed structure--location study is reported in Appendix~\ref{app:structure-location}.

All 32 method groups and eight model-specific statistical families are complete. The six-arm mechanism design contains 96 runs; the remaining arms complete the calibrated-baseline, replay-budget, schedule, and target-allocation comparisons. Table~\ref{tab:causal-all-groups} lists every group; Table~\ref{tab:causal-registered} reports all 22 prespecified contrasts. Previously inspected results were not overwritten: all 196 completed result files and their associated scientific logs are byte-identical to the earlier snapshot.

All formal cells pair seeds 51001--51008. Two-sided sign-flip tests enumerate $2^8$ assignments; paired intervals use $t_{.975,7}$. Corrections are within model and family. The factorial family contains refresh at $n12$, coverage at $r8$, and their difference-in-differences. The baseline family contains I-GEM/PEGP/ARR minus supervised10 on Old generation. The schedule family contains front/back minus uniform on validation-NLL AUC. The hybrid family contains hybrid minus supervision at 2/10\% and their interaction. No family is reduced after observing which cells have finished. Budget curves, same-state contrasts, and individual New-NI intervals are supporting summaries, not additional corrected primary families.

\paragraph{Audit and loss definitions.}
The supplement adds a maximum qualified-generation drop of .10 to the original multimetric Old contract (NLL .20, margin 1.0, token accuracy .05, sequence exact .10, globally and by stratum). It also requires New acquisition. Status is assessed at the final budgeted endpoint, with no hidden-audit checkpoint search. NI is instead a paired method-versus-plain statement with the full-precision lower 95\% bound above $-.03$; a failed NI check is not proof of harm.

The old controller is partitioned deterministically into a protection pool of 48 and validation sets of 50 Mistral/23 Qwen records. Diagnostic $E$ selects 12 validation records and never supplies optimization gradients. The nested $n12/n48$ probes are fixed across methods and seeds. $E$ is disjoint from the protection pool and contained in validation. Hidden audits and generation-qualified subsets are outside both. Every diagnostic record first averages its supervised tokens, followed by a record mean; validation NLL uses a supervised-token mean. The diagnostic decomposition and validation AUC consequently differ in both weighting and coverage.

\paragraph{AUC and exposure matching.}
Validation is measured at the common grid $0,16,\ldots,128$. We subtract each run's baseline and integrate the piecewise-linear curve against normalized consumed-New exposure; the plotting pipeline uses these actual grid values without an added interpolation grid. All replay schedules share 2,048 ordered New records, 128 updates, and the same ordered list of 228 Old records. Front/back allocate this list to the first/second 64 updates. Exposure-plan hashes verify the within-seed match. Fractions of 2/5/10/20\% correspond to $42/108/228/512$ Old records after integer allocation; the plotted fractions use these exact counts.

\paragraph{Positive-part loss exposure (post-hoc supporting analysis).}
\label{app:positive-exposure}
To separate above-baseline loss exposure from net loss improvement, we additionally compute
\[
 \mathrm{AUC}_{+}=\int_0^1[\widehat L(s)-\widehat L(0)]_+\,ds,
 \qquad
 D_{\max}^{\mathrm{obs}}=\max_j\{L_j-L_0\},
\]
where $\widehat L$ is the piecewise-linear interpolation of the recorded validation NLL and $s$ is normalized consumed-New exposure. We integrate separately within each run, splitting intervals at zero crossings, before summarizing the eight seeds. The maximum includes the baseline observation. The grid is unchanged: updates $0,16,\ldots,128$, corresponding to $0,256,\ldots,2048$ New exposures. These are post-hoc supporting measures of sampled, token-weighted validation loss, distinct from both the diagnostic record-mean loss and intermediate generation.

Table~\ref{tab:positive-exposure} reports all three schedules and the four front/back-versus-uniform contrasts. Holm correction for this post-hoc analysis covers these four contrasts jointly across the two models, separately from the frozen schedule-AUC family. Peak degradation is reported descriptively. Delayed replay increases positive-part exposure in both models; front and uniform replay have zero positive-part area on the recorded grid, so their signed-AUC difference concerns below-baseline improvement rather than reduced above-baseline exposure. Zero area on this grid does not establish preservation at unobserved updates, within every stratum, or in generated behavior.

\begin{table}[htbp]
\centering\small\setlength{\tabcolsep}{3pt}
\caption{Post-hoc loss-exposure summaries on the original validation grid. Panel A reports means of per-seed signed AUC, positive-part AUC, and maximum observed NLL increase (nats/token, with exposure normalized for AUC). Panel B reports all four newly specified schedule contrasts on positive-part AUC (U: uniform), with unadjusted paired 95\% $t$ intervals and Holm correction jointly across both models. Intervals describe seed variation on fixed validation records. Peak degradation is descriptive.}
\label{tab:positive-exposure}
\begin{tabular}{llrrr}
\toprule
\multicolumn{5}{l}{\textit{A. Schedule summaries (eight seeds per cell)}}\\
Model & Schedule & Signed AUC & $\mathrm{AUC}_{+}$ & $D_{\max}^{\mathrm{obs}}$\\
\midrule
Mistral & Uniform & -0.549378 & 0.000000 & 0.000000 \\
Mistral & Front & -0.580445 & 0.000000 & 0.000000 \\
Mistral & Back & 0.045054 & 0.291597 & 1.005464 \\
Qwen & Uniform & -0.488842 & 0.000000 & 0.000000 \\
Qwen & Front & -0.537497 & 0.000000 & 0.000000 \\
Qwen & Back & -0.238135 & 0.010666 & 0.034878 \\
\bottomrule
\end{tabular}
\par\smallskip
\begin{tabular}{llrrrr}
\toprule
\multicolumn{6}{l}{\textit{B. Complete post-hoc positive-part family}}\\
Model & Contrast & $\Delta\mathrm{AUC}_{+}$ & 95\% CI & Exact $p$ & Holm $p$\\
\midrule
Mistral & Front $-$ U & 0.000000 & [0.000000, 0.000000] & 1.0000000 & 1.00000 \\
Mistral & Back $-$ U & 0.291597 & [0.035361, 0.547832] & 0.0078125 & 0.03125 \\
Qwen & Front $-$ U & 0.000000 & [0.000000, 0.000000] & 1.0000000 & 1.00000 \\
Qwen & Back $-$ U & 0.010666 & [0.005418, 0.015914] & 0.0078125 & 0.03125 \\
\bottomrule
\end{tabular}
\end{table}

\paragraph{Supporting New-generation non-inferiority.}
Table~\ref{tab:new-ni-key} reports the paired New-generation differences for the key replay and geometry comparisons. The criterion compares the full-precision lower endpoint of the two-sided 95\% paired $t$ interval with $-.03$, separately from per-run \texttt{FEASIBLE}. These are supporting, unadjusted comparisons on the fixed audit questions. The frozen analysis also contains the complete 60 New-NI records: 30 generation and 30 teacher-forced token comparisons.

\begin{table}[htbp]
\centering\small
\caption{Supporting New-generation non-inferiority for key method--plain comparisons, using all eight paired seeds. Differences and unadjusted two-sided 95\% $t$ intervals are in percentage points. NI is established when the full-precision lower bound exceeds $-3$ points; it is separate from per-run \texttt{FEASIBLE}. Not established does not mean demonstrated harm.}
\label{tab:new-ni-key}
\begin{tabular}{llrrl}
\toprule
Model & Method $-$ Plain & $\Delta$New (pp) & 95\% CI (pp) & NI\\
\midrule
Mistral & Replay10 & +1.15 & [-2.88, +5.19] & Established \\
Mistral & $r1,n48$ & -5.90 & [-20.83, +9.02] & Not established \\
Mistral & I-GEM & -0.04 & [-2.88, +2.79] & Established \\
Qwen & Replay10 & +0.29 & [-0.83, +1.42] & Established \\
Qwen & $r1,n48$ & -0.04 & [-0.79, +0.71] & Established \\
Qwen & I-GEM & -0.23 & [-0.70, +0.24] & Established \\
\bottomrule
\end{tabular}
\end{table}

\begingroup
\small
\setlength{\tabcolsep}{4pt}
\renewcommand{\arraystretch}{1.08}
\begin{longtable}{llrrrrr}
\caption{All formal method groups in the completed causal supplement (256/256 runs). Old and New are mean final generation exact (percent); AUC is baseline-centered, exposure-normalized validation NLL (nats/token). F is the number satisfying the complete audit, out of the available $n$. Every scientific outcome is retained; each group contains the same eight registered seeds.}%
\label{tab:causal-all-groups}\\
\toprule
Model & Method & $n$ & Old (\%) & New (\%) & AUC & F/$n$ \\
\midrule
\endfirsthead
\multicolumn{7}{c}{Table~\thetable\ continued}\\
\toprule
Model & Method & $n$ & Old (\%) & New (\%) & AUC & F/$n$ \\
\midrule
\endhead
\midrule
\multicolumn{7}{r}{Continued on next page}\\
\endfoot
\bottomrule
\endlastfoot
Mistral & Plain SFT & 8 & 3.26 & 92.50 & 0.7402 & 0/8 \\
Mistral & Supervised 10\% & 8 & 98.37 & 93.66 & -0.5494 & 8/8 \\
Mistral & Adam projection $r=8,n=12$ & 8 & 3.80 & 92.53 & 0.4350 & 0/8 \\
Mistral & Adam projection $r=1,n=12$ & 8 & 3.80 & 93.43 & 0.2879 & 0/8 \\
Mistral & Adam projection $r=8,n=48$ & 8 & 4.35 & 92.36 & 0.4070 & 0/8 \\
Mistral & Adam projection $r=1,n=48$ & 8 & 4.35 & 86.60 & 0.2876 & 0/8 \\
Mistral & I-GEM adaptation & 8 & 96.20 & 92.46 & -0.3416 & 7/8 \\
Mistral & PEGP adaptation & 8 & 78.80 & 94.16 & 0.0085 & 1/8 \\
Mistral & ARR adaptation & 8 & 72.83 & 90.62 & 0.0012 & 0/8 \\
Mistral & Supervised 2\% & 8 & 84.24 & 93.51 & -0.2989 & 4/8 \\
Mistral & Supervised 5\% & 8 & 97.28 & 93.30 & -0.4725 & 8/8 \\
Mistral & Supervised 20\% & 8 & 98.91 & 91.44 & -0.6066 & 8/8 \\
Mistral & 10\% front-half & 8 & 96.74 & 90.41 & -0.5804 & 5/8 \\
Mistral & 10\% back-half & 8 & 97.83 & 91.02 & 0.0451 & 8/8 \\
Mistral & Hybrid 2\% & 8 & 79.89 & 93.03 & -0.2162 & 3/8 \\
Mistral & Hybrid 10\% & 8 & 95.11 & 90.66 & -0.4621 & 7/8 \\
\midrule
Qwen & Plain SFT & 8 & 99.26 & 94.74 & 0.0234 & 3/8 \\
Qwen & Supervised 10\% & 8 & 100.00 & 95.04 & -0.4888 & 8/8 \\
Qwen & Adam projection $r=8,n=12$ & 8 & 100.00 & 95.10 & -0.0036 & 4/8 \\
Qwen & Adam projection $r=1,n=12$ & 8 & 100.00 & 94.18 & -0.0075 & 7/8 \\
Qwen & Adam projection $r=8,n=48$ & 8 & 100.00 & 94.07 & -0.0167 & 5/8 \\
Qwen & Adam projection $r=1,n=48$ & 8 & 100.00 & 94.70 & -0.0291 & 6/8 \\
Qwen & I-GEM adaptation & 8 & 100.00 & 94.51 & -0.3400 & 8/8 \\
Qwen & PEGP adaptation & 8 & 97.79 & 94.37 & -0.0007 & 8/8 \\
Qwen & ARR adaptation & 8 & 99.26 & 94.77 & -0.0028 & 8/8 \\
Qwen & Supervised 2\% & 8 & 100.00 & 94.64 & -0.2513 & 8/8 \\
Qwen & Supervised 5\% & 8 & 100.00 & 94.93 & -0.4120 & 8/8 \\
Qwen & Supervised 20\% & 8 & 100.00 & 94.85 & -0.5597 & 8/8 \\
Qwen & 10\% front-half & 8 & 100.00 & 95.27 & -0.5375 & 8/8 \\
Qwen & 10\% back-half & 8 & 100.00 & 94.41 & -0.2381 & 8/8 \\
Qwen & Hybrid 2\% & 8 & 97.79 & 94.93 & -0.1539 & 7/8 \\
Qwen & Hybrid 10\% & 8 & 100.00 & 94.62 & -0.3760 & 8/8 \\
\end{longtable}
\endgroup

\begingroup
\small
\setlength{\tabcolsep}{4pt}
\renewcommand{\arraystretch}{1.10}
\begin{longtable}{llrrrr}
\caption{Every registered causal-supplement contrast, with paired mean, unadjusted 95\% $t$ CI and exact sign-flip Holm $p$. Negative AUC differences favor the first condition; positive generation differences favor it. Each contrast uses all eight registered paired seeds; Holm correction is applied within the frozen model/family.}%
\label{tab:causal-registered}\\
\toprule
Model & Contrast & $n$ & Mean & 95\% CI & Holm $p$ \\
\midrule
\endfirsthead
\multicolumn{6}{c}{Table~\thetable\ continued}\\
\toprule
Model & Contrast & $n$ & Mean & 95\% CI & Holm $p$ \\
\midrule
\endhead
\midrule
\multicolumn{6}{r}{Continued on next page}\\
\endfoot
\bottomrule
\endlastfoot
\multicolumn{6}{l}{\textit{Refresh and coverage: validation-NLL AUC}}\\
Mistral & Refresh: $r1-r8$ at $n12$ & 8 & -0.1471 & [-0.3464, 0.0522] & 0.4219 \\
Mistral & Coverage: $n48-n12$ at $r8$ & 8 & -0.0280 & [-0.2713, 0.2153] & 1.0000 \\
Mistral & Refresh $\times$ coverage & 8 & 0.0277 & [-0.1477, 0.2030] & 1.0000 \\
Qwen & Refresh: $r1-r8$ at $n12$ & 8 & -0.0039 & [-0.0102, 0.0023] & 0.1953 \\
Qwen & Coverage: $n48-n12$ at $r8$ & 8 & -0.0132 & [-0.0198, -0.0065] & 0.0234 \\
Qwen & Refresh $\times$ coverage & 8 & -0.0084 & [-0.0168, -0.0001] & 0.1094 \\
\addlinespace[4pt]
\multicolumn{6}{l}{\textit{External baselines: Old generation difference (pp)}}\\
Mistral & I-GEM minus replay10 & 8 & -2.17 & [-4.92, 0.57] & 0.2188 \\
Mistral & PEGP minus replay10 & 8 & -19.57 & [-32.31, -6.82] & 0.0625 \\
Mistral & ARR minus replay10 & 8 & -25.54 & [-34.33, -16.76] & 0.0234 \\
Qwen & I-GEM minus replay10 & 8 & 0.00 & [0.00, 0.00] & 1.0000 \\
Qwen & PEGP minus replay10 & 8 & -2.21 & [-4.75, 0.34] & 0.7500 \\
Qwen & ARR minus replay10 & 8 & -0.74 & [-2.47, 1.00] & 1.0000 \\
\addlinespace[4pt]
\multicolumn{6}{l}{\textit{Replay timing: validation-NLL AUC}}\\
Mistral & Front minus uniform & 8 & -0.0311 & [-0.0503, -0.0118] & 0.0156 \\
Mistral & Back minus uniform & 8 & 0.5944 & [0.3287, 0.8602] & 0.0156 \\
Qwen & Front minus uniform & 8 & -0.0487 & [-0.0617, -0.0356] & 0.0156 \\
Qwen & Back minus uniform & 8 & 0.2507 & [0.2355, 0.2659] & 0.0156 \\
\addlinespace[4pt]
\multicolumn{6}{l}{\textit{Replay target: validation-NLL AUC}}\\
Mistral & Hybrid minus supervised, 2\% & 8 & 0.0828 & [0.0477, 0.1178] & 0.0234 \\
Mistral & Hybrid minus supervised, 10\% & 8 & 0.0873 & [0.0650, 0.1095] & 0.0234 \\
Mistral & Target $\times$ ratio & 8 & 0.0045 & [-0.0463, 0.0553] & 0.8359 \\
Qwen & Hybrid minus supervised, 2\% & 8 & 0.0974 & [0.0738, 0.1209] & 0.0234 \\
Qwen & Hybrid minus supervised, 10\% & 8 & 0.1128 & [0.0987, 0.1270] & 0.0234 \\
Qwen & Target $\times$ ratio & 8 & 0.0155 & [-0.0094, 0.0404] & 0.1875 \\
\end{longtable}
\endgroup

\subsection{Original refresh--coverage endpoints and paths}

Table~\ref{tab:factorial} and Figure~\ref{fig:factorialpaths} retain the original refresh--coverage summaries. They concern the 256-run supplement, not the separate structure--location study in Appendix~\ref{app:structure-location}.

\begin{table}[htbp]
\centering\small
\caption{Original refresh--coverage factorial endpoints, eight seeds per cell. Old and New are reference-length greedy exact percentages; F counts the full Old/New contract. AUC is signed, baseline-centered validation NLL. All projection updates satisfy their declared actual-displacement tolerance.}
\label{tab:factorial}
\begin{tabular}{lrrrcrrrc}
\toprule
& \multicolumn{4}{c}{Mistral} & \multicolumn{4}{c}{Qwen}\\
Method & Old & New & AUC & F/8 & Old & New & AUC & F/8\\
\midrule
Plain & 3.26 & 92.50 & .7402 & 0 & 99.26 & 94.74 & .0234 & 3\\
Replay 10\% & 98.37 & 93.66 & $-.5494$ & 8 & 100.00 & 95.04 & $-.4888$ & 8\\
$r8,n12$ & 3.80 & 92.53 & .4350 & 0 & 100.00 & 95.10 & $-.0036$ & 4\\
$r1,n12$ & 3.80 & 93.43 & .2879 & 0 & 100.00 & 94.18 & $-.0075$ & 7\\
$r8,n48$ & 4.35 & 92.36 & .4070 & 0 & 100.00 & 94.07 & $-.0167$ & 5\\
$r1,n48$ & 4.35 & 86.60 & .2876 & 0 & 100.00 & 94.70 & $-.0291$ & 6\\
\bottomrule
\end{tabular}
\end{table}

The diagnostic measurements identify what the halfspace leaves uncontrolled. For Mistral $r1,n12$, the eight-seed mean path sums are $-.5213$ for the current/cached term, approximately zero staleness, $+.4634$ probe mismatch, and $+.3282$ finite-step residual, yielding $+.2703$ summed diagnostic-loss change. At $r1,n48$, summed diagnostic change is near zero ($-.0231$, CI $[-.1601,.1139]$), yet final qualified generation is 4.35\% and the broader hidden-audit NLL increases by .3234. The small-set average objective and the sequence-level contract thus identify different properties of the served model.

\begin{figure}[htbp]
\centering
\includegraphics[width=\linewidth]{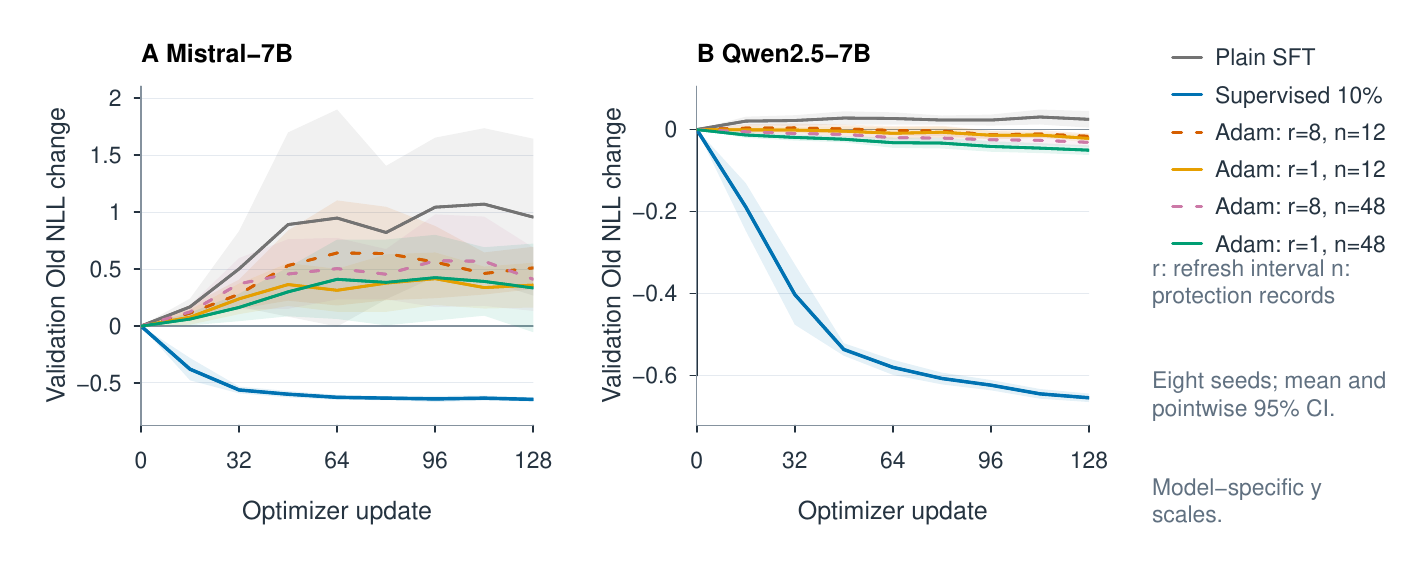}
\caption{Original refresh--coverage trajectories. Curves are eight-seed mean validation-NLL changes at the recorded steps, with pointwise 95\% $t$ intervals. Validation and the 12-record diagnostic loss have different averaging and coverage. Model-specific axes expose changes beneath Qwen's generation ceiling; they are not an architecture comparison.}
\label{fig:factorialpaths}
\end{figure}

Replay illustrates why the net path matters. Its Mistral cached-term sum is $-6.7412$, accompanied by positive staleness $3.7386$, mismatch $.2808$, and residual $1.1578$, for a favorable net sum $-1.5641$. These positive terms exceed their plain-SFT counterparts in several cases. Replay need not shrink every diagnostic component to protect the function. This observation matches the restoring-direction account in Proposition~\ref{prop:replay}, without identifying a causal mediation fraction from different visited paths.

At the same Mistral state and candidate update, refreshing the cached gradient reduces the immediate change in $E$ by $.01391$ at $n12$ (CI $[-.02055,-.00728]$) and $.01272$ at $n48$ ($[-.01869,-.00675]$). Both effects have eight negative seed-level differences. Qwen's intervals include zero; the $r1$ branches agree exactly (Figure~\ref{fig:localrefresh}). On Qwen, increasing $n$ from 12 to 48 at $r8$ improves validation-NLL AUC by $.01315$ (CI $[-.01984,-.00647]$, Holm $p=.02344$). Mistral's corresponding estimate is $-.02801$ ($[-.27133,.21531]$, $p=1$). Faster refresh at $n12$ has estimates $-.14711$ on Mistral and $-.00392$ on Qwen, but neither passes its correction; neither model's interaction passes.
\FloatBarrier

\subsection{Baseline specification and calibration}
\label{app:causal-methods}

The comparisons adapt each method to the same rank-16 QLoRA factors, 2,048 New exposures, 128 updates, frozen banks, and New task. They do not reproduce the source papers' original benchmark settings. Calibration seeds 50001/50002 have no hidden-audit outputs. Candidates must acquire New on both development seeds and have no numerical failures; ranking then maximizes feasible development outcomes, minimizes mean old risk, maximizes New-development token accuracy, and breaks ties by candidate index. The selected parameters and the 36 development-result hashes are frozen before confirmation.

\paragraph{I-GEM.}
We implement the nonnegative dual projected-gradient solve in Equation 3 and Algorithms 1--2 of \citet{tekmen2026igem}, with current gradients for each old stratum in the 48-record pool. Rows are normalized, the small Gram matrix uses an exact eigenspectrum in place of power iteration, and the step size is $.9/\lambda_{\max}$. The dual vector is warm-started; a grid of 4/16/64 iterations selects 16 for both models. The resulting gradient is passed to AdamW, without a post-Adam correction. Thus this condition differs from the single aggregate actual-update halfspace in constraint multiplicity and optimizer interaction. It is a finite-iteration approximation, not an exact quadratic-program feasibility claim; the observed minimum Mistral gradient residual is approximately $-5.76\times10^{-7}$. The source's GPT-2/LoRA setting and task memories are replaced by this experiment's QLoRA model and stratum-defined constraints.

\paragraph{PEGP.}
The adaptation follows the two-factor construction in Section III-C, Equations 22--24 of the arXiv v3 text of \citet{qiao2024pegp}: LoRA $A$ gradients are projected against input $x$, and $B$ gradients against $Ax$. The initial 48-record protection pool supplies four deterministic valid token positions per record. Compact right-singular bases are frozen at the old-task boundary; every trainable factor must have a basis. The energy grid $.90/.95/.99$ selects $.99$ for Mistral and $.90$ for Qwen. The original vision evaluation is replaced by all-linear-module LLM QLoRA. An earlier equation-number comment in the frozen implementation refers to \emph{Eq.\ 35}; the equations cited here identify the actual v3 construction without modifying the experimental code lock.

\paragraph{Approximate regularized replay (ARR).}
Following \citet{riemer2025arr}, ARR adds reverse KL, $\mathrm{KL}(p_\theta\Vert p_{\rm base})$, on New-input next-token positions and next-token training on generic OpenWebText. The teacher uses the same quantized base with adapters disabled; it is not an independently trained full-precision teacher. The generic stream has 228 exposures for 2,048 New exposures, approximately $.111$ generic records per New record. This budget differs from the source's $1{\times}/3{\times}/7{\times}$ ratios. Our New-KL grid $.01/.1/1$ selects 1 for both models, whereas the source grid includes $.001/.01/.1$. These differences are part of the shared-budget adaptation; its outcome does not rank ARR across its original operating range.

\paragraph{Hybrid targets.}
The Old loss is $.5\,{\rm CE}+.5\,T^2{\rm KL}(p_{\rm base}\Vert p_\theta)$ on answer positions with $T=2$. Half of the original supervised objective mass is replaced, not retained with an added KL penalty. The two-model AUC contrast therefore identifies the effect of this target allocation. It cannot establish that an incremental KL term would harm an otherwise identical supervised objective. Equal nominal coefficient sums do not imply equal gradient strength; the teacher-KL gradient vanishes at exact student--teacher agreement, whereas the hard-label CE gradient generally does not.

\subsection{Numerical execution and empirical decomposition}

Across the refresh--coverage supplement's 64 formal projection runs, all 8,192 actual displacements satisfy
$h^\top u\leq\max(10^{-7},10^{-3}|h^\top v|)$.
There are 2,815 positive candidate dots on Mistral and 1,924 on Qwen; 16 Mistral steps require an FP32 rounding correction and none do on Qwen. All 640 same-state counterfactual branches also pass. Tolerance feasibility refers to computed gradients and stored displacements, not the ideal real-arithmetic inequality $h^\top u\leq0$. No supplementary run has a recorded numerical failure. The full-matrix completion receipt passes with 256 expected and observed outcomes and no integrity errors.

The 96 mechanism runs provide 12,288 before/after diagnostic records. The maximum algebraic decomposition closure error is $5.56\times10^{-17}$. Table~\ref{tab:causal-paths} reports sums of measured per-update increments and their terms. The finite-step residual includes numerical evaluation and gradient error as in Equation~\ref{eq:measured-remainder}; it is not a measured Hessian norm or a fraction of forgetting causally attributable to curvature.

\begin{table}[htbp]
\centering\small
\caption{Mean diagnostic path sums over eight seeds. $\Delta E$ is the sum of measured within-update loss increments; the other columns add to it up to rounding. $H$: cached dot; $S$: staleness; $A$: diagnostic/probe mismatch; $R$: finite-step residual. These are record-mean diagnostic NLL quantities, not validation AUC or hidden-generation effects.}
\label{tab:causal-paths}
\begin{tabular}{llrrrrr}
\toprule
Model & Method & $\Delta E$ & $H$ & $S$ & $A$ & $R$\\
\midrule
Mistral & Plain & 1.5597 & .4220 & .8423 & $-.0367$ & .3320\\
& Replay10 & $-1.5641$ & $-6.7412$ & 3.7386 & .2808 & 1.1578\\
& $r8,n12$ & .6394 & $-2.1041$ & 1.9386 & .3312 & .4737\\
& $r1,n12$ & .2703 & $-.5213$ & .0000 & .4634 & .3282\\
& $r8,n48$ & .4511 & $-.9352$ & .8747 & .2311 & .2804\\
& $r1,n48$ & $-.0231$ & $-.6029$ & .0000 & .2378 & .3420\\
\midrule
Qwen & Plain & $-.0268$ & $-.0499$ & .0450 & $-.0260$ & .0041\\
& Replay10 & $-.6063$ & $-.6018$ & .1608 & $-.2194$ & .0541\\
& $r8,n12$ & $-.0921$ & $-.1227$ & .0378 & $-.0135$ & .0062\\
& $r1,n12$ & $-.1011$ & $-.0979$ & .0000 & $-.0091$ & .0059\\
& $r8,n48$ & $-.1054$ & $-.0752$ & .0245 & $-.0607$ & .0061\\
& $r1,n48$ & $-.1362$ & $-.0670$ & .0000 & $-.0735$ & .0042\\
\bottomrule
\end{tabular}
\end{table}
\FloatBarrier

Adjacent evaluations also reveal two continuity gaps: $-.0012959043$ after update 56 of Mistral plain seed 51001, and $-.0048195322$ after update 16 of Mistral $r8,n12$ seed 51007. The other 94 mechanism traces have no nonzero recorded gap. The corresponding group-mean endpoint-minus-increment-sum differences are $-.0001619880$ and $-.0006024415$. The cause is not identified from the downloaded logs; these gaps are retained using Equation~\ref{eq:measurement-gaps}, not silently assigned to curvature.

\paragraph{Same-state branches.}
At each registered checkpoint, cached and fresh projections start from the same parameters and candidate Adam displacement; branch evaluation restores the training state. Five effects are averaged within each seed before the eight-seed interval. For $r8$, all five branch points have age seven. Mistral's $n12/n48$ unadjusted sign-flip probabilities are both $.0078125$; Qwen's are $.71875/.4921875$. The $r1$ fresh/cached branches agree exactly. Full observed trajectories cover ages zero through seven, but those age-conditioned observations are not randomized age interventions.

\begin{figure}[htbp]
\centering
\includegraphics[width=\linewidth]{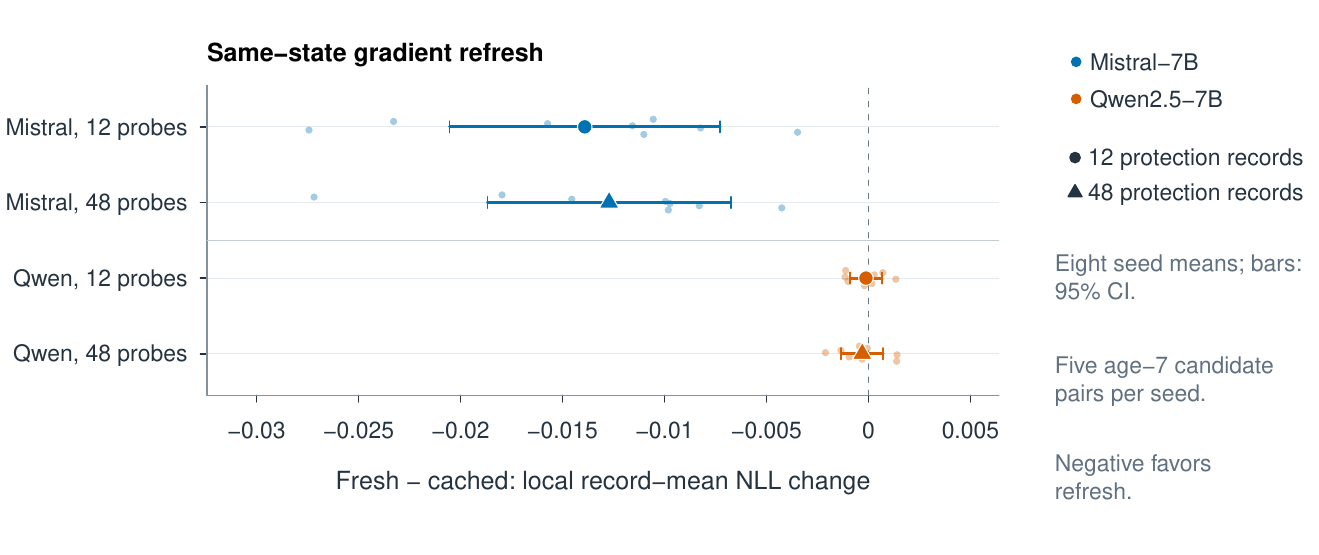}
\caption{Same-state refresh intervention in the $r8$ conditions. Each seed first averages the five registered branch points; points and bars then summarize eight seeds with 95\% paired $t$ intervals. Negative values favor fresh projection on diagnostic $E$. These age-seven, one-step effects are supporting local interventions, not full-path treatment effects.}
\label{fig:localrefresh}
\end{figure}

\paragraph{Timing, provenance, and rendering.}
Timers separately record New/Old objectives, protection initialization, projection, diagnostics, and evaluations. All 256 formal runs are ineligible for uninterrupted exclusive-GPU timing; concurrent invocation time is not used to rank method efficiency. Independent checks match all result/configuration hashes, 23 frozen scientific-source hashes, and 16 input-file hashes. The six main splits have disjoint identifiers and group identifiers within each model; generation-qualified audit records are a subset of the corresponding hidden audit. All 32 group AUC summaries, 22 registered paired contrasts, and 60 supporting New-NI summaries are recomputed from seed-level outcomes. The supplement's three figures use base R 4.6.1 in R.app, vector PDFs and 300-dpi PNGs, with right-side legend panels. Pointwise bands are not simultaneous intervals. The original result files, training code, scheduling state, and prespecified gates are unchanged by this manuscript revision.

The original 256-run archive contains baseline/final generation aggregates and intermediate teacher-forced metrics. Per-example generated sequences and intermediate generation audits were not recorded for that study. Final adapters were not retained by its protocol, and its resumable scratch state was removed after completion; the archived logs support the derived loss-path analysis reported here. The separate structure--location study records per-example generation at fixed intermediate checkpoints and retains four adapter checkpoints per run; its artifacts and verification scope are described in Appendix~\ref{app:structure-location}.

\FloatBarrier

\section{Three-Model Scale Evidence: Local Geometry and the Retained Function}
\label{app:scale-results}

\subsection{A satisfied cached constraint can accompany severe forgetting}
\label{sec:case}

The Mistral pure-Adam-projection run with seed 1101 provides a direct test of the implication from cached safety to retention. Projection executes throughout its 128 updates, and all cached post-update dots meet the original tolerance. Nevertheless, the final checkpoint generates none of its 23 acquired old answers correctly, compared with all 23 at baseline. It acquires the new task, reaching $94.30\%$ new generation exact. This is a same-interface loss of previously demonstrated behavior.

At update 32, the cached gradient's dot product with the candidate update is $0.0083557963$. Projection reduces it to $2.6702764\times10^{-9}$. The fresh gradient, computed on the \emph{identical probe records} at the current pre-update state, has dot $0.1128072154$ with the \emph{same applied update}. Figure~\ref{fig:projection} shows the execution trace and the population of paired measurements. A nearly zero cached term in Equation~\ref{eq:decomp} therefore coexists with a positive current first-order term, without a probe-distribution change in this comparison.

\begin{figure}[t]
\centering
\includegraphics[width=\linewidth]{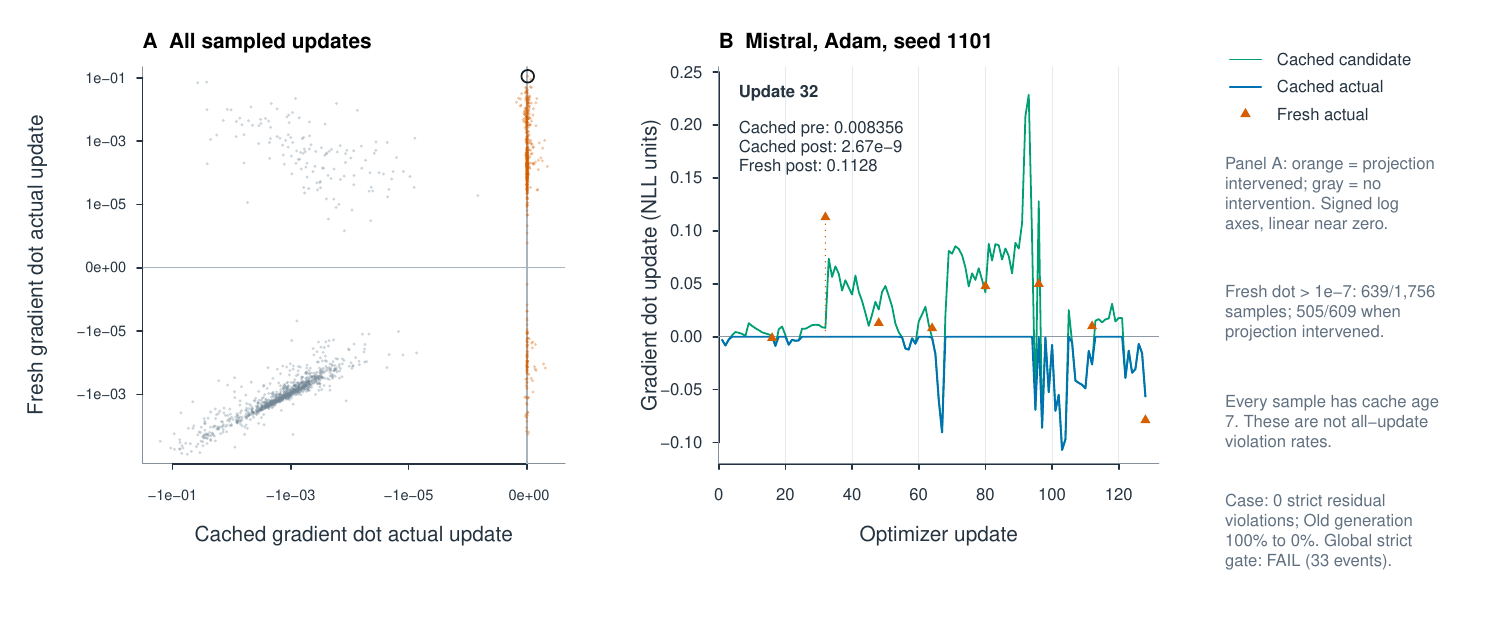}
\caption{Cached and current geometry on the same realized update. Left: all 1,756 paired diagnostics; axes use a signed-log transform and original-unit labels. All samples have cache age seven. Right: a Mistral run satisfying every cached numerical tolerance, with the update-32 discrepancy highlighted. Fresh dots are current first-order probe quantities, not observed finite-step audit-loss differences. The global 33-event numerical failure is reported separately from this strict-safe case.}
\label{fig:projection}
\end{figure}

Across 152 Adam-projection runs, 11,143 of 29,047 updates trigger projection. Among 1,756 sampled updates, 639 have fresh dots above $10^{-7}$; among the 609 sampled steps where projection intervenes, 505 still have positive fresh dots above that threshold. These counts describe diagnostics at the oldest cache age, not a uniform all-update failure rate. They show that the measured current-direction term can escape the cached constraint. They do not isolate staleness as the sole cause of final forgetting: curvature, finite coverage, and the relation between mean loss and per-item generation remain distinct terms or measurement gaps.

\paragraph{Method-level comparison.}
In the standard Mistral phase, both plain SFT and pure Adam projection lose $97.28$ generation points on average. Pure Adam projection does not establish an old-generation gain in any of the three models in either the standard or equal-MTE phase. The registered Adam-minus-raw-projection contrast is $-81.5$ points on Mistral, with 95\% CI $[-86.9,-76.1]$ and Holm $p=.0234$. Raw projection itself includes a new-acquisition failure in this phase; this is a comparison of old endpoints, not a claim that raw projection solves the joint problem. The informative conclusion is that optimizing the local cached-update criterion does not order final retention quality.

\subsection{Fixed replay protects behavior without freezing its implementation}

Figure~\ref{fig:geometry} compares fixed supervised replay with plain SFT. Old-token retention improves while the measured adapter-factor displacement increases and final last-layer CKA decreases. Thus the protected function need not use an unchanged parameter or representation configuration. These diagnostics concern trainable LoRA $A/B$ factors and a fixed probe readout, not the full effective weight or all representations. Their role is to test a monotone ``less movement means less forgetting'' account, not to claim that task-relevant representations are unimportant.

\begin{figure}[t]
\centering
\includegraphics[width=\linewidth]{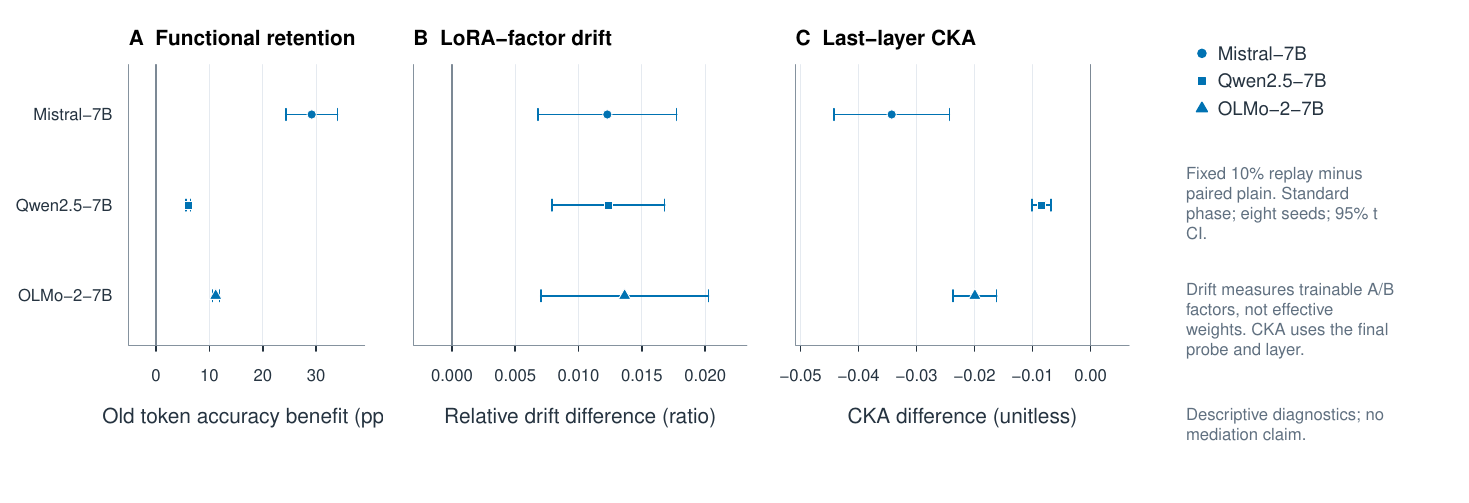}
\caption{Functional retention and measured geometric change can move in different directions. Points are fixed-replay-minus-plain effects over eight paired seeds with 95\% $t$ intervals. Old-token benefits use percentage points; factor drift and CKA differences are unitless. CKA is evaluated on the final diagnostic probe at the last layer, before the first answer token. These are descriptive comparisons, not a mediation analysis.}
\label{fig:geometry}
\end{figure}

The fixed-replay change in raw-gradient cosine has an interval crossing zero for every model, whereas its mean current-gradient/actual-update dot shifts downward in all three. The latter is a better-defined local functional measurement, but a negative method-minus-plain difference does not imply a negative absolute dot or prove a causal mediator. The distinction is exactly the one tested above: a relevant diagnostic is not automatically a sufficient control algorithm.

\subsection{Informative targets and the long-horizon mechanism}

Earlier controlled Pythia experiments intervene on replay's information and timing from common parameter/optimizer states. Moving old examples from an early block into the new-task phase, while preserving their total concrete budget, improves old exact by $0.670$ $[0.549,0.791]$ across six registered seeds. Source-normalized replay adds $0.021$ $[-0.129,0.171]$ over ordinary replay, leaving objective normalization without a stable incremental effect in that comparison.

Common-state five-step rollouts give an uncertain source-normalized protection estimate: mean $0.420$, with interval $[-0.693,1.533]$. At 100 updates, old-loss AUC protection is $6.668$ $[4.008,9.328]$ for source-normalized supervision and $7.067$ $[4.574,9.560]$ for teacher KL, compared with $0.510$ $[-1.895,2.916]$ for input-only replay. These target-information interventions support repeated functional supervision rather than old-input exposure as the useful payload in that setting. The complete trajectories and controls are in Appendix~\ref{app:historical}. A separate same-parameter Adam-state transplant gives an AUC effect $0.956$ $[-0.351,2.263]$, leaving a freely transferable optimizer-memory account unconfirmed.

Together, the theory and interventions support a specific explanation: new-only optimization leaves old-output directions unconstrained, local cached control can miss their current direction, and informative replay repeatedly supplies constraints on the old function as the model moves. They do not establish irreversible erasure of latent knowledge or a single universal source of SFT forgetting.

\section{Mitigation Across Exposure, Compute, and Horizon}
\label{sec:mitigation}

\subsection{The strongest generation effect survives compute matching}

Fixed 10\% supervision raises Mistral old generation relative to plain by $94.02$ percentage points in the matched-exposure phase (95\% CI $[90.69,97.35]$, Holm $p=.0234$). Under matched optimization MTE, the effect is again $94.02$ points ($[92.14,95.90]$, Holm $p=.0234$). Its benefit is therefore not fully explained by extra optimization MTE in this experiment. These are matched proxies rather than an exact hardware-FLOP comparison.

Qwen and OLMo have smaller plain-generation losses and positive fixed-replay point estimates, but their standard and equal-MTE generation contrasts do not pass the registered correction. Supporting token and NLL results measure additional functional changes rather than replacing the generation primary. Figure~\ref{fig:primary} makes the model-conditional scale and the phase distinctions explicit.

\begin{figure}[t]
\centering
\includegraphics[width=\linewidth]{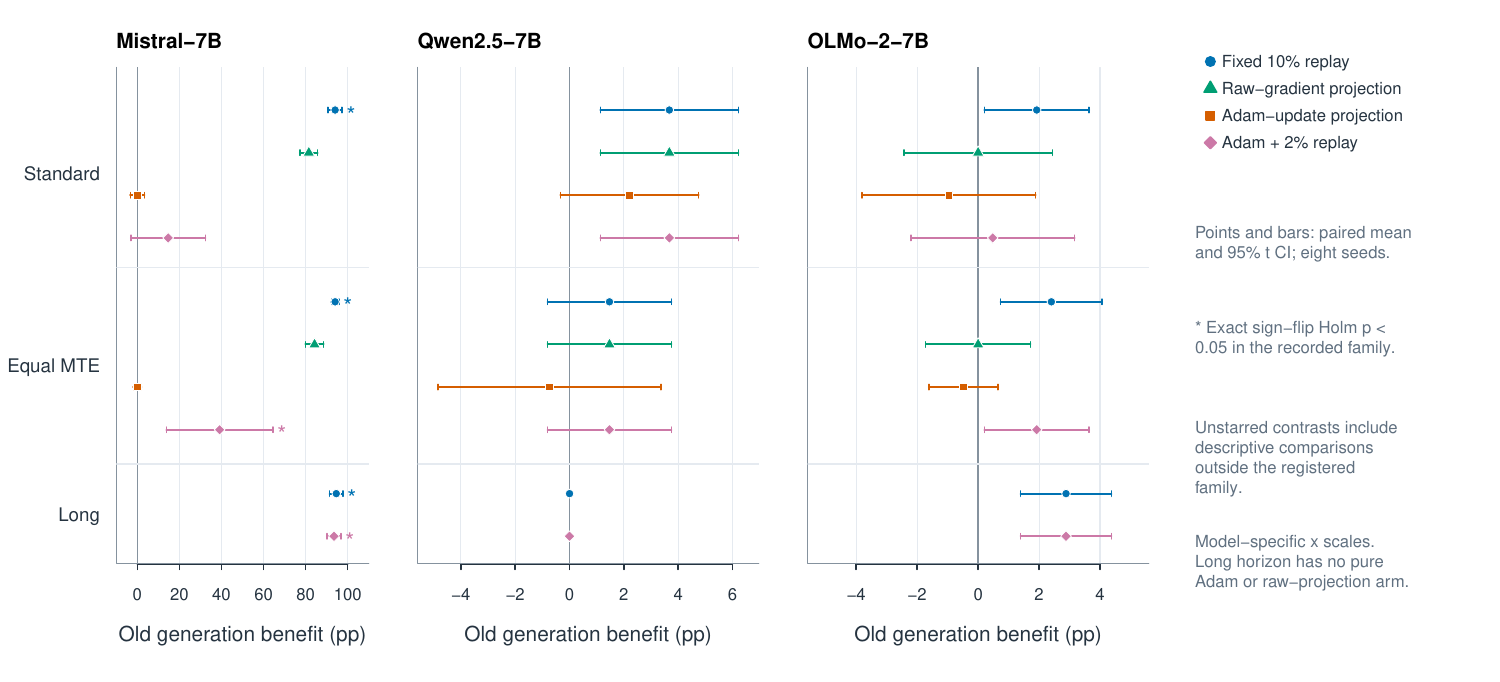}
\caption{Old-generation benefits relative to paired plain SFT. Each point is an eight-seed mean with a 95\% paired $t$ interval. Stars denote Holm $p<.05$ within the original generation family, which also includes unplotted MSSR in the standard phase. Raw projection and standard Adam-plus-replay are descriptive comparisons. Axes are model-specific; model banks differ. Long-horizon arms do not include pure Adam or raw projection.}
\label{fig:primary}
\end{figure}

\subsection{Long-horizon protection with new-task non-inferiority}

At 8,192 new exposures, fixed replay returns a feasible final checkpoint in all 24 runs. Table~\ref{tab:long} shows the absolute old generation, new generation, and full-audit outcomes together. For Mistral, fixed replay raises old generation by $94.57$ points ($[91.34,97.79]$, Holm $p=.015625$). Qwen is at a generation ceiling in all arms. OLMo's fixed-replay gain is $2.88$ points ($[1.40,4.37]$), with Holm $p=.0625$; that generation contrast is not confirmed after correction.

\begin{table}[t]
\centering\small
\caption{Long-horizon final endpoints (means of eight seeds). Old and New columns are free-generation exact percentages. F is the count passing the entire declared retention/new-acquisition audit, not just generation. Every old-generation baseline is 100\%.}
\label{tab:long}
\begin{tabular}{llrrc}
\toprule
Model & Method & Old (\%) & New (\%) & F / 8\\
\midrule
Mistral & Plain & 2.72 & 94.64 & 0\\
 & Fixed 10\% replay & 97.28 & 94.49 & 8\\
 & Adam + 2\% replay & 96.20 & 93.68 & 2\\
\midrule
Qwen & Plain & 100.00 & 95.10 & 2\\
 & Fixed 10\% replay & 100.00 & 95.27 & 8\\
 & Adam + 2\% replay & 100.00 & 94.39 & 8\\
\midrule
OLMo & Plain & 97.12 & 94.68 & 7\\
 & Fixed 10\% replay & 100.00 & 94.30 & 8\\
 & Adam + 2\% replay & 100.00 & 94.74 & 7\\
\bottomrule
\end{tabular}
\end{table}

Fixed replay passes the registered new-generation and new-token NI criteria in all three long-horizon model comparisons. The generation differences in percentage points are Mistral $-0.147$ $[-0.834,0.541]$, Qwen $+0.168$ $[-0.481,0.816]$, and OLMo $-0.377$ $[-1.877,1.124]$. Figure~\ref{fig:ni} places these intervals against the $-3$-point margin. This new long-horizon finding is phase-specific: Mistral in the standard phase and Mistral/Qwen in the equal-MTE phase do not establish NI. A non-inferiority test that does not pass is not proof of harm.

\begin{figure}[t]
\centering
\includegraphics[width=\linewidth]{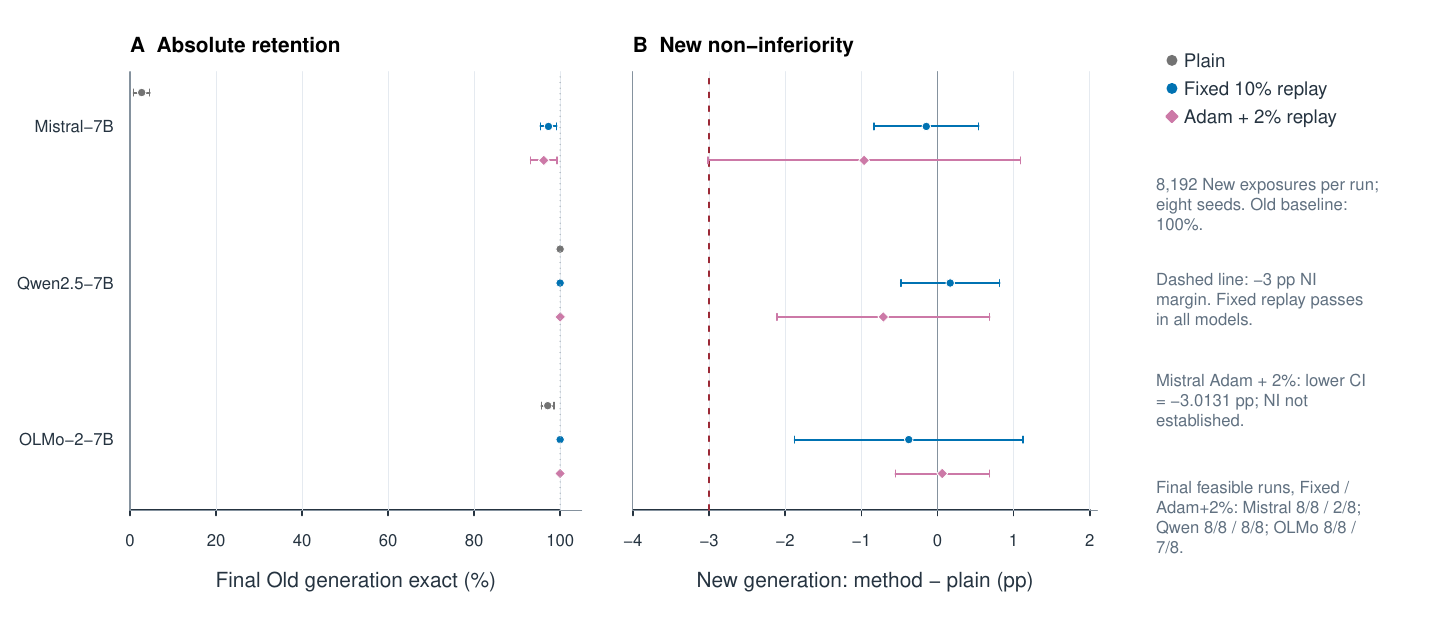}
\caption{Retention and plasticity after 8,192 new exposures. Left: mean final old generation, with marginal $t$ intervals clipped to the attainable 0--100\% range for display. Right: untruncated paired new-generation differences and the $-3$-point NI boundary. Fixed replay passes NI in all models; Mistral Adam plus 2\% has lower bound $-3.0130649$ points and does not. This combination has no standalone 2\% replay ablation.}
\label{fig:ni}
\end{figure}

The combined Adam-plus-2\% arm illustrates why the full contract matters. On Mistral it reaches $96.20\%$ old generation, but six of eight runs exceed the instruction-format margin budget, and its new-generation NI lower bound is $-3.0130649$ points. Qwen plain reaches 100\% old generation while six runs fail a formatting-margin budget; that domain is absent from Qwen's qualified generation subset. A small generation audit can thus remain correct while other declared old behavior deteriorates.

\subsection{Trajectory and bank evidence}

Fixed replay improves long-horizon old token accuracy over plain by $30.12$, $8.14$, and $15.23$ points for Mistral, Qwen, and OLMo, respectively, and lowers time-normalized controller Old-NLL AUC in all three. The trajectory curves are shown in Appendix Figure~\ref{fig:longcurves}. Final feasibility is not a claim of monotonic protection: Mistral fixed replay still has transient controller violations. Final audit, cumulative loss exposure, and per-update geometry measure different properties of the same path.

Across all four 7B phases, fixed replay has 104/104 feasible endpoints. In Qwen's balanced, factual-only, full, and without-factual bank settings, its old-token benefits are $5.7$, $8.0$, $6.3$, and $7.1$ points, each with within-bank two-method Holm $p=.015625$; all four satisfy new-task NI. These are supporting token contrasts. Factual-only has no qualified generation endpoint, and the bank study changes evaluation coverage along with bank composition, so it does not identify training-bank composition effects on a common broad hidden test.

The earlier 384-run matrix adds adaptation coverage: fixed replay improves the registered old-token endpoint in all nine Pythia/Qwen/TinyLlama by Full/LoRA/QLoRA cells, with eight positive paired seeds and within-model Holm $p=.0234$. That matrix's new-generation NI was not established. Its results support the portability of functional replay across parameterizations, not a retroactive three-model 7B generation claim (Appendix Figure~\ref{fig:legacy}).

\section{Model-Qualified Banks and Measurement}
\label{app:banks}

\begin{table}[htbp]
\centering\small
\caption{Old-bank coverage in the original 7B scale and refresh--coverage studies. Counts are records, not supervised tokens. Teacher-forced eligibility is assessed at the zero-step QLoRA start; generation includes acquired exact records. Factual (F), instruction format (I), multilingual lexical (M), and arithmetic (A) counts refer to the generation subset.}
\begin{tabular}{lrrrl}
\toprule
Model & Train & Controller & Audit & Generation audit composition\\
\midrule
Mistral & 487 & 98 & 96 & 23: 2 F, 17 I, 4 M\\
Qwen & 407 & 71 & 74 & 17: 1 F, 16 A\\
OLMo & 355 & 57 & 66 & 26: 3 F, 23 A\\
\bottomrule
\end{tabular}
\end{table}

\begin{figure}[htbp]
\centering
\includegraphics[width=\linewidth]{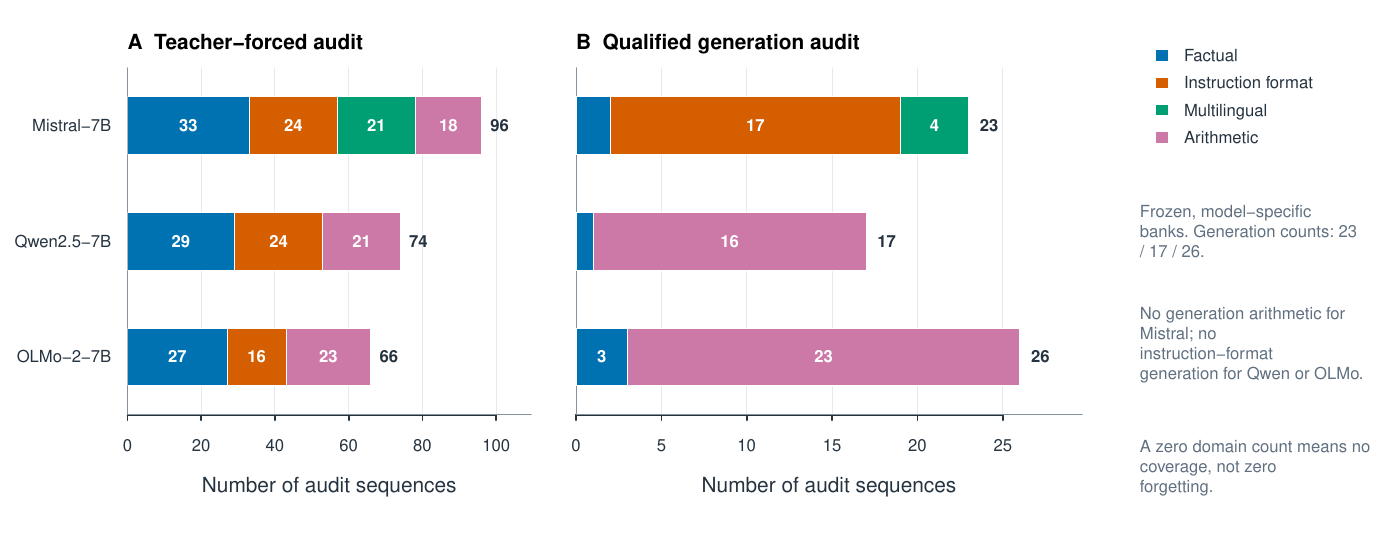}
\caption{Actual audit coverage by domain, redrawn from the frozen bank lock. Teacher-forced and generation evaluations have different coverage. A zero-count domain is unmeasured, not retained with zero loss. Different model banks preclude attributing effect-size differences solely to architecture.}
\label{fig:coverage}
\end{figure}

\paragraph{Qualification.}
The original 7B bank is qualified under the actual QLoRA instantiation, not a Full/LoRA/QLoRA intersection. Its per-record teacher-forced threshold is $0.6$. Generation eligibility is evaluated before SFT and requires at least 12 records overall; the subsequent acquisition gate is $0.99$ exact on the selected subset. In contrast, the earlier 384-run protocol uses an intersection across the three adaptation starts with a $0.5$ token threshold and an eight-record generation minimum. Those historical choices are not silently reused for the 7B matrix. The two additional strict-EOS banks in the structure--location study use the separate baseline qualification in Appendix~\ref{app:structure-location}.

\paragraph{Split roles.}
Six frozen splits distinguish new train/development/audit and old replay/controller/audit. The source manifests record identity/group separation and hashes. Projection and masks use a subset of controller records for gradients, so controller results are monitoring signals. Hidden audit is outside optimization and checkpoint choice. The qualified-generation subset is explicitly recorded and must not be replaced by a teacher-forced bank in interpreting a generation result. For the original 576-run package, the independent local audit verifies recorded provenance and split manifests rather than re-certifying the contents of 49 referenced upstream data paths. The complete causal supplement separately includes the 16 frozen model-specific input JSONL files, whose hashes and identity/group separation were checked (Appendix~\ref{app:causal}). These are distinct artifact checks.

\paragraph{Endpoint definitions.}
Teacher-forced token accuracy scores the correct next token on reference prefixes; sequence exact requires all supervised positions to be correct under teacher forcing. In the original studies and the Legacy endpoint of the structure--location study, reference-length greedy exact is token-ID exact match under greedy decoding. Each prompt is decoded for at most the reference answer's token length; trailing pad/EOS tokens are removed from both sequences before comparison. No case folding, whitespace normalization, or semantic matching is used. This reference-length interface does not evaluate additional continuation under a longer decoding budget. The structure--location study additionally measures fixed-cap strict-EOS generation as specified in Appendix~\ref{app:structure-location}. The answer margin is the correct-token logit minus the strongest competitor, aggregated according to the recorded metric. Controller Old-NLL AUC integrates baseline-centered NLL over normalized consumed-new-example progress. It is not a hidden-generation AUC or a proof that every update preserves the old function.

\Needspace{9\baselineskip}
\paragraph{Frozen task templates and reference examples.}
In the original studies, New records use SST-2 reviews with a frozen sentiment-to-code mapping. Their complete prompt template is:
\begin{quote}
\small\ttfamily\raggedright
Task: sentiment-code-v1.\par
Classify the sentiment of the movie review using the learned internal code. Return exactly one code token.\par
Review: \{review\}\par
Code:
\end{quote}
Here \texttt{\{review\}} denotes the stored review text, and leading spaces in the stored completions are part of the frozen input definition. In the original studies, these prompt/completion records are tokenized separately without added special tokens; prompt positions and an artificially appended EOS boundary are excluded from answer supervision. The separate structure--location protocol supervises EOS on New-task training examples in all three adaptation conditions, leaving Old replay/probe supervision and answer-only evaluation masks unchanged (Appendix~\ref{app:structure-location}).

The following are actual frozen generation-audit prompts and reference completions, not model predictions. Each shown Old prompt has one trailing space in the stored record.
\begin{description}
\raggedright
\item[Mistral instruction format.] Prompt: \texttt{Copy the value between brackets and answer exactly as TAG=<value>. Value=[7031928e]:}\newline
Reference: \texttt{TAG=7031928e}.
\item[Mistral multilingual.] Prompt: \texttt{In Spanish, how do you write 'water green'? Give only the translation:}\newline
Reference: \texttt{agua verde}.
\item[Qwen arithmetic.] Prompt: \texttt{Calculate 12 + 15. Answer with only the integer:}\newline
Reference: \texttt{27}.
\end{description}
\begingroup\raggedright
The input records retain their identifiers and provenance: New reviews identify \texttt{stanfordnlp/sst2}, factual items identify \texttt{zjunlp/KnowEdit}, and the controlled examples above identify a deterministic generator. The frozen JSONL files specify the actual mappings and references used for scoring; these examples do not add new audit items.\par
\endgroup

\paragraph{Retention budgets.}
The original 7B contract permits NLL increase $0.20$, answer-margin drop $1.0$, token-accuracy drop $0.05$, and sequence-exact drop $0.10$, checked globally and by stratum. Its New-acquisition rule requires token accuracy at least $0.50$ and improvement at least $0.05$. These budgets determine per-run status, while the registered $-0.03$ NI bound compares new generation and token outcomes across methods and seeds. The 7B long phase retains the same budgets, with controller evaluation every 32 updates instead of 16. The separately frozen, task-specific acquisition rules for the structure--location study are given in Appendix~\ref{app:structure-location}; they do not change historical gate outcomes.

\section{Statistical Families and Complete Results}
\label{app:statistics}

For paired seed differences $d_1,\ldots,d_8$, the two-sided sign-flip test enumerates all $2^8$ sign assignments and compares the absolute mean to the observed absolute mean. The smallest nonzero two-sided probability is $2/256=.0078125$. Confidence intervals are $\bar d\pm t_{.975,7}s_d/\sqrt8$ and describe seed variation on the fixed model and question set. Holm decisions and unadjusted intervals need not coincide: an interval may exclude zero while the multiplicity-corrected test does not pass.

\begin{table}[htbp]
\centering\small
\caption{Registered contrast families in the original 576-run scale study. Corrections are local to the stated family, not across the entire paper. All primary comparisons are method minus matched plain except the three head-to-head contrasts.}
\begin{tabular}{p{3.0cm}p{8.8cm}}
\toprule
Family & Members and correction scope\\
\midrule
Standard generation & Fixed 10\% replay; MSSR paper-spec; pure Adam projection. Three tests per model.\\
Equal-MTE generation & Fixed 10\%; pure Adam; Adam plus 2\%. Three tests per model.\\
Long generation & Fixed 10\%; Adam plus 2\%. Two tests per model.\\
Bank sensitivity & Fixed 10\%; Adam plus 2\%. Two tests within each Qwen bank; the reported token family is supporting, distinct from generation.\\
Mechanistic head-to-head & Adam minus raw projection; dynamic minus static mask; Adam plus 2\% minus fixed 10\%. Three tests per model.\\
\bottomrule
\end{tabular}
\end{table}

Every run contributes to its assigned method endpoint, including 156 retention failures and one new-acquisition failure. The latter is a Mistral raw-projection run in the standard phase and remains visible in the complete table. A favorable old-endpoint contrast does not establish new acquisition or the full contract. Supporting methods outside a registered generation family are not relabeled confirmatory by their unadjusted intervals.

\paragraph{New non-inferiority.}
The frozen analyzer uses the full-precision lower endpoint of the 95\% paired interval against $-0.03$. Fixed replay passes both new endpoints in the long phase for each model. Its standard Mistral comparison and equal-MTE Mistral/Qwen comparisons do not pass. Mistral long-horizon Adam plus 2\% has generation lower bound $-0.030130649384$, while its token lower bound is approximately $-0.028920815$; the former fails the criterion and the latter passes. We make the decision before rounding. These NI checks use the registered per-comparison rule, not an additional joint NI multiplicity correction.

\paragraph{Other method contrasts.}
Adaptive hybrid has 24/24 feasible endpoints in the standard phase, without an established registered advantage over fixed replay. Old-input teacher KL improves Mistral directionally but is heterogeneous across models; generic OpenWebText KL is not a substitute for the strong task-relevant supervision effect. Their inputs, masks, and target distributions differ, so this comparison alone does not isolate label information. MSSR is an independent adaptation of the sample-memory and scheduling equations of \citet{lu2026mssr}, using the experiment's common task, adapter, and optimizer. Dynamic masks do not establish a registered generation advantage over static masks.

The frozen decision helper emits \texttt{PARTIAL\_SUPPORT} because a permissive branch accepts positive old-token point estimates in at least two models. Its own generation confirmations are $0/3$ for pure Adam in the standard phase, $0/3$ in equal MTE, and $0/3$ positive Adam-versus-raw contrasts. The paper uses the actual endpoints and tests rather than interpreting that convenience label as confirmation of the projection hypothesis.

The refresh--coverage supplement's post-hoc positive-part analysis and key supporting New-generation NI summaries are reported separately in Appendix~\ref{app:causal}. The structure--location study has its own registered factorial family and descriptive task extensions (Appendix~\ref{app:structure-location}); neither is pooled with the scale-study families above.

\begingroup
\scriptsize
\setlength{\tabcolsep}{4pt}
\renewcommand{\arraystretch}{1.08}
\begin{longtable}{lllrrr}
\caption{All 72 model--strategy groups in the 7B study. Old and New are mean final generation exact (percent); F is the number satisfying all final retention budgets and new acquisition, out of eight seeds. Every run, including retention and acquisition failures, contributes to its mean. All endpoints use the frozen final-unconstrained policy. MTE matched denotes the optimized token-pass compute proxy; all models use rank-16 QLoRA.}%
\label{tab:scale-all-groups}\\
\toprule
Phase & Model / bank & Strategy & Old (\%) & New (\%) & F/8 \\
\midrule
\endfirsthead
\multicolumn{6}{c}{Table~\thetable\ continued}\\
\toprule
Phase & Model / bank & Strategy & Old (\%) & New (\%) & F/8 \\
\midrule
\endhead
\midrule
\multicolumn{6}{r}{Continued on next page}\\
\endfoot
\bottomrule
\endlastfoot
Mechanism & Mistral & Plain & 2.72 & 94.28 & 0/8 \\
Mechanism & Mistral & Fixed 10\% & 96.74 & 91.92 & 8/8 \\
Mechanism & Mistral & Retain KL & 79.35 & 93.51 & 3/8 \\
Mechanism & Mistral & Generic KL & 11.41 & 91.48 & 0/8 \\
Mechanism & Mistral & Adaptive & 96.20 & 92.17 & 8/8 \\
Mechanism & Mistral & MSSR & 4.89 & 92.09 & 0/8 \\
Mechanism & Mistral & Static mask & 74.46 & 94.30 & 1/8 \\
Mechanism & Mistral & Dynamic mask & 71.20 & 92.96 & 0/8 \\
Mechanism & Mistral & Raw projection & 84.24 & 88.17 & 0/8 \\
Mechanism & Mistral & Adam projection & 2.72 & 93.93 & 0/8 \\
Mechanism & Mistral & Adam + 2\% & 17.39 & 92.21 & 0/8 \\
\addlinespace[2pt]
Mechanism & Qwen & Plain & 96.32 & 94.62 & 5/8 \\
Mechanism & Qwen & Fixed 10\% & 100.00 & 94.74 & 8/8 \\
Mechanism & Qwen & Retain KL & 91.91 & 94.60 & 6/8 \\
Mechanism & Qwen & Generic KL & 99.26 & 94.85 & 6/8 \\
Mechanism & Qwen & Adaptive & 100.00 & 94.85 & 8/8 \\
Mechanism & Qwen & MSSR & 98.53 & 94.64 & 4/8 \\
Mechanism & Qwen & Static mask & 94.85 & 95.23 & 8/8 \\
Mechanism & Qwen & Dynamic mask & 96.32 & 95.10 & 8/8 \\
Mechanism & Qwen & Raw projection & 100.00 & 95.39 & 8/8 \\
Mechanism & Qwen & Adam projection & 98.53 & 94.45 & 8/8 \\
Mechanism & Qwen & Adam + 2\% & 100.00 & 95.39 & 8/8 \\
\addlinespace[2pt]
Mechanism & OLMo & Plain & 98.08 & 94.95 & 7/8 \\
Mechanism & OLMo & Fixed 10\% & 100.00 & 94.89 & 8/8 \\
Mechanism & OLMo & Retain KL & 99.52 & 94.45 & 7/8 \\
Mechanism & OLMo & Generic KL & 97.12 & 94.58 & 8/8 \\
Mechanism & OLMo & Adaptive & 100.00 & 94.58 & 8/8 \\
Mechanism & OLMo & MSSR & 97.12 & 94.66 & 7/8 \\
Mechanism & OLMo & Static mask & 100.00 & 94.39 & 8/8 \\
Mechanism & OLMo & Dynamic mask & 100.00 & 94.85 & 8/8 \\
Mechanism & OLMo & Raw projection & 98.08 & 94.47 & 8/8 \\
Mechanism & OLMo & Adam projection & 97.12 & 94.64 & 8/8 \\
Mechanism & OLMo & Adam + 2\% & 98.56 & 94.16 & 7/8 \\
\midrule
MTE matched & Mistral & Plain & 3.80 & 91.83 & 0/8 \\
MTE matched & Mistral & Fixed 10\% & 97.83 & 92.65 & 8/8 \\
MTE matched & Mistral & Generic KL & 4.89 & 91.58 & 0/8 \\
MTE matched & Mistral & Raw projection & 88.04 & 91.21 & 1/8 \\
MTE matched & Mistral & Adam projection & 3.80 & 93.53 & 0/8 \\
MTE matched & Mistral & Adam + 2\% & 42.93 & 92.76 & 0/8 \\
\addlinespace[2pt]
MTE matched & Qwen & Plain & 98.53 & 95.21 & 8/8 \\
MTE matched & Qwen & Fixed 10\% & 100.00 & 94.47 & 8/8 \\
MTE matched & Qwen & Generic KL & 99.26 & 94.28 & 7/8 \\
MTE matched & Qwen & Raw projection & 100.00 & 94.60 & 8/8 \\
MTE matched & Qwen & Adam projection & 97.79 & 95.10 & 7/8 \\
MTE matched & Qwen & Adam + 2\% & 100.00 & 94.62 & 7/8 \\
\addlinespace[2pt]
MTE matched & OLMo & Plain & 97.60 & 94.43 & 7/8 \\
MTE matched & OLMo & Fixed 10\% & 100.00 & 94.26 & 8/8 \\
MTE matched & OLMo & Generic KL & 97.60 & 94.56 & 8/8 \\
MTE matched & OLMo & Raw projection & 97.60 & 95.00 & 7/8 \\
MTE matched & OLMo & Adam projection & 97.12 & 94.33 & 8/8 \\
MTE matched & OLMo & Adam + 2\% & 99.52 & 94.43 & 8/8 \\
\midrule
Bank sensitivity & Qwen (full) & Plain & 96.32 & 94.68 & 6/8 \\
Bank sensitivity & Qwen (full) & Fixed 10\% & 100.00 & 94.81 & 8/8 \\
Bank sensitivity & Qwen (full) & Adam + 2\% & 100.00 & 94.89 & 7/8 \\
\addlinespace[2pt]
Bank sensitivity & Qwen (balanced) & Plain & 96.32 & 94.68 & 6/8 \\
Bank sensitivity & Qwen (balanced) & Fixed 10\% & 100.00 & 94.62 & 8/8 \\
Bank sensitivity & Qwen (balanced) & Adam + 2\% & 100.00 & 94.28 & 7/8 \\
\addlinespace[2pt]
Bank sensitivity & Qwen (factual) & Plain & \textit{n/a} & 94.68 & 8/8 \\
Bank sensitivity & Qwen (factual) & Fixed 10\% & \textit{n/a} & 93.87 & 8/8 \\
Bank sensitivity & Qwen (factual) & Adam + 2\% & \textit{n/a} & 95.23 & 8/8 \\
\addlinespace[2pt]
Bank sensitivity & Qwen (no factual) & Plain & 96.09 & 94.68 & 6/8 \\
Bank sensitivity & Qwen (no factual) & Fixed 10\% & 100.00 & 94.30 & 8/8 \\
Bank sensitivity & Qwen (no factual) & Adam + 2\% & 100.00 & 95.27 & 8/8 \\
\midrule
Long horizon & Mistral & Plain & 2.72 & 94.64 & 0/8 \\
Long horizon & Mistral & Fixed 10\% & 97.28 & 94.49 & 8/8 \\
Long horizon & Mistral & Adam + 2\% & 96.20 & 93.68 & 2/8 \\
\addlinespace[2pt]
Long horizon & Qwen & Plain & 100.00 & 95.10 & 2/8 \\
Long horizon & Qwen & Fixed 10\% & 100.00 & 95.27 & 8/8 \\
Long horizon & Qwen & Adam + 2\% & 100.00 & 94.39 & 8/8 \\
\addlinespace[2pt]
Long horizon & OLMo & Plain & 97.12 & 94.68 & 7/8 \\
Long horizon & OLMo & Fixed 10\% & 100.00 & 94.30 & 8/8 \\
Long horizon & OLMo & Adam + 2\% & 100.00 & 94.74 & 7/8 \\
\end{longtable}
\noindent\textit{Key.} Fixed 10\% is supervised old-data replay; Retain KL uses teacher targets on retention data; Generic KL uses OpenWebText; Adaptive is adaptive hybrid; MSSR is the registered paper-specification implementation; Static/Dynamic mask are function-column masks; Raw/Adam projection denote the two cached-probe projection rules; Adam + 2\% adds supervised replay. In the bank phase, full, balanced, factual-only, and without-factual denote the four frozen Qwen bank variants. Factual-only has no generation-qualified old audit, hence \textit{n/a}; this does not mean zero forgetting. All 576 runs are included: 419 feasible, 156 retention-audit failures, and one new-acquisition failure.
\endgroup

\section{Execution and Projection Audit}
\label{app:execution}

\subsection{Separating residuals, staleness, and outcomes}

The Adam-projection audit covers 152 runs: 29,047 step events, 3,689 gradient-refresh events, and 152 final-budget events. Projection intervenes at 11,143 steps, exactly when the recorded candidate dot is positive. The checks find no zero old-gradient vector, zero candidate update, or entirely unexecuted intervention. Post-projection update norms agree with the independently recorded actual displacement norms to numerical precision at sampled steps.

The frozen post-projection tolerance is
\begin{equation}
 \mathrm{tol}_t=\max\{10^{-7},10^{-3}|h_t^\top v_t|\}.
\end{equation}
Of the 11,143 interventions, 11,110 meet this tolerance and 33 exceed it across 26 runs: 16 events on Mistral, 17 on OLMo, none on Qwen. The largest positive residual is $6.3778134\times10^{-7}$. These values are small but constitute a genuine failure of the frozen gate. No threshold is relaxed and no implicated seed is removed. The population supports an implementation audit with observed deviations, not an ideal projection guarantee on every run.

The central Mistral seed-1101 witness has \emph{zero} violations over all 128 steps. It is therefore not explained by the 33 residual failures. The fresh-versus-cached comparison uses the same probe examples, objective, and actual applied displacement; only the gradient's parameter state changes. Refreshes occur at updates 1, 9, 17, and so on, whereas diagnostics occur at 16, 32, and so on. All 1,756 paired diagnostics consequently have age seven. The $639/1756$ and $505/609$ counts in Appendix~\ref{app:scale-results} are not estimates for uniformly sampled training steps.

\subsection{Meaning of each diagnostic}

Parameter snapshots include trainable parameters, which in this setting are LoRA factors. They do not measure $W_{\rm base}+sBA$ drift. Representation diagnostics evaluate a deterministic stratified probe at the hidden position immediately preceding the first answer token. The final summary uses the last diagnostic and the last layer. Mean cosine and centered linear CKA \citep{kornblith2019cka} therefore cover a specified probe/readout, not the complete representation geometry.

An actual-update dot is $\nabla Q(\theta_t)^\top u_t$, with the fresh gradient evaluated before the update and $u_t$ measured after AdamW and projection. A separate field called \texttt{linearized\_old\_loss\_change} in the cumulative parameter-drift summary is $\nabla Q(\theta_T)^\top(\theta_T-\theta_0)$. That field is neither an observed loss difference nor an initial-state Taylor prediction nor the path sum in Equation~\ref{eq:path}. It is not used as a measured old-loss reduction in this paper.

\subsection{Compute and environment}

The recorded environment is Python 3.12, PyTorch 2.8.0+cu128, CUDA runtime 12.8, Transformers 4.57.6, PEFT 0.20.0, Accelerate 1.14.0, and bitsandbytes 0.50.1. Training uses bf16, AdamW with learning rate $10^{-4}$ and zero weight decay, batch size 2, eight-step gradient accumulation, length 256, clipping 1.0, and gradient checkpointing. QLoRA uses rank 16, alpha 32, dropout $0.05$, all linear modules, NF4, and double quantization. New-acquisition-only calibration precedes the frozen learning-rate lock.

Execution started on one reported 48-GB RTX 4090 and continued with three such GPUs. Some tasks shared a GPU. No reported MTE claim treats those concurrent wall times as an isolated hardware efficiency experiment. Equal-MTE groups have all six registered methods and eight seeds; observed budget overshoot is below 1\%. Long-horizon plain, fixed, and Adam-plus-2\% arms consume respectively 512, 569, and 523 optimizer updates, with 0, 910, and 167 old-example exposures and 8,192 new exposures in each arm.

\subsection{Artifact completeness}

The downloaded source has 576 unique formal results and matching configurations. Local verification matches all result/config hashes and 45 frozen source fingerprints. The 65,064 lines across the three analyzed JSONL types parse correctly. Forty-nine referenced source-dataset paths are not available on the Mac, so their contents and semantic separation were not rehashed locally. All 576 scale-study results record removal of their selected and final checkpoints after audit. The analyzed package supports recomputation of logged statistics, but not re-evaluation of these trained models or additional weight interventions.

The remote aggregation script stopped on a three-versus-four generation-family configuration mismatch after training completed. The independent local analysis reads the frozen four-family specification and writes a separate derived report without modifying training sources, data hashes, results, or gates. Its rerun of the original completion checks retains \texttt{FAIL} for the 33 numerical violations. It is a derived audit, not a rewritten remote completion receipt.

The artifact includes the independent analysis, source hashes, complete group tables, and plotting inputs. The eight scale/historical figures were redrawn in base R 4.6.1 through R.app on macOS as vector PDFs, with dedicated right-side legends; all remain included alongside the three refresh--coverage figures. The extractor verifies row-level contrasts, seed sets, qualification counts, projection events, historical validity records, and full-precision NI decisions before plotting. Paired-effect, marginal-endpoint, and pointwise-trajectory intervals are identified in their captions; only the explicitly labeled controller grid uses interpolation.

\section{Long-Horizon Functional Trajectories}
\label{app:long}

\begin{figure}[htbp]
\centering
\includegraphics[width=\linewidth]{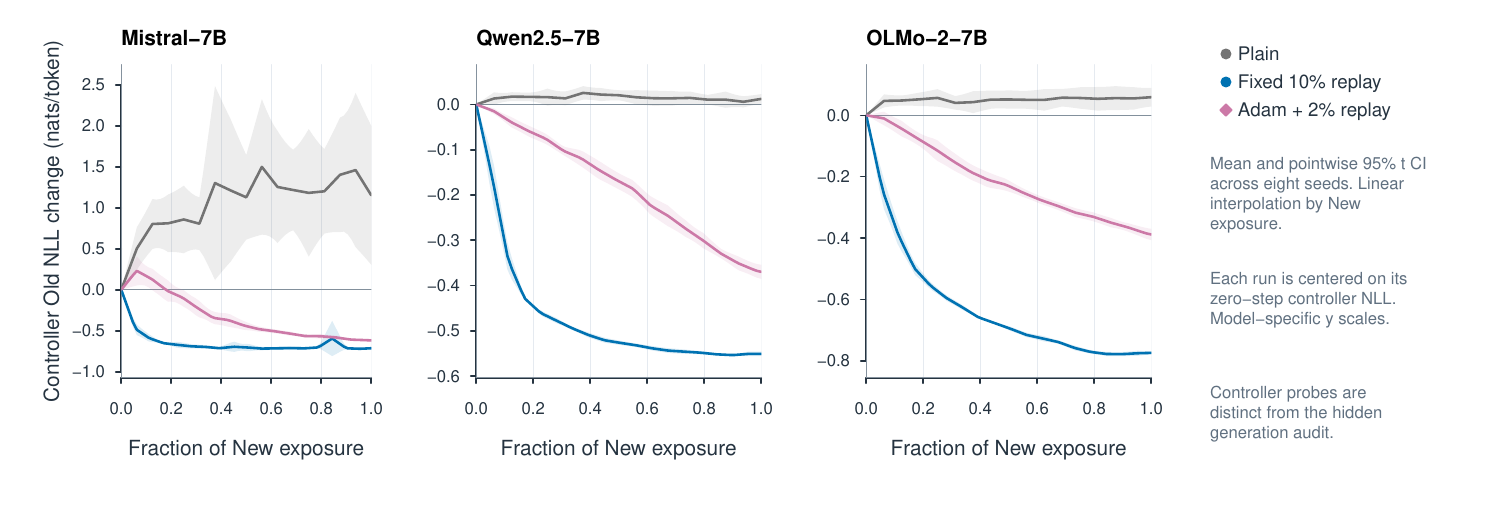}
\caption{Controller Old NLL relative to each run's zero-step baseline over 8,192 new exposures. Curves use linear interpolation to a common 65-point grid, followed by eight-seed means and pointwise 95\% $t$ intervals. Y axes are model-specific. These are monitoring trajectories, not hidden-generation curves or simultaneous confidence bands.}
\label{fig:longcurves}
\end{figure}

Time-normalized Old-NLL AUC benefits of fixed replay versus plain are $1.733$, $0.490$, and $0.679$ for Mistral, Qwen, and OLMo. The Adam-plus-2\% combination gives $1.411$, $0.195$, and $0.266$. Thus high final generation can coexist with different accumulated old-loss exposure. Mistral fixed replay records 11 controller violations across its long runs despite all eight final endpoints being feasible. Raw violation counts have different evaluation opportunities across methods and are not directly comparable rates. The long-phase seeds differ from standard-phase seeds; the phase contrast is not a replay-versus-short-horizon counterfactual from identical trained states.

\FloatBarrier
\section{Supporting Controlled and Cross-Adaptation Evidence}
\label{app:historical}

The supporting controlled study uses a public Pythia checkpoint and synthetic old/new biographical facts \citep{biderman2023pythia}. Concrete example budgets, prompt masks, random streams, and optimizer settings separate data order from replay information. This is a secondary-domain training counterfactual, not a reconstruction of the original pretraining trajectory. The 100-update information experiment and the subsequent optimizer-state swap use common starting parameters and optimizer states where required by their causal question.

\begin{table}[htbp]
\centering\small
\caption{Six-seed common-state information experiment. Positive AUC protection means lower old-probe loss than the matched new-only path. Values are old-loss AUC units, not percentage points.}
\begin{tabular}{lr}
\toprule
Replay information & Old-loss AUC protection [95\% paired CI]\\
\midrule
Source-normalized supervision & $6.668\ [4.008,9.328]$\\
Global-normalized supervision & $9.635\ [7.412,11.857]$\\
Teacher KL & $7.067\ [4.574,9.560]$\\
Input only & $0.510\ [-1.895,2.916]$\\
\bottomrule
\end{tabular}
\end{table}

\begin{figure}[htbp]
\centering
\includegraphics[width=\linewidth]{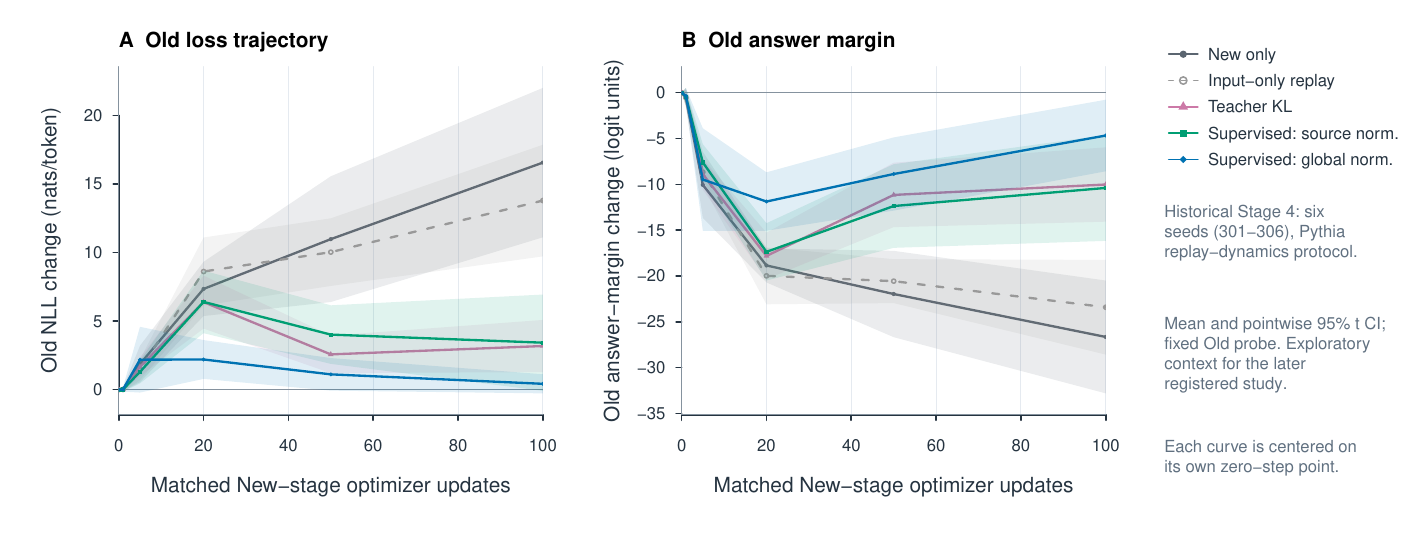}
\caption{Earlier common-state Pythia trajectories (seeds 301--306), redrawn in R. Loss and answer margin are centered at the starting probe value. Bands are pointwise 95\% $t$ intervals. This supporting information intervention is distinct from the 7B confirmation, and informative replay can have early deterioration before longer-horizon protection emerges.}
\label{fig:historydynamics}
\end{figure}

At 100 updates, final old exact is $0.167$ for new-only, $0.142$ for input-only, $0.792$ for global supervision, $0.785$ for source-normalized supervision, and $0.698$ for teacher KL. Supervised-minus-input-only is $0.642$ $[0.551,0.733]$; teacher-KL-minus-input-only is $0.556$ $[0.411,0.700]$. The input-only control carries no old answer or teacher target, so it tests exposure to prompts without a functional answer constraint, not every possible unlabeled learning objective.

In the optimizer-state experiment, identical rewound parameters receive sequential-donor, replay-donor, cross-donor first/second moments, zero moments at a matched step, or a fresh optimizer; all receive the same later pure-new batches. Replay-donor versus sequential-donor old-loss AUC is $0.956$ $[-0.351,2.263]$, exact $p=.09375$, with final old-exact difference $-0.052$ $[-0.129,0.025]$. This leaves the independent transferable-moment hypothesis unresolved; it does not establish that optimizer history has no role.

\paragraph{Qualified Full/LoRA/QLoRA confirmation.}
The earlier 384-run study contains 168 mitigation, 144 information, and 72 controller-stress runs, with disjoint eight-seed sets 711--718, 811--818, and 911--918. Models are Pythia-1B, Qwen2.5-0.5B-Instruct, and TinyLlama-1.1B-Chat \citep{biderman2023pythia,qwen2024qwen25,zhang2024tinyllama}. The frozen old-token primary improves under fixed replay in all nine model--adaptation cells, each with eight positive paired effects and within-model Holm $p=.0234375$ (Figure~\ref{fig:legacy}). Its effects span $1.5$ to $44.6$ token-accuracy points. All nine new-generation NI intervals extend below the registered bound. Only Qwen has an eligible old-generation subset (12 items), so that study's cross-model primary is token retention, not generation.

\begin{figure}[htbp]
\centering
\includegraphics[width=\linewidth]{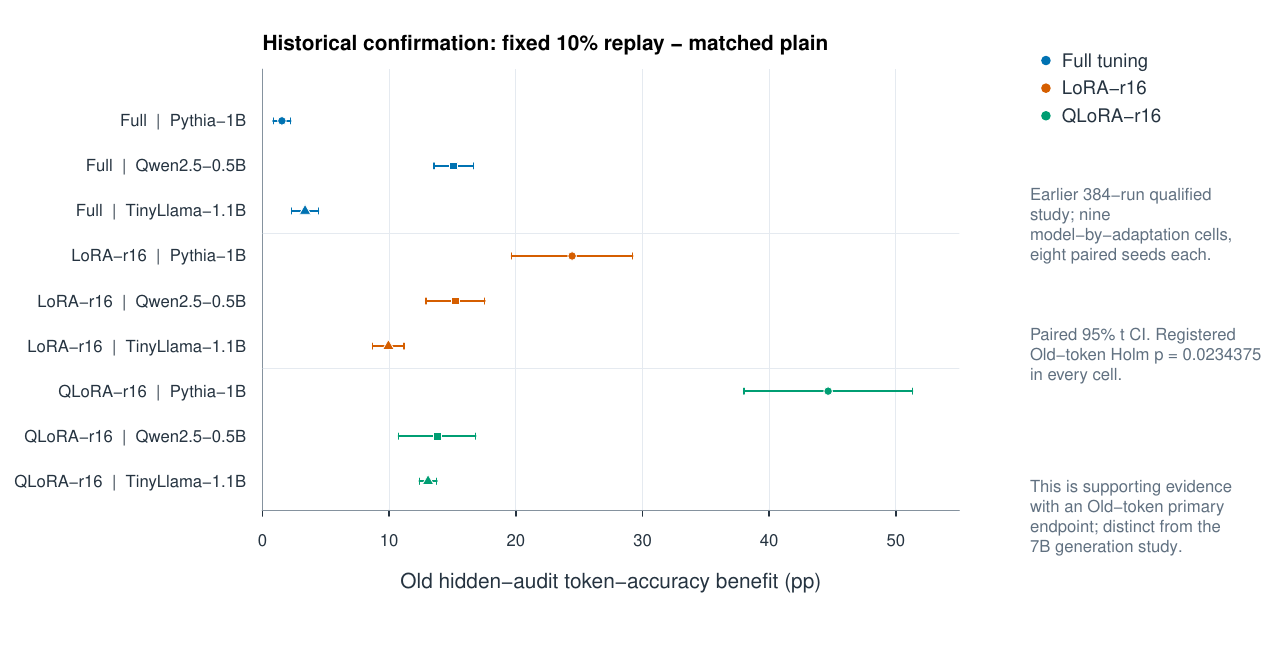}
\caption{Historical cross-adaptation confirmation: fixed replay minus matched plain. Each cell has eight seeds; bars are 95\% paired $t$ intervals and all nine recorded old-token Holm values are $.0234375$. This result demonstrates adaptation coverage under the historical qualified banks; it is not a 7B generation experiment.}
\label{fig:legacy}
\end{figure}

That complete matrix has 301 feasible and 83 retention-audit outcomes, with no missing matched-inference runs. Supervised and hybrid replay exceed input-only on final old token and old-loss AUC in each model; teacher KL is more variable across paths. The controller-stress comparison shows adaptive hybrid using less replay but lower old-token retention than fixed replay. These results motivate a fixed functional baseline while leaving adaptive scheduling as a separate trade-off. Earlier development matrices and acquisition-gate failures are retained in the repository reports but are not pooled into these confirmatory estimates.

\FloatBarrier

\section{Implementation Details and Retention Contract}
\label{app:implementation}

\subsection{Detailed experimental configuration}

We analyze three complementary designs separately. A complete 576-run, three-model scale study tests mitigation across new exposure, optimization-compute proxies, bank composition, and horizon. An independently seeded, 256-run refresh--coverage supplement intervenes on gradient refresh, protected sample size, replay timing, and objective allocation. A 64-run structure--location study crosses constraint granularity with optimizer placement and adds observed generation paths and two adaptation conditions. Results are paired within their own protocol rather than pooled across the three designs.

\paragraph{Models, acquired behaviors, and task.}
The scale study uses Mistral-7B-Instruct-v0.3, Qwen2.5-7B-Instruct, and OLMo-2-1124-7B-Instruct under rank-16 NF4 QLoRA \citep{dettmers2023qlora,qwen2024qwen25}; the refresh--coverage supplement uses Mistral and Qwen. In these original studies, New SFT maps sentiment instructions to controlled code labels, with 4,775 train, 628 development, and 597 audit records. The single-token target provides a reproducible new-acquisition endpoint. Old banks are qualified at the actual zero-step QLoRA start. Their hidden teacher-forced audits contain 96/74/66 records, and their qualified generation subsets contain 23/17/26, respectively. Mistral's generation subset primarily covers instruction format; Qwen's and OLMo's primarily cover arithmetic. Appendix~\ref{app:banks} and Figure~\ref{fig:coverage} give the exact domains. The models therefore provide different acquired-function settings, not an architecture ranking. The Mistral-only structure--location study retains the legacy code setting, adds a natural-language sentiment variant and a controlled structured-extraction task, and evaluates two independently templated, 60-record strict-EOS banks (Appendix~\ref{app:structure-location}).

\paragraph{Registered experimental programs.}
The scale matrix comprises 264 matched-new-exposure, 144 equal-MTE, 96 Qwen bank-sensitivity, and 72 long-horizon runs, each cell with eight seeds. The complete refresh--coverage supplement contains \textbf{256 formal runs}: two models, 16 methods, and paired seeds 51001--51008, following 32 smoke and 36 development runs. Every registered cell and all eight model-specific statistical families are complete. Its final manifest-bound audit verifies result/configuration hashes, matched exposures, and numerical integrity. Appendix~\ref{app:causal} reports all 32 method groups and 22 registered contrasts, including every scientific failure. The separate structure--location matrix contains 48 code runs (six methods, seeds 63001--63008) and eight runs for each additional task (plain/replay, seeds 63001--63004); the two four-seed extensions are descriptive. Its EOS-supervised New-training protocol and acquisition rules were frozen separately, after development, and are not retroactive amendments to the original studies.

\paragraph{Matched updates and explicit replay mass.}
Every refresh--coverage run consumes the same 2,048 ordered New records and 128 AdamW updates within its model/seed, with 16 New records per update and microbatch size two. Trainable adapter factors use FP32 storage, the frozen NF4 base uses BF16 computation, learning rate is $10^{-4}$, weight decay is zero, and adapter dropout is zero. Each record's loss first averages its supervised tokens. The update objective is
\begin{equation}
 \ell_t=\frac{1}{16}\sum_{i\in B_t^{\rm new}}\ell_i+
         \frac{1}{16}\sum_{j\in B_t^{\rm old}}\ell_j .
 \label{eq:mass-detail}
\end{equation}
For nominal fractions $\rho=.02,.05,.10,.20$, rounding $2048\rho/(1-\rho)$ gives exactly 42, 108, 228, or 512 Old exposures. Replay adds old information without replacing New records or changing optimizer-step count. This matches exposure, not hardware FLOPs.

\paragraph{Factorial and same-state interventions.}
In the refresh--coverage factorial, actual-update projection varies refresh interval $r\in\{1,8\}$ and protected sample size $n\in\{12,48\}$. The smaller probe is nested in the larger; a separate 12-record diagnostic set $E$ is fixed across all cells and never supplies optimization gradients. At every update we compute the actual displacement, cached/current probe gradients, and diagnostic gradient, and log their inner products, update norm, and before/after diagnostic loss. At updates 8/32/64/96/128, cached and fresh projections share the \emph{same} pre-update parameters and candidate Adam displacement; optimizer moments are left unchanged, and the realized post-update parameters are restored after these diagnostic branches. Full traces cover all cache ages; these five $r=8$ branch points have age seven.

The projection is applied in FP32 and the stored displacement is checked against $\max(10^{-7},10^{-3}|h^\top v|)$, with up to three explicit rounding corrections. Numerical failures remain visible. This refresh--coverage implementation does not relabel the 33 residual violations in the original scale study. Validation records exclude the 48-record protection pool but include $E$; the hidden audit is separate from both. The structure--location study instead crosses aggregate/stratum-wise constraints with pre-/post-Adam placement under a shared solver and fresh 48-record protection gradients. Its four same-state branches share the pre-intervention parameters, moments, batch, and RNG, not the post-intervention Adam candidate (Appendix~\ref{app:structure-location}).

\paragraph{Replay timing, targets, and calibrated baselines.}
Uniform, front-half, and back-half replay share the identical ordered list of 228 Old records and the same New sequence; the intervention changes when Old targets occur. At fractions .02 and .10, hybrid replay replaces each old objective by $.5\,{\rm CE}+.5\,T^2{\rm KL}(p_{\rm base}\Vert p_\theta)$, $T=2$, holding old-example counts and the sum of nominal CE/KL coefficients fixed. The teacher is the same quantized base with adapters disabled. This is a target-allocation comparison, not KL added on top of unchanged supervision.

I-GEM uses current stratum-specific gradient constraints before AdamW \citep{tekmen2026igem}; PEGP projects both LoRA-factor gradients using frozen input-feature subspaces \citep{qiao2024pegp}; ARR combines New-input reverse KL with OpenWebText next-token replay \citep{riemer2025arr}. Two development seeds choose parameters from frozen grids, subject to New acquisition, before formal audits are opened. These are shared-budget QLoRA adaptations; Appendix~\ref{app:causal-methods} specifies their relation to the source methods.

\paragraph{Endpoints and inference.}
All methods return the final fixed-budget endpoint, without selecting a favorable earlier checkpoint. In the original studies and the new Legacy endpoint, Old generation uses the frozen greedy decoder and reference-length budget with the token-ID comparison specified in Appendix~\ref{app:banks}. The structure--location study additionally records strict-EOS generation under a fixed 64-token cap at six prespecified checkpoints; the new banks and scoring rules are specified in Appendix~\ref{app:structure-location}. Validation Old-NLL AUC integrates signed, baseline-centered, token-weighted loss over normalized New exposure. The additional positive-part exposure and observed-peak summaries are defined in Appendix~\ref{app:positive-exposure}. Diagnostic $E$ instead averages record losses and supports the update decomposition. Per-run \texttt{FEASIBLE} requires New acquisition and the complete global/stratum retention contract. New non-inferiority (NI) separately requires the paired 95\% $t$ interval lower bound to exceed $-.03$.

The original studies use exact two-sided paired sign flips over eight seeds. The refresh--coverage supplement's primary families are factorial Old-NLL AUC, external-baseline Old generation, schedule AUC, and hybrid AUC, with Holm correction within each model/family. Each family is evaluated only after all of its registered eight-seed cells are complete. Same-state effects and New NI are supporting, per-comparison summaries; ratio curves are descriptive. Its complete family definitions and numerical/measurement audits are in Appendix~\ref{app:causal}. The structure--location study registers three eight-seed factorial contrasts on Legacy reference-length exact and reports a separate supporting code replay contrast on the independent strict-EOS banks. Its natural-language and extraction extensions use four paired seeds, descriptive intervals, and no confirmatory NI decisions (Appendix~\ref{app:structure-location}). Seed intervals describe variation on fixed question sets, not bank-sampling uncertainty.

\paragraph{Scale and efficiency controls.}
The 576-run matrix includes fixed supervised replay, target and generic KL, adaptive/scheduled replay, masks, and raw/actual-update projections. Equal-compute arms match dense-model-token equivalents (MTE) within the frozen 3\% tolerance; long runs consume 8,192 New exposures. Full details and results remain in Appendix~\ref{app:scale-results}. Supplementary timers separate optimization, diagnosis, and evaluation, but shared-GPU runs are excluded from exclusive-method timing comparisons. No inference treats an MTE proxy or shared wall time as measured hardware FLOPs.

\subsection{A retention contract for SFT}
\label{app:contract}
\label{sec:tool}

\method{} turns the functional-path account into an acceptance workflow: freeze acquired old behaviors, train with informative anchors, record old-loss trajectories, and audit the final model. Supervised replay supplies the baseline, with its budget chosen from the measured Old/New trade-off.

For stratum $g$ and metric $k$, let $\Delta_{gk}(\theta,\theta_0)$ measure degradation in its natural direction and let $b_{gk}>0$ be a declared budget. Define
\begin{equation}
 R(\theta)=\max_{g,k}\frac{\Delta_{gk}(\theta,\theta_0)}{b_{gk}}.
 \label{eq:contract}
\end{equation}
Endpoint acceptance requires New acquisition and $R(\theta)\leq1$ on hidden audit. The contract specifies an operating threshold, whereas paired endpoint effects quantify mitigation: a method can improve a declared behavior relative to its control without satisfying the entire contract. All evaluable runs remain in comparisons regardless of retention-contract status. Controller data may guide replay or checkpoint restoration; these experiments return fixed-budget endpoints. A trajectory contract additionally constrains risk or integrated loss during training, which endpoint acceptance cannot certify.

Old NLL admits the path decomposition; current-gradient dots measure local pressure; margins track decision boundaries; generation checks the served sequence. Exposing these complementary signals keeps local feasibility distinct from retained behavior.

\section{Expanded Related Work}
\label{app:related-work}

\paragraph{Functional regularization and replay.}
Rehearsal, distillation, parameter regularization, and episodic constraints are established approaches \citep{li2017learning,chaudhry2018efficient,rolnick2019replay}. Gaussian-process regularization and FROMP explicitly protect functions at selected inputs \citep{titsias2019functional,pan2020continual}. Self-synthesized rehearsal and MSSR address missing data and replay allocation \citep{huang2024ssr,lu2026mssr}; approximate regularized replay combines KL and generic-text replay \citep{riemer2025arr}. Our contribution is a common accounting of realized updates and old-loss changes, together with controlled interventions on target timing and objective allocation and separate measurement of loss trajectories and final generated behavior. The ARR and MSSR comparisons are independent shared-budget adaptations, not official-code benchmark reproductions.

\paragraph{Geometry and parameter-efficient adaptation.}
LoRA and QLoRA modify the served function through low-rank factors \citep{hu2022lora,dettmers2023qlora}. Continual adaptations exploit subspaces, adapter connectivity, instruction similarity, and selected distillation \citep{qiao2024pecl,ren2024ilora,wang2024inscl,he2024seekr}. PEGP protects factor-input subspaces, while I-GEM efficiently handles multiple gradient constraints \citep{qiao2024pegp,tekmen2026igem}. Their distinct protected objects matter: our aggregate actual-update halfspace is not equivalent to either method. I-GEM provides a competitive retention baseline, while Qwen's probe-size intervention demonstrates an effect of protection coverage on the measured loss trajectory. These findings motivate distinguishing gradient-constraint implementations and probe coverage in functional evaluations.

\paragraph{Modern LoRA protection in continual learning.}
\citet{lu2025clora} study both single-task and continual fine-tuning with fixed-subspace regularization of LoRA factors. \citet{xiong2026oplora} protect selected pretrained-weight singular subspaces with a two-sided projection. Our common QLoRA sequence evaluates these mechanisms alongside task-gradient constraints, adapter EWC, and replay, retaining their average benefits and task-wise trade-offs. Its OPLoRA adaptation uses approximate singular bases of the initial dequantized weight, not the historical feature-covariance null spaces of Adam-NSCL \citep{wang2021adamnscl}. We do not experimentally equate these mechanisms or claim a tuned ranking of their original implementations (Appendix~\ref{app:continual-bridge}).

\paragraph{Behavior, representation, and recoverability.}
Old information may remain decodable after task performance regresses \citep{tao2023can}; conjugate prompting and alignment recovery can restore some outputs \citep{kotha2024understanding,zheng2025spurious}. We measure regression through the original frozen interface. This functional loss can coexist with recoverability, and neither the loss identity nor adapter-distance/CKA measurements establishes irreversible knowledge erasure.

\section{Structure--Location Study: Protocol, Behavioral Paths, and Complete Outcomes}
\label{app:structure-location}

\subsection{Frozen protocol and task construction}
\label{app:structure-protocol}

This separate Mistral-7B-Instruct-v0.3 study contains 64 formal runs: six Code methods with paired seeds 63001--63008 (48 runs), and Plain/Replay10 on Natural sentiment and Structured extraction with paired seeds 63001--63004 (eight runs per task). Four full-length Natural/Extraction development pilots use seed 62001 and precede the formal lock. The EOS-supervised protocol was frozen after inspecting an earlier development protocol, before these four pilots; the earlier outcomes retain their original status and do not enter these estimates. In particular, the earlier five-point token-gain gate is not retrospectively applied or declared passed.

All runs use rank-16, alpha-32, all-linear-module QLoRA, an NF4 base with BF16 computation, FP32 adapter factors, zero adapter dropout/weight decay, and AdamW at learning rate $10^{-4}$. The 128 updates consume the same 2,048 ordered New records within task/seed, with 16 records per update and microbatch size two. Replay10 adds 228 ordered Old exposures. Record losses first average supervised tokens and then use the fixed denominator of 16 New records, as in Equation~\ref{eq:mass}; replay does not replace New exposures. New training supervises an appended EOS for \emph{every} task and rejects truncation of the answer or EOS. Old replay/probe masks and answer-only evaluation losses remain unchanged. This study matches New exposure and updates; it has no hardware-compute matching claim.

The same protection pool contains 48 controller records, 12 each from factual, instruction-format, multilingual, and arithmetic strata. Protection gradients are fresh at every update. The 50-record validation set excludes protection; a fixed 12-record diagnostic subset $E$ lies within validation and supplies no training gradient. The 96-record hidden old audit and all generation banks are outside protection/validation. Every method reports update 128, with no favorable-checkpoint selection.

\paragraph{Three frozen New interfaces.}
Code and Natural use the same SST-2 review identities, labels, ordering, and 4,775/628/597 train/development/audit split. Natural changes both the prompt and the answer form, while retaining sentiment semantics. The following are the complete prompt templates; braces denote substituted data, and each stored completion has a leading space.
\begin{quote}\small\ttfamily\raggedright
Task: sentiment-code-v1.\par
Classify the sentiment of the movie review using the learned internal code. Return exactly one code token.\par
Review: \{review\}\par
Code:
\end{quote}
The negative/positive completions are \texttt{A}/\texttt{B}, respectively.
\begin{quote}\small\ttfamily\raggedright
Task: sentiment-natural-v2.\par
Classify the sentiment of the movie review. Return exactly one sentence: This review expresses positive sentiment. or This review expresses negative sentiment.\par
Review: \{review\}\par
Answer:
\end{quote}
For the actual audit review \emph{an intimacy that sucks you in and dares you not to believe it} (source index 7324), the targets are \texttt{B} and \texttt{This review expresses positive sentiment.} Natural is high-baseline adaptation: its full-audit content accuracy starts at 82.24\%, reference exact at 88.78\%, and strict EOS at 80.23\%.

Extraction has 2,048/128/128 train/development/audit records, disjoint entity identifiers, and split-specific note templates. The complete outer template is:
\begin{quote}\small\ttfamily\raggedright
Task: dispatch-schema-v2.\par
Convert the dispatch note to the learned three-field JSON schema. Use keys rava, nemi, tovu in that order, with no explanation.\par
Note: \{note\}\par
JSON:
\end{quote}
Training notes read \texttt{Dispatch \{count\} units of \{cargo\} to customer \{recipient\}.}; development uses \texttt{Customer \{recipient\} is to receive \{count\} units; the item is \{cargo\}.}; audit uses \texttt{Item: \{cargo\}. Recipient: \{recipient\}. Number of units: \{count\}.} The first archived audit note is \texttt{Item: CAe8f2b0000. Recipient: RA0000. Number of units: 83.}, with target \texttt{\{"rava":"RA0000","nemi":"CAe8f2b0000","tovu":83\}}. This controlled extraction task uses new field names; it is not a broad natural benchmark.

\paragraph{New validity and joint retention.}
The frozen task-level gate requires finite metrics, final answer-token accuracy at least .50, answer NLL no worse than baseline, and either lower answer NLL or improved reference exact. Code/Extraction additionally require reference exact at least .50 with positive gain and strict-EOS exact at least .50. Extraction requires content accuracy at least .80 with positive gain. Natural requires reference, content, and strict accuracy each to remain within .03 of its own baseline. These per-run validity checks are distinct from statistical NI. The joint contract also applies the original global/stratum NLL, margin, token, and sequence budgets, and a maximum .10 generation drop separately to Legacy reference exact and each independent strict bank. All 64 formal runs completed with valid New acquisition and numerical checks. Seven satisfy the full frozen retention contract; the remaining 57 do not meet that absolute retention criterion. Endpoint effects use all evaluable runs rather than conditioning on contract satisfaction.

\subsection{Qualified behavioral banks and decoding}
\label{app:structure-banks}

\begin{table}[htbp]
\centering\small
\caption{Pre-SFT qualification. Eligibility uses the declared interface; selection is deterministic within eligible strata. Each independent bank has 20 arithmetic, 20 literal-copy, and 20 format records.}
\label{tab:structure-banks}
\begin{tabular}{lrrrp{5.5cm}}
\toprule
Bank & Candidates & Eligible & Selected & Qualification/interface\\
\midrule
Legacy & 23 & 23 & 23 & Original raw prompts; reference-length exact\\
Independent A & 240 & 239 & 60 & Native chat; fixed-cap strict EOS\\
Independent B & 240 & 199 & 60 & Different templates/content; fixed-cap strict EOS\\
\bottomrule
\end{tabular}
\end{table}

Eligibility rates are 99.58\% and 82.92\% for A and B; the selected fraction is 25\% in each. Qualification uses only baseline behavior. A/B have disjoint content groups and template families, and are disjoint from the legacy protection/replay/audit records. Their frozen native-chat prompts use the pinned tokenizer's rendering. The legacy prompts retain their original raw interface. A uses addition, identifier copying, and \texttt{VALUE=} formatting; B uses multiplication word problems, passphrase copying, and bracket formatting. These banks enlarge controlled behavioral coverage without representing a random sample of general capabilities.

Each prompt receives two separate greedy calls: one capped at the reference token count, and one capped at 64 new tokens with natural EOS stopping, no forced EOS, and no reference-triggered stopping. Reference exact removes trailing pad/EOS before token comparison. Fixed-cap strict exact requires the tokens before the first EOS to equal the reference and requires EOS to be observed. Content scoring uses frozen lexical rules: reference containment for literal answers, the first matched integer for arithmetic, the correct sentiment word without its opposite, or equality of the parsed three-field JSON object. These rules are not human semantic judgments. Format and EOS outcomes remain separately recorded.

Both independent strict baselines are 100\%. Legacy reference exact also starts at 100\%, but its strict-EOS baseline is 0\%; extending the generation horizon exposes continuations outside its original qualification interface. Accordingly, the independent strict endpoint averages A and B equally and excludes Legacy. Full New audits contain 597 sentiment or 128 extraction records. The 64-record New trajectory subset is used only for intermediate paths; it does not replace full-audit endpoints or NI.

\paragraph{Cross-call and same-output consistency.}
\label{app:structure-score-audit}
Let $r$ be the reference after trailing pad/EOS removal, $z^{\rm ref}$ the reference-length output, and $z^{64}$ the fixed-cap output. Write $\nu$ for trailing pad/EOS removal and $b(z)$ for the tokens before the first EOS, trimming trailing padding if no EOS occurs. The original scores and a post-hoc same-output prefix score are
\[
\begin{aligned}
R&=\mathbf{1}\{\nu(z^{\rm ref})=r\}, &
P_{64}&=\mathbf{1}\{b(z^{64})_{:|r|}=r\},\\
F_{64}&=\mathbf{1}\{b(z^{64})=r\}, &
S_{64}&=F_{64}\,\mathbf{1}\{\text{EOS observed in }z^{64}\}.
\end{aligned}
\]
Thus $S_{64}\Rightarrow F_{64}\Rightarrow P_{64}$ on the same output; $R$ comes from a separate call. Reference calls group records by reference-token length before batching in pairs, whereas fixed-cap calls batch pairs in original cohort order. Both use manual left padding. Their companions, padding widths, batch sizes, or row positions can therefore differ.

Re-scoring saved tokens/text and re-aggregating summaries yields no inconsistencies in 2,004,080 score fields or 91,968 aggregate checks across the 64 formal and four development runs. Formal trajectories contain 58/148,400 records with $S_{64}=1,R=0$: zero at baseline, 41 at intermediate checkpoints, and 17 among 47,704 final records. The separate pilot count is 1/5,760. Correctness differs between $R$ and $P_{64}$ on 120 formal records, including 27 final records; these counts are distinct from all token-prefix differences. Records repeat prompts across runs/checkpoints and overlapping full/subset New cohorts, rather than supplying independent training replicates.

A targeted GPU recheck restores the 13 final adapters associated with all 17 final discordant record contexts and includes 17 matched prefix-consistent controls. Matching prefers records correct under both historical interfaces and otherwise retains explicitly marked, incorrect but prefix-consistent controls. Crossing the original reference/fixed batch contexts and a singleton context with short/64-token horizons, twice per cell, gives 408 scored calls; 68 further calls use the original generation arguments without score collection. Both archived interfaces reproduce their saved focal outputs token-for-token (Table~\ref{tab:structure-score-audit}). With batch context fixed, changing the horizon changes no reference-length prefix in any of the 34 contexts. With horizon fixed, switching the original batch contexts changes the prefix for all 17 targets and no controls, consistently across both horizons and repetitions. All 204 repeated-call pairs agree on the full batch tokens, and same-output strict implies prefix exact throughout.

\begin{table}[t]
\centering\small
\caption{Archived-score consistency and targeted inference recheck. Offline field checks include 64 formal and four development runs; discrepancy rows below include only formal records. The GPU cohort contains all 17 final discrepant records and 17 matched prefix-consistent controls, not a random sample. Original-kwargs calls omit score collection. The layout result holds for both horizons and both repeats; repeated calls are not independent training samples.}
\label{tab:structure-score-audit}
\begin{tabular}{@{}p{0.60\linewidth}p{0.33\linewidth}@{}}
\toprule
Check & Result\\
\midrule
Scoring-field mismatches & 0 / 2,004,080\\
Aggregate-field mismatches & 0 / 91,968\\
Formal $S_{64}=1$, $R=0$: all checkpoints & 58 / 148,400 records\\
Formal $S_{64}=1$, $R=0$: final checkpoint & 17 / 47,704 records\\
\midrule
Final adapters / target records / controls & 13 / 17 / 17\\
Scored / original-kwargs calls & 408 / 68\\
Historical raw-token reproduction: scored & 136 / 136 calls\\
Historical raw-token reproduction: original kwargs & 68 / 68 calls\\
Prefix changes: fixed layout, changed horizon & 0 / 34 records\\
Prefix changes: fixed horizon, original layouts & 17 / 17 targets; 0 / 17 controls\\
Stable identical-input repeat pairs & 204 / 204 (full batch)\\
Same-output $S_{64}\Rightarrow P_{64}$ violations & 0\\
\bottomrule
\end{tabular}
\end{table}

The batch-context intervention reproduces these selected final discrepancies; it jointly changes co-batched examples, padding, batch size, or row placement. For Code/plain seed 63007, bank-B item 0048 has reference \texttt{135}: the original reference batch produces \texttt{A} at either horizon, while the fixed and singleton contexts produce the target prefix at either horizon and EOS under the 64-token cap. Its prefix matches fixed for seven targets and reference for ten. Stored first-divergence scores are processed scores from the recheck.
\subsection{Matched projection implementation}
\label{app:structure-projection}

Let $g_k$ be the record-mean old gradient for stratum $k$, with $n_k$ protected records. Aggregate protection uses $\bar g=\sum_k(n_k/48)g_k$; stratum-wise protection uses all four rows. Strata denote old capability categories, not network layers. The common Euclidean solver computes
\[
\Pi_G(v)=\arg\min_u \tfrac12\|u-v\|^2\quad\text{subject to }Gu\leq0.
\]
It normalizes nonzero constraint rows, forms the small Gram system in FP64, and enumerates active sets with a common KKT tolerance $10^{-9}\max(1,\|Gv\|_\infty)$ and a 10,000-set budget. Rank-deficient systems use the same pseudoinverse rule; solver failures are explicit. Row normalization preserves the halfspaces.

After global gradient clipping at norm 1, \emph{pre-Adam} projection takes the negative stored gradient as its candidate, writes back the negative projected vector as the optimizer gradient, and lets AdamW update its moments and parameters. It adds no post-Adam correction. \emph{Post-Adam} projection first performs AdamW, projects the resulting actual displacement, and writes the corrected parameters; the moments remain those produced by the incoming gradient. Thus the location intervention includes the ensuing optimizer-state trajectory.

Both locations check the vector after storage, against each normalized row, with tolerance $\max(10^{-7},10^{-3}|\widehat g_k^\top v|)$ and up to three storage corrections. The checked object is the negative stored gradient for pre and the stored parameter displacement for post. A passing pre check therefore does not certify its Adam displacement. The archive contains 8,192 training events; the four projection cells contribute 32 runs and 4,096 projection updates. No numerical failures are recorded in these updates or the 160 same-state branches. This matched implementation is distinct from the original I-GEM adaptation.

\subsection{Complete outcomes and statistical families}
\label{app:structure-outcomes}

Tables~\ref{tab:structure-outcomes}--\ref{tab:structure-bank-endpoints} report every group, interval, seed, and bank/interface endpoint. Legacy reference and independent strict changes equal their final accuracy minus 100 percentage points; Table~\ref{tab:structure-intervals} also reports New content gains. Code starts at zero on all three New generation metrics. Extraction also starts at zero on response-level measures, although its answer-token accuracy is already 90.04\%; after training, both methods reach 100\% content accuracy. The teacher-forced baseline therefore cannot substitute for generated-response acquisition.

\begingroup
\small\setlength{\tabcolsep}{3pt}
\begin{longtable}{llrrrrrrrc}
\caption{All ten groups at the final fixed budget. Accuracy columns are percentages: Legacy reference exact (L), independent-bank A/B strict-EOS exact (A/B), and complete New-audit reference/strict/content accuracy (R/S/C). F counts runs satisfying the full frozen retention contract, not valid or beneficial experiments; all evaluable runs are included. Code uses eight paired seeds; Natural and Extraction use four and are exploratory. Agg./Str. denote aggregate/stratum-wise constraints. All Legacy reference and independent strict baselines are 100\%.}\label{tab:structure-outcomes}\\
\toprule
Task & Method & $n$ & L & A & B & New R & New S & New C & F/$n$\\
\midrule
\endfirsthead
\toprule
Task & Method & $n$ & L & A & B & New R & New S & New C & F/$n$\\
\midrule
\endhead
\bottomrule
\endfoot
Code & Plain & 8 & 5.43 & 3.13 & 4.17 & 92.40 & 92.40 & 92.40 & 0/8 \\
Code & Agg./pre & 8 & 97.83 & 43.96 & 65.21 & 92.00 & 92.00 & 92.00 & 0/8 \\
Code & Agg./post & 8 & 7.07 & 3.13 & 4.17 & 93.20 & 93.20 & 93.20 & 0/8 \\
Code & Str./pre & 8 & 96.74 & 77.92 & 76.88 & 89.01 & 89.01 & 89.01 & 3/8 \\
Code & Str./post & 8 & 70.65 & 19.79 & 28.75 & 93.55 & 93.55 & 93.55 & 0/8 \\
Code & Replay10 & 8 & 97.83 & 65.21 & 82.71 & 93.80 & 93.80 & 93.80 & 1/8 \\
Natural & Plain & 4 & 47.83 & 99.17 & 99.58 & 94.43 & 94.43 & 94.43 & 0/4 \\
Natural & Replay10 & 4 & 100.00 & 90.42 & 90.42 & 93.68 & 93.68 & 93.68 & 3/4 \\
Extraction & Plain & 4 & 72.83 & 100.00 & 95.83 & 99.80 & 100.00 & 100.00 & 0/4 \\
Extraction & Replay10 & 4 & 94.57 & 82.08 & 68.33 & 99.61 & 99.61 & 100.00 & 0/4 \\
\end{longtable}
\endgroup

\begingroup
\small\setlength{\tabcolsep}{3pt}
\begin{longtable}{llrrrr}
\caption{Marginal 95\% seed $t$ intervals for all final groups (percentage points). $\Delta$Strict is the equal-weight A/B strict-EOS change from baseline; $\Delta$New is complete-audit content change. Intervals are unadjusted and untruncated, so a bound may extend beyond the attainable range. They condition on the fixed qualified records.}\label{tab:structure-intervals}\\
\toprule
Task & Method & Legacy Ref. & $\Delta$Strict A/B & New Content & $\Delta$New Content\\
\midrule
\endfirsthead
\toprule
Task & Method & Legacy Ref. & $\Delta$Strict A/B & New Content & $\Delta$New Content\\
\midrule
\endhead
\bottomrule
\endfoot
Code & Plain & $[2.86,8.01]$ & $[-101.20,-91.51]$ & $[90.16,94.64]$ & $[90.16,94.64]$ \\
Code & Agg./pre & $[95.08,100.57]$ & $[-70.11,-20.72]$ & $[89.73,94.27]$ & $[89.73,94.27]$ \\
Code & Agg./post & $[3.21,10.92]$ & $[-102.47,-90.23]$ & $[91.95,94.44]$ & $[91.95,94.44]$ \\
Code & Str./pre & $[92.50,100.97]$ & $[-45.87,0.66]$ & $[79.29,98.73]$ & $[79.29,98.73]$ \\
Code & Str./post & $[49.39,91.92]$ & $[-103.47,-47.99]$ & $[92.40,94.70]$ & $[92.40,94.70]$ \\
Code & Replay10 & $[95.08,100.57]$ & $[-38.35,-13.74]$ & $[92.22,95.39]$ & $[92.22,95.39]$ \\
Natural & Plain & $[9.11,86.55]$ & $[-1.29,0.04]$ & $[92.13,96.73]$ & $[9.89,14.48]$ \\
Natural & Replay10 & $[100.00,100.00]$ & $[-35.71,16.54]$ & $[92.40,94.95]$ & $[10.16,12.71]$ \\
Extraction & Plain & $[66.20,79.45]$ & $[-5.42,1.25]$ & $[100.00,100.00]$ & $[100.00,100.00]$ \\
Extraction & Replay10 & $[87.94,101.19]$ & $[-40.57,-9.02]$ & $[100.00,100.00]$ & $[100.00,100.00]$ \\
\end{longtable}
\endgroup

\begingroup
\small\setlength{\tabcolsep}{3pt}
\begin{longtable}{llrrrrrrc}
\caption{All 64 seed-level outcomes, including retention-contract violations. Accuracy columns are percentages; Strict averages independent banks A/B and New is complete-audit content accuracy. Loss AUC and AUC$_+$ use validation NLL and normalized New exposure. F=1 denotes satisfaction of the full frozen retention contract; F=0 denotes a contract violation, not an invalid experiment or absence of relative benefit. All runs passed New-acquisition and numerical checks and remain included. Seeds are paired within each task.}\label{tab:structure-seeds}\\
\toprule
Task & Method & Seed & Legacy Ref. & Strict A/B & New & NLL AUC & AUC$_+$ & F\\
\midrule
\endfirsthead
\toprule
Task & Method & Seed & Legacy Ref. & Strict A/B & New & NLL AUC & AUC$_+$ & F\\
\midrule
\endhead
\bottomrule
\endfoot
Code & Plain & 63001 & 8.70 & 15.83 & 94.97 & 0.3601 & 0.3601 & 0 \\
Code & Plain & 63002 & 0.00 & 0.00 & 87.77 & 1.3287 & 1.3287 & 0 \\
Code & Plain & 63003 & 4.35 & 0.00 & 88.78 & 1.2420 & 1.2420 & 0 \\
Code & Plain & 63004 & 4.35 & 0.00 & 93.63 & 0.9098 & 0.9098 & 0 \\
Code & Plain & 63005 & 8.70 & 0.00 & 94.64 & 0.7481 & 0.7481 & 0 \\
Code & Plain & 63006 & 4.35 & 0.00 & 93.80 & 1.0052 & 1.0052 & 0 \\
Code & Plain & 63007 & 4.35 & 7.50 & 92.96 & 0.2799 & 0.2799 & 0 \\
Code & Plain & 63008 & 8.70 & 5.83 & 92.63 & 0.5387 & 0.5387 & 0 \\
Code & Agg./pre & 63001 & 91.30 & 0.00 & 90.79 & -0.1781 & 0.0100 & 0 \\
Code & Agg./pre & 63002 & 100.00 & 58.33 & 88.27 & -0.1858 & 0.0211 & 0 \\
Code & Agg./pre & 63003 & 100.00 & 67.50 & 92.46 & -0.1683 & 0.0184 & 0 \\
Code & Agg./pre & 63004 & 95.65 & 33.33 & 87.94 & 0.0842 & 0.2586 & 0 \\
Code & Agg./pre & 63005 & 95.65 & 89.17 & 93.97 & -0.1558 & 0.0452 & 0 \\
Code & Agg./pre & 63006 & 100.00 & 38.33 & 93.47 & -0.1745 & 0.0424 & 0 \\
Code & Agg./pre & 63007 & 100.00 & 65.00 & 94.97 & -0.2086 & 0.0000 & 0 \\
Code & Agg./pre & 63008 & 100.00 & 85.00 & 94.14 & -0.2995 & 0.0082 & 0 \\
Code & Agg./post & 63001 & 4.35 & 0.00 & 95.14 & 0.2530 & 0.2530 & 0 \\
Code & Agg./post & 63002 & 4.35 & 0.00 & 91.12 & 0.5896 & 0.5896 & 0 \\
Code & Agg./post & 63003 & 8.70 & 0.00 & 92.46 & 0.1950 & 0.1950 & 0 \\
Code & Agg./post & 63004 & 8.70 & 0.00 & 95.31 & 0.5758 & 0.5758 & 0 \\
Code & Agg./post & 63005 & 4.35 & 0.00 & 94.14 & 0.2907 & 0.2907 & 0 \\
Code & Agg./post & 63006 & 17.39 & 6.67 & 92.46 & 0.1970 & 0.1970 & 0 \\
Code & Agg./post & 63007 & 4.35 & 20.83 & 92.46 & 0.1067 & 0.1067 & 0 \\
Code & Agg./post & 63008 & 4.35 & 1.67 & 92.46 & 0.2403 & 0.2403 & 0 \\
Code & Str./pre & 63001 & 100.00 & 86.67 & 93.30 & -0.3000 & 0.0000 & 0 \\
Code & Str./pre & 63002 & 100.00 & 100.00 & 94.47 & -0.2949 & 0.0011 & 1 \\
Code & Str./pre & 63003 & 91.30 & 90.00 & 95.14 & -0.3110 & 0.0000 & 0 \\
Code & Str./pre & 63004 & 95.65 & 60.00 & 92.63 & -0.3245 & 0.0000 & 0 \\
Code & Str./pre & 63005 & 100.00 & 99.17 & 86.26 & -0.3136 & 0.0034 & 1 \\
Code & Str./pre & 63006 & 100.00 & 65.00 & 95.31 & -0.3053 & 0.0000 & 0 \\
Code & Str./pre & 63007 & 86.96 & 20.00 & 61.14 & -0.2992 & 0.0000 & 0 \\
Code & Str./pre & 63008 & 100.00 & 98.33 & 93.80 & -0.3263 & 0.0000 & 1 \\
Code & Str./post & 63001 & 82.61 & 31.67 & 95.81 & -0.0005 & 0.0040 & 0 \\
Code & Str./post & 63002 & 78.26 & 7.50 & 92.13 & -0.0076 & 0.0183 & 0 \\
Code & Str./post & 63003 & 86.96 & 95.83 & 91.79 & -0.0212 & 0.0033 & 0 \\
Code & Str./post & 63004 & 78.26 & 0.00 & 94.97 & 0.0635 & 0.0679 & 0 \\
Code & Str./post & 63005 & 73.91 & 0.00 & 93.47 & 0.0071 & 0.0156 & 0 \\
Code & Str./post & 63006 & 8.70 & 0.00 & 93.63 & 0.1077 & 0.1077 & 0 \\
Code & Str./post & 63007 & 82.61 & 44.17 & 93.97 & 0.0032 & 0.0145 & 0 \\
Code & Str./post & 63008 & 73.91 & 15.00 & 92.63 & -0.0213 & 0.0029 & 0 \\
Code & Replay10 & 63001 & 100.00 & 77.50 & 94.97 & -0.5700 & 0.0000 & 0 \\
Code & Replay10 & 63002 & 100.00 & 95.83 & 94.64 & -0.5760 & 0.0000 & 1 \\
Code & Replay10 & 63003 & 95.65 & 82.50 & 91.79 & -0.5675 & 0.0000 & 0 \\
Code & Replay10 & 63004 & 95.65 & 61.67 & 95.31 & -0.5525 & 0.0000 & 0 \\
Code & Replay10 & 63005 & 100.00 & 78.33 & 94.64 & -0.5492 & 0.0000 & 0 \\
Code & Replay10 & 63006 & 100.00 & 81.67 & 94.81 & -0.5838 & 0.0000 & 0 \\
Code & Replay10 & 63007 & 91.30 & 65.83 & 94.30 & -0.5711 & 0.0000 & 0 \\
Code & Replay10 & 63008 & 100.00 & 48.33 & 89.95 & -0.5471 & 0.0000 & 0 \\
Natural & Plain & 63001 & 17.39 & 99.17 & 93.30 & 0.0895 & 0.0895 & 0 \\
Natural & Plain & 63002 & 39.13 & 100.00 & 96.31 & 0.0224 & 0.0224 & 0 \\
Natural & Plain & 63003 & 69.57 & 99.17 & 93.30 & 0.0213 & 0.0226 & 0 \\
Natural & Plain & 63004 & 65.22 & 99.17 & 94.81 & 0.0331 & 0.0341 & 0 \\
Natural & Replay10 & 63001 & 100.00 & 97.50 & 94.64 & -0.5711 & 0.0000 & 1 \\
Natural & Replay10 & 63002 & 100.00 & 98.33 & 92.80 & -0.6001 & 0.0000 & 1 \\
Natural & Replay10 & 63003 & 100.00 & 65.83 & 93.30 & -0.5778 & 0.0000 & 0 \\
Natural & Replay10 & 63004 & 100.00 & 100.00 & 93.97 & -0.5403 & 0.0000 & 1 \\
Extraction & Plain & 63001 & 69.57 & 100.00 & 100.00 & 0.1036 & 0.1036 & 0 \\
Extraction & Plain & 63002 & 78.26 & 98.33 & 100.00 & 0.0210 & 0.0213 & 0 \\
Extraction & Plain & 63003 & 69.57 & 98.33 & 100.00 & 0.1055 & 0.1055 & 0 \\
Extraction & Plain & 63004 & 73.91 & 95.00 & 100.00 & 0.0842 & 0.0842 & 0 \\
Extraction & Replay10 & 63001 & 95.65 & 85.00 & 100.00 & -0.4948 & 0.0000 & 0 \\
Extraction & Replay10 & 63002 & 91.30 & 82.50 & 100.00 & -0.5818 & 0.0000 & 0 \\
Extraction & Replay10 & 63003 & 91.30 & 66.67 & 100.00 & -0.5745 & 0.0000 & 0 \\
Extraction & Replay10 & 63004 & 100.00 & 66.67 & 100.00 & -0.4976 & 0.0000 & 0 \\
\end{longtable}
\endgroup

\begingroup
\small\setlength{\tabcolsep}{3pt}
\begin{longtable}{llrrrrrrrrr}
\caption{Separate final bank endpoints: reference-length exact, fixed-cap content, and fixed-cap strict-EOS exact (percent). Ref. and Strict use separate generation calls with different batch layouts, so their values need not be ordered; Appendix~\ref{app:structure-score-audit} audits cross-call and same-output consistency. The Legacy strict baseline is 0\%, despite its 100\% reference-length baseline; it is therefore excluded from the qualified strict-bank mean. A and B each start at 100\% for all three reported measures. These fixed banks are never replayed.}\label{tab:structure-bank-endpoints}\\
\toprule
Task & Method & \multicolumn{3}{c}{Legacy} & \multicolumn{3}{c}{Bank A} & \multicolumn{3}{c}{Bank B}\\
 & & Ref. & Content & Strict & Ref. & Content & Strict & Ref. & Content & Strict\\
\midrule
\endfirsthead
\toprule
Task & Method & \multicolumn{3}{c}{Legacy} & \multicolumn{3}{c}{Bank A} & \multicolumn{3}{c}{Bank B}\\
 & & Ref. & Content & Strict & Ref. & Content & Strict & Ref. & Content & Strict\\
\midrule
\endhead
\bottomrule
\endfoot
Code & Plain & 5.43 & 5.43 & 0.00 & 3.13 & 3.13 & 3.13 & 3.96 & 6.46 & 4.17 \\
Code & Agg./pre & 97.83 & 97.28 & 88.04 & 43.75 & 43.96 & 43.96 & 65.21 & 65.21 & 65.21 \\
Code & Agg./post & 7.07 & 7.61 & 2.72 & 3.13 & 3.13 & 3.13 & 3.75 & 6.25 & 4.17 \\
Code & Str./pre & 96.74 & 96.74 & 89.67 & 77.92 & 81.25 & 77.92 & 76.46 & 80.63 & 76.88 \\
Code & Str./post & 70.65 & 71.20 & 65.76 & 20.00 & 19.79 & 19.79 & 28.96 & 31.46 & 28.75 \\
Code & Replay10 & 97.83 & 98.37 & 80.98 & 65.42 & 84.38 & 65.21 & 86.04 & 92.71 & 82.71 \\
Natural & Plain & 47.83 & 51.09 & 40.22 & 99.17 & 99.17 & 99.17 & 100.00 & 99.58 & 99.58 \\
Natural & Replay10 & 100.00 & 100.00 & 86.96 & 90.42 & 98.33 & 90.42 & 95.00 & 92.50 & 90.42 \\
Extraction & Plain & 72.83 & 72.83 & 64.13 & 100.00 & 100.00 & 100.00 & 96.25 & 96.25 & 95.83 \\
Extraction & Replay10 & 94.57 & 95.65 & 85.87 & 82.08 & 99.58 & 82.08 & 68.33 & 86.25 & 68.33 \\
\end{longtable}
\endgroup

For each Code seed, write $A,B,C,D$ for aggregate/pre, aggregate/post, strata/pre, strata/post. The registered effects are
\[
G=\tfrac12[(C-A)+(D-B)],\qquad
P=\tfrac12[(B-A)+(D-C)],\qquad
I=(D-B)-(C-A).
\]
The interaction is the full difference-in-differences. Exact two-sided sign flips enumerate $2^8$ assignments; Holm correction covers these three Legacy-reference contrasts. Intervals are $\bar d\pm t_{.975,n-1}s_d/\sqrt n$, unadjusted and untruncated, conditional on the fixed records. The Legacy stratum-wise simple effect is $-1.09$ points before Adam and $+63.59$ after Adam. Independent-bank reference and strict interactions are $-1.67$ ($[-42.19,38.86]$) and $-2.19$ points ($[-42.32,37.94]$), respectively (Table~\ref{tab:structure-contrasts}). Their post-minus-pre effects retain the Legacy direction. These post-hoc results locate an endpoint-dependent, accuracy-scale interaction rather than a difference explained simply by requiring EOS. Bank content and prompt interface still vary together; near-ceiling pre-Adam accuracy also limits observable granularity gains. The registered Legacy inferential family remains separate.

\begingroup
\small\setlength{\tabcolsep}{3pt}
\begin{longtable}{llrrc}
\caption{Factorial effects in percentage points. Legacy inference uses the three registered paired contrasts and Holm correction within that family. Effects on equal-weight independent-bank strict-EOS, reference-length, and frozen content accuracy are post-hoc descriptive; no Legacy $p$-value is transferred to them. Each independent-bank score is averaged within seed before forming the paired contrasts. Intervals are unadjusted 95\% seed $t$ intervals.}\label{tab:structure-contrasts}\\
\toprule
Endpoint & Contrast & Estimate & 95\% CI & Holm $p$\\
\midrule
\endfirsthead
\toprule
Endpoint & Contrast & Estimate & 95\% CI & Holm $p$\\
\midrule
\endhead
\bottomrule
\endfoot
Legacy Ref. & Granularity $G$ & 31.25 & $[18.92,43.58]$ & 0.03125 \\
Legacy Ref. & Post minus pre $P$ & -58.42 & $[-68.82,-48.03]$ & 0.02344 \\
Legacy Ref. & Interaction $I$ & 64.67 & $[38.72,90.63]$ & 0.03125 \\
Independent Strict & Granularity $G$ & 21.72 & $[0.68,42.75]$ & -- \\
Independent Strict & Interaction $I$ & -2.19 & $[-42.32,37.94]$ & -- \\
Independent Strict & Post minus pre $P$ & -52.03 & $[-76.84,-27.22]$ & -- \\
\midrule
Independent Ref. & Granularity $G$ & 21.88 & $[0.76,42.99]$ & -- \\
Independent Ref. & Post minus pre $P$ & -51.88 & $[-76.86,-26.89]$ & -- \\
Independent Ref. & Interaction $I$ & -1.67 & $[-42.19,38.86]$ & -- \\
Independent Content & Granularity $G$ & 23.65 & $[1.47,45.82]$ & -- \\
Independent Content & Post minus pre $P$ & -52.60 & $[-78.90,-26.30]$ & -- \\
Independent Content & Interaction $I$ & -5.42 & $[-46.44,35.61]$ & -- \\
\end{longtable}
\endgroup

Across the six Code methods, three of 1,104 final Legacy records differ in $R$ versus $P_{64}$ correctness. The same-output sensitivity preserves the direction and similar magnitude of all three factorial effects, including an interaction of $64.13$ points (CI $[38.15,90.11]$; Table~\ref{tab:structure-prefix-sensitivity}). This post-hoc check leaves the registered endpoint and its Holm family unchanged.

\Needspace{12\baselineskip}
\begingroup
\small\setlength{\tabcolsep}{3pt}
\begin{longtable}{lrr}
\caption{Post-hoc same-output prefix sensitivity on the final Code Legacy audit, in percentage points. $P_{64}$ is derived from the archived 64-token call; $R$ uses the separately generated reference-length output. The original registered endpoint and inferential family remain unchanged. Intervals are unadjusted 95\% paired-seed $t$ intervals across eight seeds. The interaction is the full difference-in-differences, not a half-scaled effect.}\label{tab:structure-prefix-sensitivity}\\
\toprule
Effect & Registered $R$ & Same-output $P_{64}$ [95\% CI]\\
\midrule
Granularity $G$ & $+31.25$ & $+31.52\ [19.08,43.96]$\\
Post minus pre $P$ & $-58.42$ & $-57.61\ [-68.43,-46.79]$\\
Interaction $I$ & $+64.67$ & $+64.13\ [38.15,90.11]$\\
\bottomrule
\end{longtable}
\endgroup

\Needspace{17\baselineskip}
\begingroup
\small\setlength{\tabcolsep}{3pt}
\begin{longtable}{lrrrll}
\caption{Replay minus plain, paired within task. Old effects use the equal-weight independent-bank strict-EOS change; New effects use the full-audit content endpoint. Effects and unadjusted 95\% seed $t$ intervals are in percentage points. Only Code receives the supporting exact sign-flip test and New NI assessment; the four-seed tasks are descriptive.}\label{tab:structure-replay-contrasts}\\
\toprule
Task & $n$ & $\Delta$Old Strict [95\% CI] & $\Delta$New Content [95\% CI] & Old exact $p$ & New NI\\
\midrule
\endfirsthead
\toprule
Task & $n$ & $\Delta$Old Strict [95\% CI] & $\Delta$New Content [95\% CI] & Old exact $p$ & New NI\\
\midrule
\endhead
\bottomrule
\endfoot
Code & 8 & \shortstack{$70.31$\\$[55.97,84.66]$} & \shortstack{$1.40$\\$[-0.90,3.71]$} & .0078125 & Established \\
Natural & 4 & \shortstack{$-8.96$\\$[-34.88,16.96]$} & \shortstack{$-0.75$\\$[-4.01,2.51]$} & -- & Not assessed \\
Extraction & 4 & \shortstack{$-22.71$\\$[-36.29,-9.13]$} & \shortstack{$0.00$\\$[0.00,0.00]$} & -- & Not assessed \\
\end{longtable}
\endgroup

The eight-seed Code replay contrast on independent strict retention is supporting; all eight paired effects are positive. New NI separately compares full-audit content differences with $-.03$ using the full-precision lower 95\% paired bound. Code establishes this supporting criterion. Natural and Extraction are four-seed descriptive comparisons; confirmatory NI is not assessed. Their reverse strict-bank effects coexist with improved Legacy reference exact: $+52.17$ points for Natural and $+21.74$ for Extraction. Within bank A/B, strict replay-minus-plain effects are $-8.75/-9.17$ points for Natural and $-17.92/-27.50$ for Extraction; content effects are $-.83/-7.08$ and $-.42/-10.00$ points. These observations locate the scope of retention gains rather than establishing a task-independent method ranking.

\subsection{Behavioral paths, positive-part exposure, and contract violations}
\label{app:structure-paths}

NLL is observed at updates $0,16,\ldots,128$; generation is observed at $0,16,32,64,96,128$. Signed and positive-part areas use the same piecewise-linear definition as Appendix~\ref{app:positive-exposure}, integrating within each seed over normalized New exposure and splitting at zero crossings. Generation degradation uses baseline accuracy minus current accuracy. Maximum degradation and violation flags use observed points only. The loss flag tests the .20 NLL budget; the generation flag tests a .10 accuracy drop. The reported fractions include baseline in their nine- or six-observation denominator. They are metric-specific sampled violations, not the entire joint endpoint contract.

\begingroup
\small\setlength{\tabcolsep}{3pt}
\begin{longtable}{llrrrc}
\caption{Validation-loss paths for all groups. AUCs are nats/token integrated over normalized New exposure; $D_{\max}$ is the mean per-run maximum baseline-centered NLL on updates $0,16,\ldots,128$. Brackets give unadjusted seed 95\% $t$ intervals. $V$ is the mean fraction of those nine observations exceeding the declared .20 NLL budget, including baseline in the denominator.}\label{tab:structure-loss-paths}\\
\toprule
Task & Method & Signed AUC [95\% CI] & AUC$_+$ [95\% CI] & $D_{\max}$ & $V$\\
\midrule
\endfirsthead
\toprule
Task & Method & Signed AUC [95\% CI] & AUC$_+$ [95\% CI] & $D_{\max}$ & $V$\\
\midrule
\endhead
\bottomrule
\endfoot
Code & Plain & $0.802\ [0.476,1.128]$ & $0.802\ [0.476,1.128]$ & 1.437 & 0.819 \\
Code & Agg./pre & $-0.161\ [-0.252,-0.070]$ & $0.050\ [-0.021,0.122]$ & 0.316 & 0.056 \\
Code & Agg./post & $0.306\ [0.156,0.456]$ & $0.306\ [0.156,0.456]$ & 0.563 & 0.597 \\
Code & Str./pre & $-0.309\ [-0.319,-0.300]$ & $0.001\ [-0.000,0.002]$ & 0.008 & 0.000 \\
Code & Str./post & $0.016\ [-0.022,0.054]$ & $0.029\ [-0.003,0.061]$ & 0.071 & 0.000 \\
Code & Replay10 & $-0.565\ [-0.576,-0.553]$ & $0.000\ [0.000,0.000]$ & 0.000 & 0.000 \\
Natural & Plain & $0.042\ [-0.010,0.093]$ & $0.042\ [-0.009,0.093]$ & 0.074 & 0.000 \\
Natural & Replay10 & $-0.572\ [-0.612,-0.533]$ & $0.000\ [0.000,0.000]$ & 0.000 & 0.000 \\
Extraction & Plain & $0.079\ [0.016,0.141]$ & $0.079\ [0.016,0.141]$ & 0.128 & 0.028 \\
Extraction & Replay10 & $-0.537\ [-0.613,-0.462]$ & $0.000\ [0.000,0.000]$ & 0.000 & 0.000 \\
\end{longtable}
\endgroup

\begingroup
\small\setlength{\tabcolsep}{3pt}
\begin{longtable}{llrrr}
\caption{Sampled behavioral degradation for every group and bank. Each cell gives (normalized degradation area, maximum observed degradation, observed violation fraction), averaged over seeds. All quantities are proportions. Legacy uses reference exact; A/B use strict EOS. Violations mean a drop greater than .10 at the six recorded generation checkpoints, including baseline. Integration occurs within seed before averaging.}\label{tab:structure-behavior-paths}\\
\toprule
Task & Method & Legacy & Bank A & Bank B\\
\midrule
\endfirsthead
\toprule
Task & Method & Legacy & Bank A & Bank B\\
\midrule
\endhead
\bottomrule
\endfoot
Code & Plain & $(0.840,0.957,0.833)$ & $(0.816,0.985,0.792)$ & $(0.758,0.973,0.813)$ \\
Code & Agg./pre & $(0.424,0.875,0.563)$ & $(0.636,0.927,0.771)$ & $(0.573,0.850,0.750)$ \\
Code & Agg./post & $(0.826,0.957,0.833)$ & $(0.803,1.000,0.771)$ & $(0.750,0.996,0.792)$ \\
Code & Str./pre & $(0.088,0.505,0.208)$ & $(0.273,0.831,0.438)$ & $(0.281,0.642,0.542)$ \\
Code & Str./post & $(0.375,0.625,0.813)$ & $(0.743,0.985,0.750)$ & $(0.638,0.877,0.792)$ \\
Code & Replay10 & $(0.049,0.120,0.146)$ & $(0.213,0.473,0.417)$ & $(0.135,0.277,0.417)$ \\
Natural & Plain & $(0.327,0.533,0.833)$ & $(0.007,0.008,0.000)$ & $(0.008,0.025,0.000)$ \\
Natural & Replay10 & $(0.033,0.087,0.042)$ & $(0.122,0.258,0.292)$ & $(0.115,0.296,0.333)$ \\
Extraction & Plain & $(0.202,0.293,0.625)$ & $(0.002,0.008,0.000)$ & $(0.035,0.071,0.000)$ \\
Extraction & Replay10 & $(0.035,0.087,0.000)$ & $(0.158,0.204,0.417)$ & $(0.214,0.325,0.625)$ \\
\end{longtable}
\endgroup

Zero validation-NLL AUC$_+$ coexists with independent strict-bank loss under replay: final A/B losses are 26.04, 9.58, and 24.79 points for Code, Natural, and Extraction. The NLL and behavioral audits protect different measured objects. Consecutive-observation intervals retain the grid definition: for Code replay seed 63001, Legacy violations occupy updates 16--32 (two observations), bank A has update 32 and updates 96--128, and bank B has updates 32 and 128 separately. These intervals do not assert violation at every intervening update. All seed-specific intervals and pointwise values are retained in the machine-readable archive. Figures~\ref{fig:structure-allpaths}--\ref{fig:structure-newpaths} show pointwise 95\% seed $t$ bands within the 0--100\% plotting viewport; the plotting data retain untruncated limits. Bands are not simultaneous intervals.

\begin{figure}[htbp]
\centering
\includegraphics[width=\linewidth]{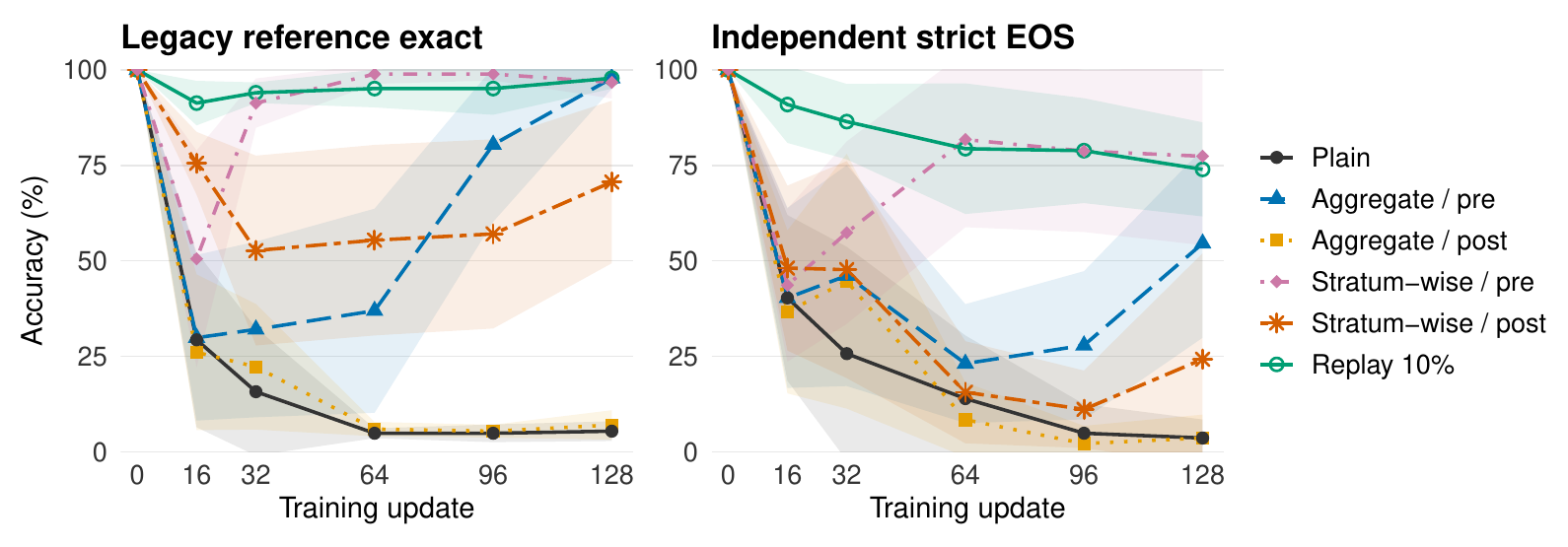}
\caption{All six Code methods on the Legacy reference and equal-weight independent strict-EOS endpoints. Points are the six fixed generation observations, summarized across eight paired seeds; lines join measured checkpoints. Endpoint agreement does not imply agreement along the sampled path.}
\label{fig:structure-allpaths}
\end{figure}

\begin{figure}[htbp]
\centering
\includegraphics[width=\linewidth]{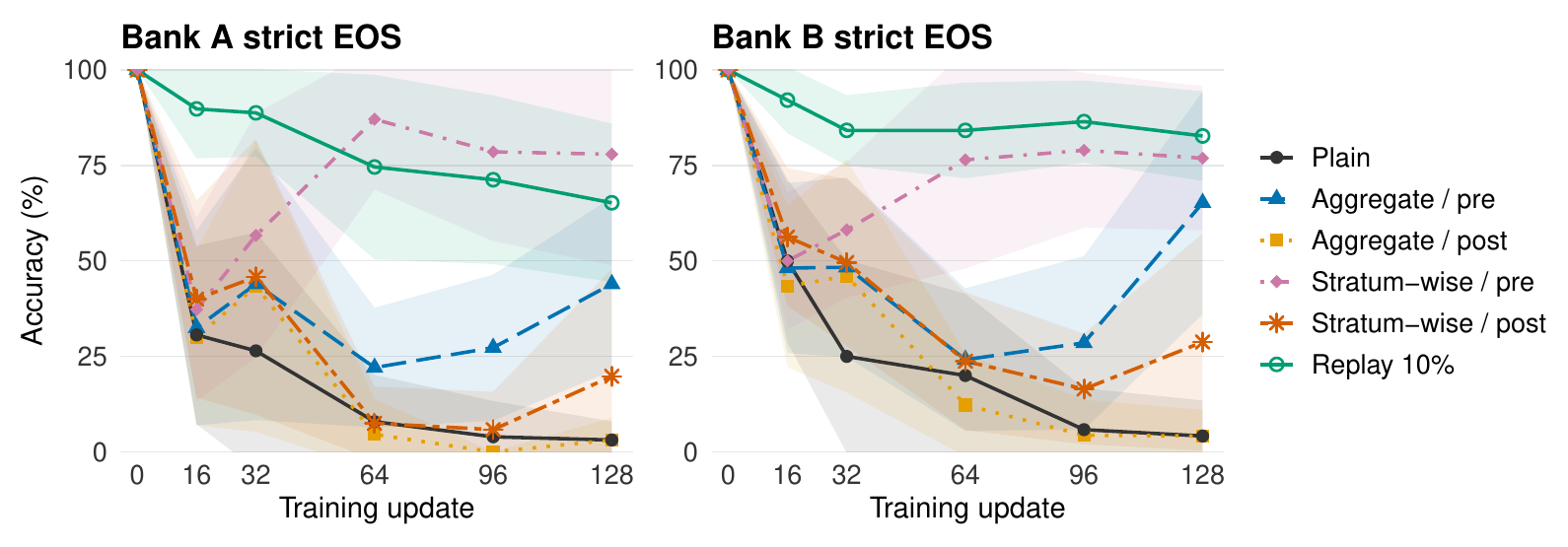}
\caption{Code paths on independent banks A and B separately, retaining all six methods and eight seeds. Each bank contains 60 fixed baseline-qualified items; bank differences do not add independent training replicates.}
\label{fig:structure-codebanks}
\end{figure}

\begin{figure}[htbp]
\centering
\includegraphics[width=\linewidth]{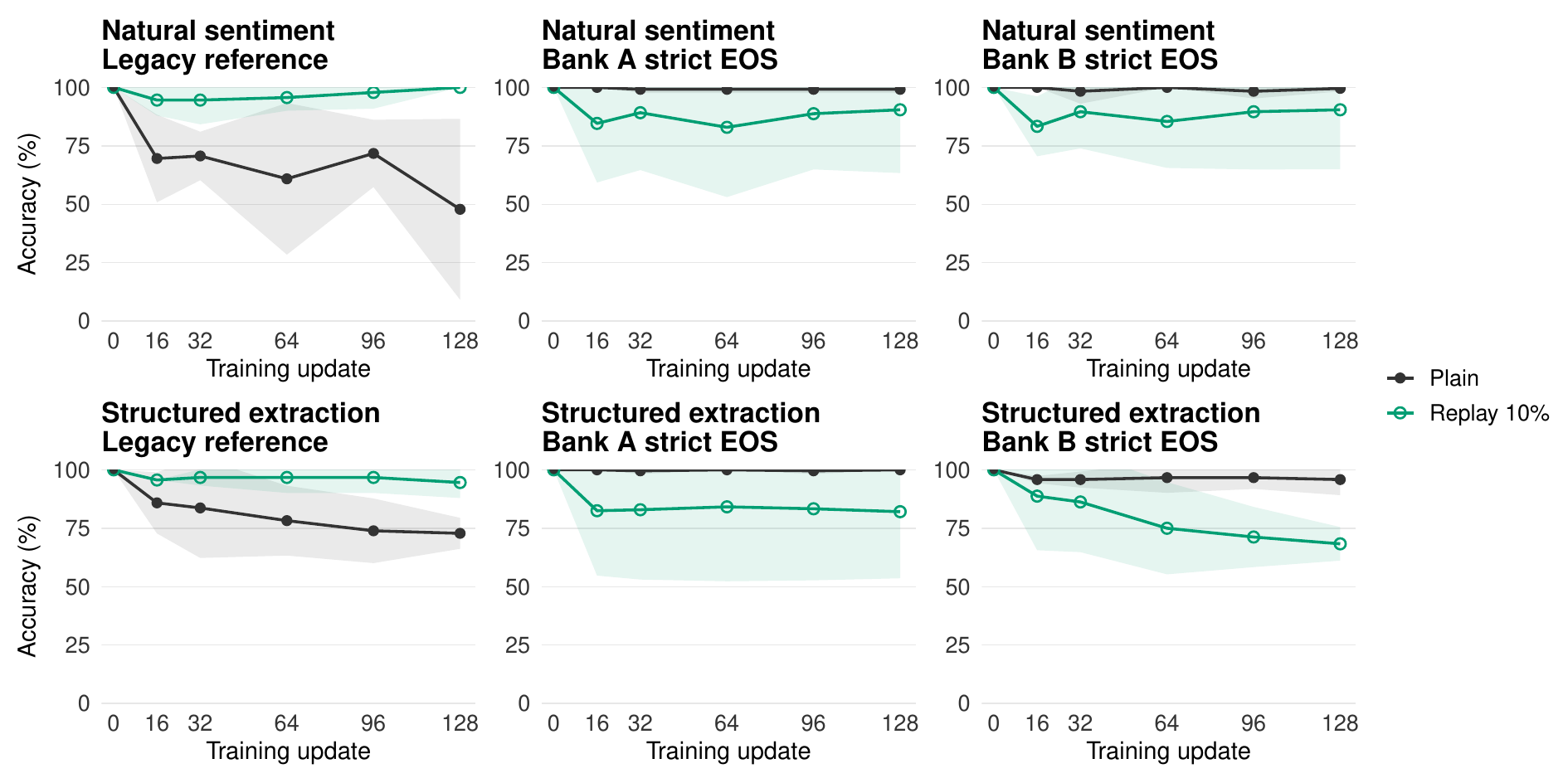}
\caption{Separate bank trajectories for Plain and Replay10 on Natural and Extraction, each with four exploratory seeds. The observed paths show the bank dependence of retention under these additional tasks; Code appears separately in Figures~\ref{fig:structure-allpaths} and~\ref{fig:structure-codebanks}.}
\label{fig:structure-bankpaths}
\end{figure}

\begin{figure}[htbp]
\centering
\includegraphics[width=\linewidth]{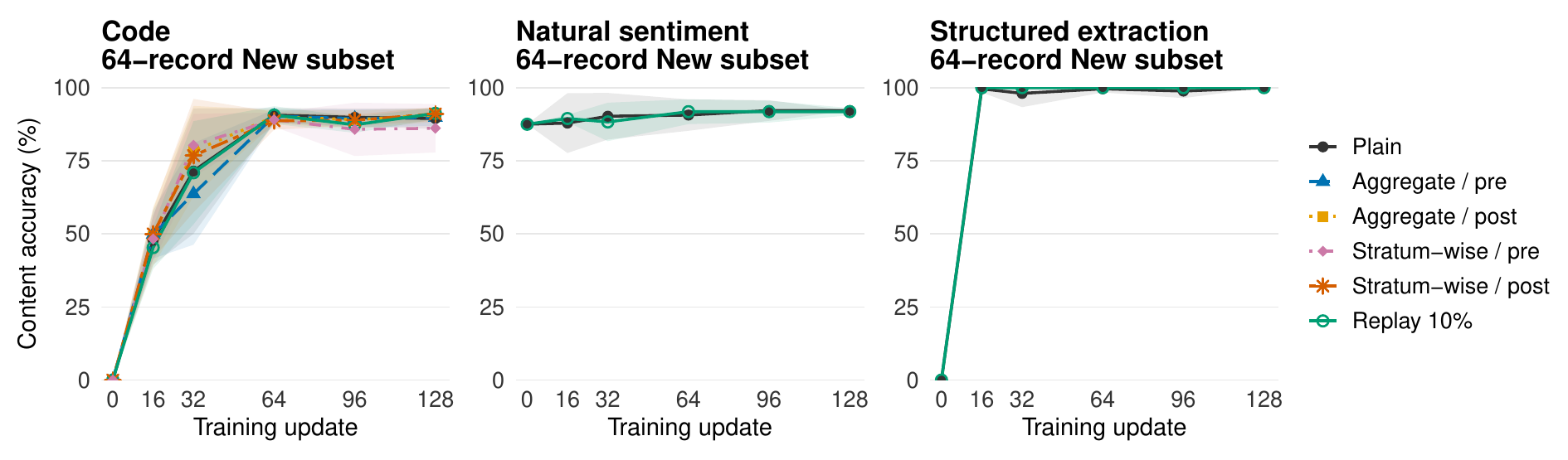}
\caption{New-generation trajectories on the fixed 64-record monitoring subset. These sampled paths are separate from the full 597-record sentiment and 128-record extraction audits used for the endpoint tables and New NI. Natural begins with substantial task performance; Extraction begins at zero response-level content accuracy.}
\label{fig:structure-newpaths}
\end{figure}
\FloatBarrier

\subsection{Same-state interventions and prospective diagnostics}
\label{app:structure-diagnostics}

At updates 8/32/64/96/128 of each Code/plain run, four branches share the pre-update parameters, Adam moments/step, batch, stored gradient, RNG, and training mode. Each branch restores this state, applies one policy, and is discarded afterward; the four resulting Adam candidates need not coincide. Diagnostic losses are measured at displacement scales $1,1/2,1/4$, with equal weight for each diagnostic stratum. Five branch effects are averaged within each seed before the eight-seed summary. The 40 states and 160 branches are not additional independent training seeds.

\begingroup
\small\setlength{\tabcolsep}{3pt}
\begin{longtable}{lrrrrrr}
\caption{Same-state diagnostic loss under the four branches and three displacement scales. Diagnostic strata receive equal weight and five states are first averaged within each of eight Code/plain seeds. $\Delta E$ is observed diagnostic NLL change, $g_E^\top u$ the linear term, and $R$ their difference; the final column averages absolute residuals before cross-seed averaging. Residuals include numerical error. The unscaled branches all pass their own declared numerical checks.}\label{tab:structure-local}\\
\toprule
Branch & Scale & $\Delta E$ & 95\% CI for $\Delta E$ & $g_E^\top u$ & $R$ & $|R|$\\
\midrule
\endfirsthead
\toprule
Branch & Scale & $\Delta E$ & 95\% CI for $\Delta E$ & $g_E^\top u$ & $R$ & $|R|$\\
\midrule
\endhead
\bottomrule
\endfoot
Agg./pre & 1.0 & -0.01233 & $[-0.02597,0.00132]$ & -0.01415 & 0.00182 & 0.01118 \\
Agg./pre & 0.5 & -0.00584 & $[-0.01310,0.00143]$ & -0.00708 & 0.00124 & 0.00946 \\
Agg./pre & 0.25 & -0.00352 & $[-0.00814,0.00110]$ & -0.00354 & 0.00002 & 0.00927 \\
Agg./post & 1.0 & -0.01510 & $[-0.03672,0.00651]$ & -0.01796 & 0.00286 & 0.01114 \\
Agg./post & 0.5 & -0.00861 & $[-0.01957,0.00234]$ & -0.00898 & 0.00037 & 0.00899 \\
Agg./post & 0.25 & -0.00502 & $[-0.01170,0.00165]$ & -0.00449 & -0.00053 & 0.00810 \\
Str./pre & 1.0 & -0.01780 & $[-0.03473,-0.00087]$ & -0.01875 & 0.00095 & 0.01145 \\
Str./pre & 0.5 & -0.00983 & $[-0.01837,-0.00129]$ & -0.00938 & -0.00045 & 0.00909 \\
Str./pre & 0.25 & -0.00406 & $[-0.00812,-0.00000]$ & -0.00469 & 0.00063 & 0.00841 \\
Str./post & 1.0 & -0.01607 & $[-0.03553,0.00340]$ & -0.01781 & 0.00175 & 0.01054 \\
Str./post & 0.5 & -0.00987 & $[-0.01860,-0.00114]$ & -0.00891 & -0.00096 & 0.00881 \\
Str./post & 0.25 & -0.00532 & $[-0.01086,0.00022]$ & -0.00445 & -0.00087 & 0.00829 \\
\end{longtable}
\endgroup

\begingroup
\small\setlength{\tabcolsep}{3pt}
\begin{longtable}{lrr}
\caption{Paired same-state contrasts of diagnostic-loss change. These supporting contrasts average five branch points within seed, then compare eight seeds. Negative values favor the first branch. The interaction is the unscaled difference-in-differences.}\label{tab:structure-local-contrasts}\\
\toprule
Contrast & Mean & Unadjusted 95\% seed CI\\
\midrule
\endfirsthead
\toprule
Contrast & Mean & Unadjusted 95\% seed CI\\
\midrule
\endhead
\bottomrule
\endfoot
Aggregate: post $-$ pre & -0.002779 & $[-0.014981,0.009423]$ \\
Interaction & 0.004515 & $[-0.001298,0.010328]$ \\
Post: strata $-$ aggregate & -0.000960 & $[-0.004695,0.002774]$ \\
Stratum-wise: post $-$ pre & 0.001736 & $[-0.008473,0.011945]$ \\
Pre: strata $-$ aggregate & -0.005475 & $[-0.009880,-0.001071]$ \\
\end{longtable}
\endgroup

Recorded repeat-forward diagnostic differences are zero at these branch states. Finite-step residuals still include gradient, displacement, and evaluation numerics and cannot be assigned entirely to curvature. The local pre stratum-minus-aggregate effect is $-.005475$ (CI $[-.009880,-.001071]$), whereas the local interaction is $+.004515$ ($[-.001298,.010328]$). This does not establish a one-step interaction explaining the long-trajectory generation interaction or a causal mediation fraction. Per-stratum branch losses and the scaled displacements remain available alongside the aggregate summaries.

\paragraph{Frozen forecasts.}
The score $C_t=\max(0,\max_k g_{k,t}^{\top}v_t^{\rm New})$ uses fresh protection gradients and a previewed, clipped New-only Adam candidate before replay or projection. Eight nonoverlapping 16-update windows begin at states $0,16,\ldots,112$. Within each seed, Spearman correlation compares the pre-update score with the following validation-NLL change along the actually executed training policy, including replay or projection where assigned, not a New-only counterfactual future. Baselines are the aggregate candidate dot, negative raw-gradient cosine, and current NLL on the same validation cohort. Constant quantities are non-estimable, not assigned zero. No threshold is fitted and the predictor does not control checkpoint selection.

\begingroup
\small\setlength{\tabcolsep}{3pt}
\begin{longtable}{llrrrr}
\caption{Complete prospective ranking summaries. Each entry is a mean across training seeds with an unadjusted 95\% $t$ interval; each seed contributes eight fixed 16-update forecast windows. $\Delta$ columns compare its Spearman correlation for $C$ with the named baseline on the same windows. Extraction is the frozen task holdout. These are descriptive associations with future validation NLL, not generated-behavior guarantees.}\label{tab:structure-predictors}\\
\toprule
Task & Method & $\rho(C)$ & $\Delta$Aggregate & $\Delta$Raw conflict & $\Delta$Current NLL\\
\midrule
\endfirsthead
\toprule
Task & Method & $\rho(C)$ & $\Delta$Aggregate & $\Delta$Raw conflict & $\Delta$Current NLL\\
\midrule
\endhead
\bottomrule
\endfoot
Code & Plain & \shortstack{$0.308$\\$[-0.099,0.715]$} & \shortstack{$-0.011$\\$[-0.172,0.151]$} & \shortstack{$0.058$\\$[-0.135,0.250]$} & \shortstack{$0.799$\\$[0.276,1.322]$} \\[3pt]
Code & Agg./pre & \shortstack{$-0.061$\\$[-0.492,0.370]$} & \shortstack{$-0.004$\\$[-0.275,0.267]$} & \shortstack{$0.118$\\$[-0.443,0.678]$} & \shortstack{$0.141$\\$[-0.307,0.589]$} \\[3pt]
Code & Agg./post & \shortstack{$0.107$\\$[-0.222,0.437]$} & \shortstack{$0.036$\\$[-0.184,0.255]$} & \shortstack{$0.202$\\$[-0.106,0.511]$} & \shortstack{$0.705$\\$[0.303,1.108]$} \\[3pt]
Code & Str./pre & \shortstack{$-0.146$\\$[-0.448,0.157]$} & \shortstack{$-0.792$\\$[-1.102,-0.481]$} & \shortstack{$-0.143$\\$[-0.565,0.279]$} & \shortstack{$0.583$\\$[0.346,0.821]$} \\[3pt]
Code & Str./post & \shortstack{$-0.223$\\$[-0.586,0.139]$} & \shortstack{$-0.018$\\$[-0.158,0.122]$} & \shortstack{$0.033$\\$[-0.417,0.482]$} & \shortstack{$0.137$\\$[-0.247,0.521]$} \\[3pt]
Code & Replay10 & \shortstack{$0.155$\\$[-0.170,0.479]$} & \shortstack{$-0.354$\\$[-0.609,-0.099]$} & \shortstack{$-0.131$\\$[-0.535,0.273]$} & \shortstack{$1.006$\\$[0.679,1.333]$} \\[3pt]
Natural & Plain & \shortstack{$0.321$\\$[-0.711,1.354]$} & \shortstack{$0.208$\\$[-0.363,0.780]$} & \shortstack{$0.054$\\$[-0.868,0.975]$} & \shortstack{$0.774$\\$[-0.737,2.284]$} \\[3pt]
Natural & Replay10 & \shortstack{$0.060$\\$[-0.238,0.357]$} & \shortstack{$-0.161$\\$[-0.355,0.033]$} & \shortstack{$0.179$\\$[-0.426,0.783]$} & \shortstack{$0.839$\\$[0.714,0.964]$} \\[3pt]
Extraction & Plain & \shortstack{$0.452$\\$[-0.512,1.417]$} & \shortstack{$0.054$\\$[-0.265,0.373]$} & \shortstack{$0.488$\\$[-0.851,1.827]$} & \shortstack{$0.982$\\$[-0.012,1.976]$} \\[3pt]
Extraction & Replay10 & \shortstack{$-0.381$\\$[-0.939,0.177]$} & \shortstack{$-0.298$\\$[-0.736,0.140]$} & \shortstack{$-0.690$\\$[-0.986,-0.395]$} & \shortstack{$0.446$\\$[-0.106,0.999]$} \\[3pt]
\end{longtable}
\endgroup

Extraction is the frozen task holdout. Its mean $C$ correlation is $+.4524$ for Plain and $-.3810$ for Replay10. Paired improvements over the aggregate-dot baseline have intervals $[-.2654,.3726]$ and $[-.7356,.1403]$, respectively. Replay's difference from the raw-conflict baseline is $-.6905$ ($[-.9855,-.3954]$). The prospective diagnostics do not establish a stable out-of-task ranking advantage for the stratum-risk score. The mechanism evidence therefore comes from the matched interventions rather than a validated deployment-time forecasting rule.

\subsection{Saved outputs and error taxonomy}
\label{app:structure-outputs}

The frozen classifier assigns a single priority-ordered label to each output (Table~\ref{tab:structure-errors}). The code-intrusion rule matches an initial A/B followed by whitespace, end of string, or declared punctuation; it can include lexical false positives such as an English article. A post-hoc, narrower count records outputs whose stripped text is exactly A or B. It yields 838/960 for Plain and 8/960 for Replay10. Counts describe repeated answers to 120 fixed items across eight training seeds, not 960 independent training experiments, and do not identify whether underlying knowledge survives.

\Needspace{17\baselineskip}
\begingroup
\small\setlength{\tabcolsep}{3pt}
\begin{longtable}{lrrrrrrrrr}
\caption{Complete frozen error taxonomy at Code endpoints on the two independent banks, pooled only as output counts (120 fixed items $\times$ eight seeds = 960 per method). Columns follow the priority-ordered classifier: strict correct, code intrusion, lexical refusal, extra continuation, correct tokens without EOS, content-correct format mismatch, incorrect at the length cap, and other lexical error. ``Format'' is the frozen content-correct/strict-incorrect category, not necessarily failure of the separate format scorer or missing EOS. ``Only A/B'' is an additional post-hoc stripped-string count within the intrusion column, not a new registered endpoint.}\label{tab:structure-errors}\\
\toprule
Method & Correct & Code & Refuse & Extra & No EOS & Format & Cap & Other & Only A/B\\
\midrule
\endfirsthead
\toprule
Method & Correct & Code & Refuse & Extra & No EOS & Format & Cap & Other & Only A/B\\
\midrule
\endhead
\bottomrule
\endfoot
Plain & 35 & 886 & 0 & 0 & 0 & 0 & 0 & 39 & 838 \\
Agg./pre & 524 & 348 & 0 & 0 & 0 & 0 & 0 & 88 & 348 \\
Agg./post & 35 & 879 & 0 & 0 & 0 & 0 & 3 & 43 & 864 \\
Str./pre & 743 & 188 & 0 & 0 & 0 & 0 & 0 & 29 & 150 \\
Str./post & 233 & 634 & 0 & 0 & 0 & 0 & 0 & 93 & 621 \\
Replay10 & 710 & 23 & 0 & 18 & 0 & 125 & 12 & 72 & 8 \\
\end{longtable}
\endgroup

The frozen content-correct/format-mismatch category includes strict token-response mismatches; it does not imply failure of the separate format scorer or absence of EOS. In Extraction bank A, all 42 content-correct but strict-incorrect replay records contain EOS, and 39 also pass the frozen format check. Bank B contains 43 such records, 41 with EOS and none passing the format check; content accuracy also decreases there (Table~\ref{tab:structure-bank-endpoints}). These observations distinguish token-exact response changes, missing termination, and scored content loss. The answer-only Old supervision and EOS-supervised New targets motivate a separate target ablation, but the present comparisons do not isolate that asymmetry's contribution.

\paragraph{Fixed-order examples.}
The following examples use seed 63001 and the first three records of qualified bank A in archived order (IDs ending 0002, 0001, 0000). Their full user prompts are, respectively:
\begin{quote}\small\ttfamily\raggedright
Use the value bc27d228d3 to fill this response pattern: VALUE=<value>. Return the filled pattern only.\par\medskip
Copy this identifier exactly, and output nothing else: ALPHA-77b9592a88\par\medskip
Compute 101 + 13. Reply with the integer alone; do not explain.
\end{quote}
Each is frozen as \texttt{<s>[INST] } followed by the displayed user text and \texttt{[/INST]}. References are \texttt{VALUE=bc27d228d3}, \texttt{ALPHA-77b9592a88}, and \texttt{114}.

\begin{table}[htbp]
\centering\small
\caption{Actual fixed-cap output paths for the fixed-order examples. ``='' denotes the exact reference text, not an omitted prediction. Leading whitespace is written as $\sqcup$. All entries naturally emit EOS except $U_{16}$. For every mismatch the first wrong token index is 0, except $\sqcup114$, whose index is 1; correct outputs have no wrong-token index. Indices are zero-based.}
\label{tab:structure-examples}
\begin{tabular}{llcccccc}
\toprule
ID suffix & Method & 0 & 16 & 32 & 64 & 96 & 128\\
\midrule
0002 & Plain & = & $U_{16}$ & $U_{32}$ & \texttt{B=B} & \texttt{B=B} & \texttt{B}\\
 & Replay10 & = & = & = & = & = & =\\
0001 & Plain & = & \texttt{A} & \texttt{B} & \texttt{B} & \texttt{B} & \texttt{B}\\
 & Replay10 & = & = & = & = & = & =\\
0000 & Plain & = & = & = & \texttt{B} & \texttt{B} & \texttt{B}\\
 & Replay10 & = & = & $\sqcup$\texttt{114} & = & $\sqcup$\texttt{114} & $\sqcup$\texttt{114}\\
\bottomrule
\end{tabular}
\end{table}

Here $U_{16}$ is the exact line sequence \texttt{B=B}, two lines \texttt{\ A=A}, and thirteen lines \texttt{\ B=B}, each terminated by a newline; generation reaches the 64-token cap without EOS. $U_{32}$ is \texttt{B=B}, \texttt{\ C=B}, \texttt{\ A=C}, \texttt{\ B=A}, \texttt{\ B=B}, joined by newlines without a trailing newline. The arithmetic replay example preserves the integer content while failing token-exact format at three observations. These saved outputs supplement the full error counts rather than defining a selected success rate.

\Needspace{20\baselineskip}
\subsection{Reproducible artifacts and archive scope}
\label{app:structure-artifacts}

The completed archive contains all 64 configurations/results, 8,192 update events, diagnostic and branch logs, and 148,400 generated-output records spanning banks, checkpoints, and full New audits. Each output retains prompt/reference, token IDs and text, stopping reason, first wrong token, metric labels, and identifiers. Four adapters per run are retained at updates 32/64/96/128, totaling 256 checkpoints; the final-adapter pointer refers to update 128 rather than a fifth copy.

The server's final report records 64 passing archive-byte audits, generation-coverage checks, and diagnostic-path checks. An independent local analysis of downloaded logs verifies the complete matrix and all 8,192 events, recomputes endpoint/factorial summaries, and reproduces all 256 predictor correlations. That original statistical check excludes adapter weights. The separate GPU recheck in Appendix~\ref{app:structure-score-audit} verifies the 13 selected final adapters' checkpoint bytes and exact equality of their loaded tensors, using the frozen loader, tokenizer, quantization configuration, and FP32 adapter storage. Its original validation-metric warmup leaves all 224 quantized modules' BF16 compute states unchanged; it is not a dtype intervention.

The recorded materials include frozen inputs/templates, construction and qualification code, model/tokenizer configurations, decoding/scoring code, protocol locks, seeds, checkpoint manifests, and statistical scripts. The added audit retains per-record rescoring, summary checks, $P_{64}$ sensitivity, the GPU cohort plan, actual padded inputs/masks and row positions, both batch rows' output tokens, stopping information, first-divergence processed scores, and environment/checkpoint verification receipts. Before/after checks match all 206 audited input-file hashes; all 39 scientific-source hashes also match their frozen lock. Original results, gates, and registered endpoints remain unchanged. These materials can support an anonymized supplement; no public release is asserted here. Their scope remains distinct from the original studies, whose missing intermediate outputs or deleted weights are not restored by this recheck.
\FloatBarrier

\section{Sequential Tasks and Protection Objects}
\label{app:continual-bridge}

This study connects protection of pre-existing functions to forgetting of tasks learned within an explicit sequence. Its central observation is task-dependent protection: replay improves retention of both preceding tasks relative to Plain, whereas gradient, parameter-importance, and LoRA-subspace interventions produce different patterns of retention. Mean improvement, numerical constraint satisfaction, and retention of each learned task are measured separately.

\subsection{Data, sequence, and common execution}
\label{app:continual-bridge-protocol}

One Mistral-7B-Instruct-v0.3 adapter learns \emph{AG News $\rightarrow$ natural-language SST-2 sentiment $\rightarrow$ structured JSON extraction}. Each task has 512 training, 32 development, and 64 fixed test records. AG News \citep{zhang2015character} uses \texttt{fancyzhx/ag\_news}, revision \texttt{eb185aade064a813bc0b7f42de02595523103ca4}: training/development are disjoint official-training subsets and evaluation uses official-test records. SST-2 \citep{socher2013recursive} reuses the frozen train/development/audit sources of Appendix~\ref{app:structure-protocol}; its selected test subset is not the official hidden test set. Extraction reuses that appendix's deterministic dispatch-schema task, with split-specific note templates and disjoint entities. Selection uses SHA256 ordering (data seed 64000), class balance for the two classification tasks, full-sequence length eligibility, and exact prompt/content/group deduplication; it does not use model scores or select examples for their observed forgetting.

The complete AG News prompt is:
\begin{quote}\small\ttfamily\raggedright
Classify the topic of the following news article. Return exactly one of: World, Sports, Business, Science and Technology.\par
Article: \{text\}\par
Answer:
\end{quote}
Its completion is a leading space followed by the category. Natural sentiment and extraction use the complete raw templates in Appendix~\ref{app:structure-protocol}: the former returns \texttt{This review expresses positive/negative sentiment.}; the latter maps recipient, cargo, and quantity to \texttt{rava}, \texttt{nemi}, and \texttt{tovu}. Prompt and completion are tokenized separately without a BOS token or chat-template wrapper. Every training example supervises its answer and an appended EOS; complete sequences must fit within 256 tokens, without answer truncation.

All methods use NF4 QLoRA with BF16 computation, FP32 adapter factors, rank 16, alpha 32, all-linear targets, and zero dropout. A single adapter and its AdamW state continue across tasks without reset, merging, or requantization. Each task has 32 updates, 16 new records per update, microbatch size two, learning rate $10^{-4}$, weight decay zero, and gradient-norm clipping at one. Thus each chain has 96 updates and 512 new-record exposures per task. Paired seeds 64101 and 64102 share new-record ordering and experienced-memory selection. Development scores do not select parameters or checkpoints; evaluation always reports the three completed boundaries.

\subsection{Protection objects and fixed adaptations}
\label{app:continual-bridge-methods}

\paragraph{Plain and supervised replay.}
Plain minimizes the record-mean answer-plus-EOS loss. Replay keeps that new-data contribution and adds old-record losses with the same denominator of 16. Each preceding task supplies 32 already-experienced training records. Round-robin replay adds 57 old exposures to 512 new exposures in each of tasks two and three, an old fraction of $57/569=10.0176\%$; the last stage allocates 29 to news and 28 to sentiment. The common new-data budget does not imply equal compute across methods.

\paragraph{Task-level GEM-style projection.}
Following episodic interference constraints , one fresh record-mean memory gradient $g_j$ is computed per preceding task. After clipping the new gradient $g$, the common exact small-cone solver computes
\[
v^*=\arg\min_v\tfrac12\|v+g\|^2
\quad\text{subject to }g_j^\top v\leq0\quad(j< i),
\]
and supplies $-v^*$ to AdamW. Stored FP32 vectors are checked after the FP64 Gram solve. This is a task-level QLoRA/AdamW adaptation of GEM, whose checked object is the pre-Adam descent gradient, not Adam's realized displacement.

\paragraph{Adapter EWC.}
The EWC reference  retains a separate factor-space anchor $\theta_j$ and diagonal empirical Fisher for each preceding task:
\[
F_{jp}=\frac1{32}\sum_{x\in M_j}
\bigl[\partial_{\theta_p}\ell(x;\theta_j)\bigr]^2,
\qquad
\Omega(\theta)=\frac{10}{2}\sum_{j,p}F_{jp}(\theta_p-\theta_{jp})^2.
\]
The estimator squares each observed-label, per-record answer-plus-EOS gradient before averaging. The exact penalty gradient is added once per update before clipping; the coefficient is fixed, not selected on test performance.

\paragraph{CLoRA and OPLoRA.}
For layer factors $A\in\mathbb R^{16\times d_{\rm in}}$ and $B\in\mathbb R^{d_{\rm out}\times16}$, CLoRA \citep{lu2025clora} adds
\[
\Omega_{\rm C}=\frac\lambda2\sum_{\ell}
\left(\|A_\ell V_\ell\|_F^2+\|U_\ell^\top B_\ell\|_F^2\right),
\qquad \lambda=1.
\]
The random orthonormal bases have 512 columns and remain fixed through all tasks. Analytic regularizer gradients are appended once after new-loss accumulation and before clipping. OPLoRA \citep{xiong2026oplora} instead uses the effective forward factors
\[
A_{\rm eff}=A-(AV)V^\top,\qquad
B_{\rm eff}=B-U(U^\top B),\qquad
\Delta W=\frac{32}{16}B_{\rm eff}A_{\rm eff}.
\]
Its 128-column bases come from a randomized truncated SVD of the \emph{initial, actually dequantized NF4 weights}, with 16 oversampling columns and four iterations. FP32 QR reorthogonalizes the selected column spaces; approximation residuals and orthogonality are logged. Both methods use basis seed 64200 independently of training RNG, with no task-boundary basis refresh. OPLoRA therefore protects selected initial-weight subspaces, not the historical-feature covariance null space used by Adam-NSCL. Its saved raw factors require the saved bases and projected forward path at inference.

\subsection{Boundary behavior and classical continual-learning metrics}
\label{app:continual-bridge-metrics}

At baseline and after every task, all three fixed test sets receive singleton greedy decoding with one common 64-new-token cap, ordinary EOS stopping, and no forced EOS. All scores derive from the \emph{same output}. Classification content exact compares the entire pre-EOS response after trimming and case folding; JSON content exact requires structural equality, exact value types, and no duplicate keys or nonfinite constants. Prefix exact compares reference token IDs; strict-EOS exact requires the entire output to equal the reference tokens followed by EOS. Answer NLL excludes EOS. These interfaces remain separate from the reference-length calls in the earlier studies.

Let $R_{ij}$ be content accuracy on task $j$ after learning task $i$, with $R_{0j}$ the baseline and $T=3$. We report
\begin{align*}
\mathrm{ACC}&=\frac1T\sum_{j=1}^T R_{Tj},\\
\mathrm{BWT}&=\frac1{T-1}\sum_{j=1}^{T-1}(R_{Tj}-R_{jj}),\\
\mathrm{Forgetting}&=\frac1{T-1}\sum_{j=1}^{T-1}
\left[\max_{i=j,\ldots,T-1}R_{ij}-R_{Tj}\right].
\end{align*}
Neither BWT nor forgetting uses the baseline row or pre-acquisition evaluations. Forgetting is not clipped to be nonnegative and need not equal $-\mathrm{BWT}$. Each method's own diagonal is reported alongside final scores: a BWT difference can include a difference in acquisition. All task diagonals exceed their respective initial scores, which are 75\%, 75\%, and 0\% in every chain.

\subsection{Complete results and task-dependent protection}
\label{app:continual-bridge-results}

Table~\ref{tab:continual-bridge} reports six-method means; Tables~\ref{tab:continual-seeds} and~\ref{tab:continual-matrices} retain both seeds and every task boundary. Plain news accuracy changes from 82.03\% just learned to 55.47\% finally; sentiment changes from 95.31\% to 78.12\%. Replay reduces these losses to 0.78 and 1.56 points. Its final old-task mean improves by 19.92 points over Plain, while final extraction accuracy is 99.22\% versus 98.44\%. The BWT improvement is 20.70 points; 0.78 points of that contrast come from differing mean acquisition references. Replay's 2.73-point forgetting differs from its $-1.17$-point BWT because news accuracy improves further at the intermediate boundary.

CLoRA and OPLoRA improve mean ACC by 3.12 and 3.91 points and mean BWT by 4.30 and 4.69 points over Plain. Their task decomposition is informative: CLoRA changes final news/sentiment by $-5.47/+13.28$ points; OPLoRA changes them by $+18.75/-7.81$ points. CLoRA's paired BWT gains are $+8.59$ and $0.00$ points; OPLoRA's are $-7.81$ and $+17.19$. Task-GEM improves mean BWT by 6.25 points, whereas fixed-coefficient adapter EWC leaves mean BWT unchanged. Thus positive mean improvements coexist with distinct task and seed trade-offs; none is excluded for failing an absolute retention target.

\begin{table}[t]
\centering\small\setlength{\tabcolsep}{4pt}
\caption{All twelve continual-task runs. ACC and New are percentages; BWT and forgetting ($F$) are percentage points. $F$ uses the maximum previously observed post-acquisition score, not simply $-\mathrm{BWT}$. Six-decimal display preserves each seed's reported resolution; exact task counts are given in Table~\ref{tab:continual-matrices}. No significance test is inferred from two seeds.}
\label{tab:continual-seeds}
\begin{tabular}{llrrrr}
\toprule
Method & Seed & ACC & BWT & $F$ & Final New\\
\midrule
Plain & 64101 & 81.770833 & -16.406250 & 16.406250 & 98.437500\\
Plain & 64102 & 72.916667 & -27.343750 & 27.343750 & 98.437500\\
Replay 10\% & 64101 & 91.145833 & -3.125000 & 3.125000 & 100.000000\\
Replay 10\% & 64102 & 90.625000 & 0.781250 & 2.343750 & 98.437500\\
Task-GEM & 64101 & 83.854167 & -14.062500 & 14.062500 & 100.000000\\
Task-GEM & 64102 & 80.208333 & -17.187500 & 17.187500 & 100.000000\\
Adapter-EWC & 64101 & 80.729167 & -17.187500 & 17.187500 & 98.437500\\
Adapter-EWC & 64102 & 73.437500 & -26.562500 & 26.562500 & 100.000000\\
CLoRA & 64101 & 88.020833 & -7.812500 & 7.812500 & 100.000000\\
CLoRA & 64102 & 72.916667 & -27.343750 & 27.343750 & 100.000000\\
OPLoRA & 64101 & 77.083333 & -24.218750 & 24.218750 & 98.437500\\
OPLoRA & 64102 & 85.416667 & -10.156250 & 10.156250 & 100.000000\\
\bottomrule
\end{tabular}
\end{table}

\begingroup\small\setlength{\tabcolsep}{5pt}
\begin{longtable}{llrrrr}
\caption{Complete task-by-boundary matrices for all twelve runs, expressed as correct outputs out of 64. Boundary 0 is the initial model; 1, 2, and 3 follow news, sentiment, and extraction training. Bold entries mark each task's just-learned diagonal. Pre-acquisition entries measure transfer, not retention. Content and strict-EOS counts coincide.}\label{tab:continual-matrices}\\
\toprule Method & Seed & Boundary & AG News & SST-2 & JSON\\\midrule\endfirsthead
\multicolumn{6}{l}{\tablename~\thetable{} (continued)}\\
\toprule Method & Seed & Boundary & AG News & SST-2 & JSON\\\midrule\endhead
\midrule\multicolumn{6}{r}{Continued on next page}\\\endfoot
\bottomrule\endlastfoot
Plain & 64101 & 0 & 48/64 & 48/64 & 0/64\\*
 &  & 1 & \textbf{54/64} & 36/64 & 0/64\\*
 &  & 2 & 49/64 & \textbf{61/64} & 0/64\\*
 &  & 3 & 51/64 & 43/64 & \textbf{63/64}\\
\addlinespace[3pt]
Plain & 64102 & 0 & 48/64 & 48/64 & 0/64\\*
 &  & 1 & \textbf{51/64} & 58/64 & 0/64\\*
 &  & 2 & 16/64 & \textbf{61/64} & 0/64\\*
 &  & 3 & 20/64 & 57/64 & \textbf{63/64}\\
\addlinespace[3pt]
Replay 10\% & 64101 & 0 & 48/64 & 48/64 & 0/64\\*
 &  & 1 & \textbf{54/64} & 36/64 & 0/64\\*
 &  & 2 & 53/64 & \textbf{61/64} & 0/64\\*
 &  & 3 & 52/64 & 59/64 & \textbf{64/64}\\
\addlinespace[3pt]
Replay 10\% & 64102 & 0 & 48/64 & 48/64 & 0/64\\*
 &  & 1 & \textbf{51/64} & 58/64 & 0/64\\*
 &  & 2 & 55/64 & \textbf{59/64} & 0/64\\*
 &  & 3 & 52/64 & 59/64 & \textbf{63/64}\\
\addlinespace[3pt]
Task-GEM & 64101 & 0 & 48/64 & 48/64 & 0/64\\*
 &  & 1 & \textbf{54/64} & 36/64 & 0/64\\*
 &  & 2 & 51/64 & \textbf{61/64} & 0/64\\*
 &  & 3 & 52/64 & 45/64 & \textbf{64/64}\\
\addlinespace[3pt]
Task-GEM & 64102 & 0 & 48/64 & 48/64 & 0/64\\*
 &  & 1 & \textbf{51/64} & 58/64 & 0/64\\*
 &  & 2 & 31/64 & \textbf{61/64} & 0/64\\*
 &  & 3 & 36/64 & 54/64 & \textbf{64/64}\\
\addlinespace[3pt]
Adapter-EWC & 64101 & 0 & 48/64 & 48/64 & 0/64\\*
 &  & 1 & \textbf{54/64} & 36/64 & 0/64\\*
 &  & 2 & 49/64 & \textbf{60/64} & 0/64\\*
 &  & 3 & 49/64 & 43/64 & \textbf{63/64}\\
\addlinespace[3pt]
Adapter-EWC & 64102 & 0 & 48/64 & 48/64 & 0/64\\*
 &  & 1 & \textbf{51/64} & 58/64 & 0/64\\*
 &  & 2 & 16/64 & \textbf{60/64} & 0/64\\*
 &  & 3 & 20/64 & 57/64 & \textbf{64/64}\\
\addlinespace[3pt]
CLoRA & 64101 & 0 & 48/64 & 48/64 & 0/64\\*
 &  & 1 & \textbf{55/64} & 56/64 & 0/64\\*
 &  & 2 & 50/64 & \textbf{60/64} & 0/64\\*
 &  & 3 & 48/64 & 57/64 & \textbf{64/64}\\
\addlinespace[3pt]
CLoRA & 64102 & 0 & 48/64 & 48/64 & 0/64\\*
 &  & 1 & \textbf{51/64} & 62/64 & 0/64\\*
 &  & 2 & 16/64 & \textbf{60/64} & 0/64\\*
 &  & 3 & 16/64 & 60/64 & \textbf{64/64}\\
\addlinespace[3pt]
OPLoRA & 64101 & 0 & 48/64 & 48/64 & 0/64\\*
 &  & 1 & \textbf{56/64} & 44/64 & 0/64\\*
 &  & 2 & 52/64 & \textbf{60/64} & 0/64\\*
 &  & 3 & 53/64 & 32/64 & \textbf{63/64}\\
\addlinespace[3pt]
OPLoRA & 64102 & 0 & 48/64 & 48/64 & 0/64\\*
 &  & 1 & \textbf{54/64} & 62/64 & 0/64\\*
 &  & 2 & 44/64 & \textbf{59/64} & 0/64\\*
 &  & 3 & 42/64 & 58/64 & \textbf{64/64}\\
\addlinespace[3pt]
\end{longtable}
\endgroup

\begin{figure}[htbp]
\centering
\includegraphics[width=\linewidth]{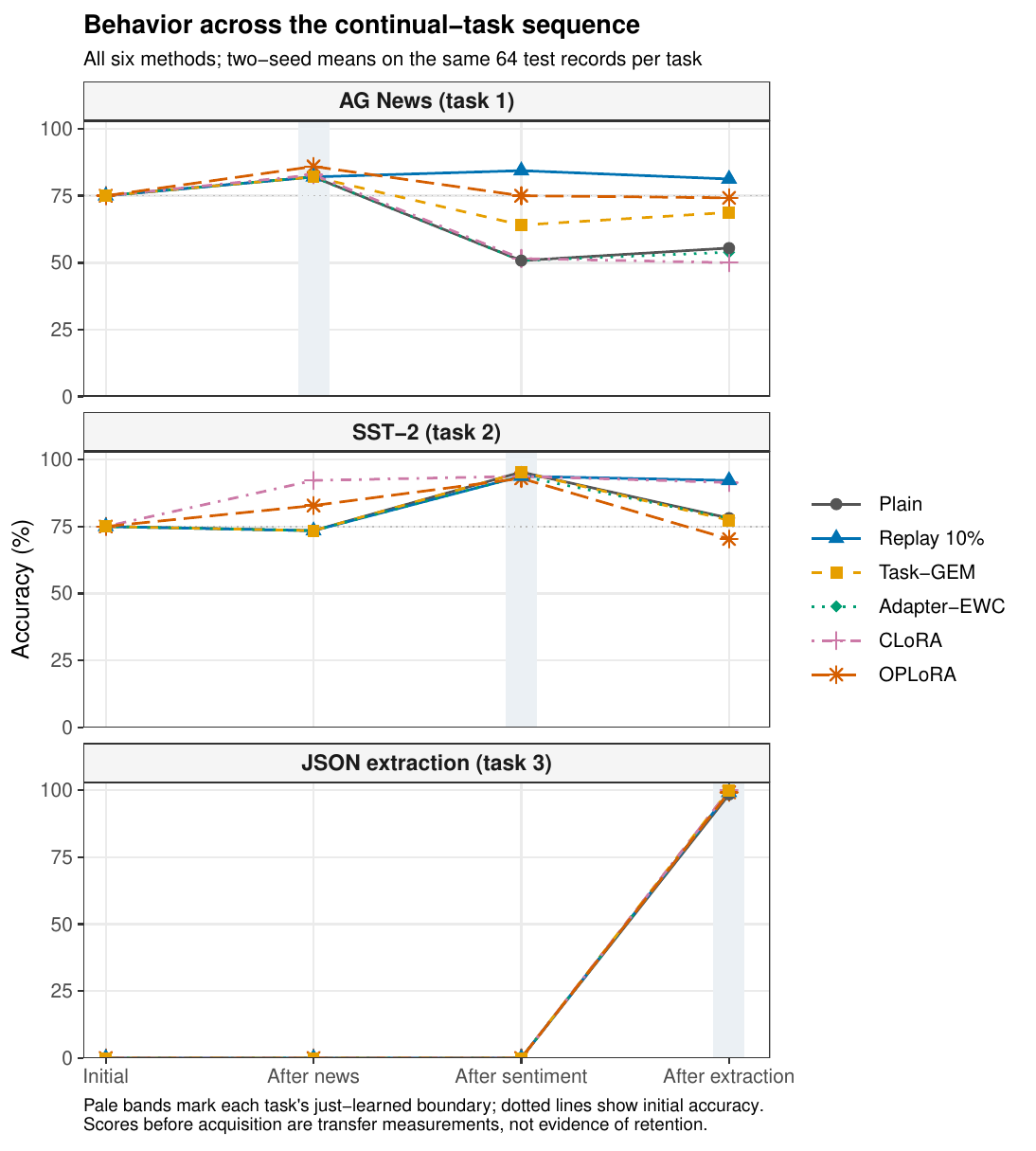}
\caption{Task-boundary generation in the three-task sequence. The same held-out records are evaluated before training and after each task; curves show two-seed means, without population confidence bands. Lines connect measured boundaries rather than unobserved within-task behavior. Individual-seed metrics appear in Table~\ref{tab:continual-seeds}.}
\label{fig:continual-bridge-paths}
\end{figure}

\begin{figure}[htbp]
\centering
\includegraphics[width=\linewidth]{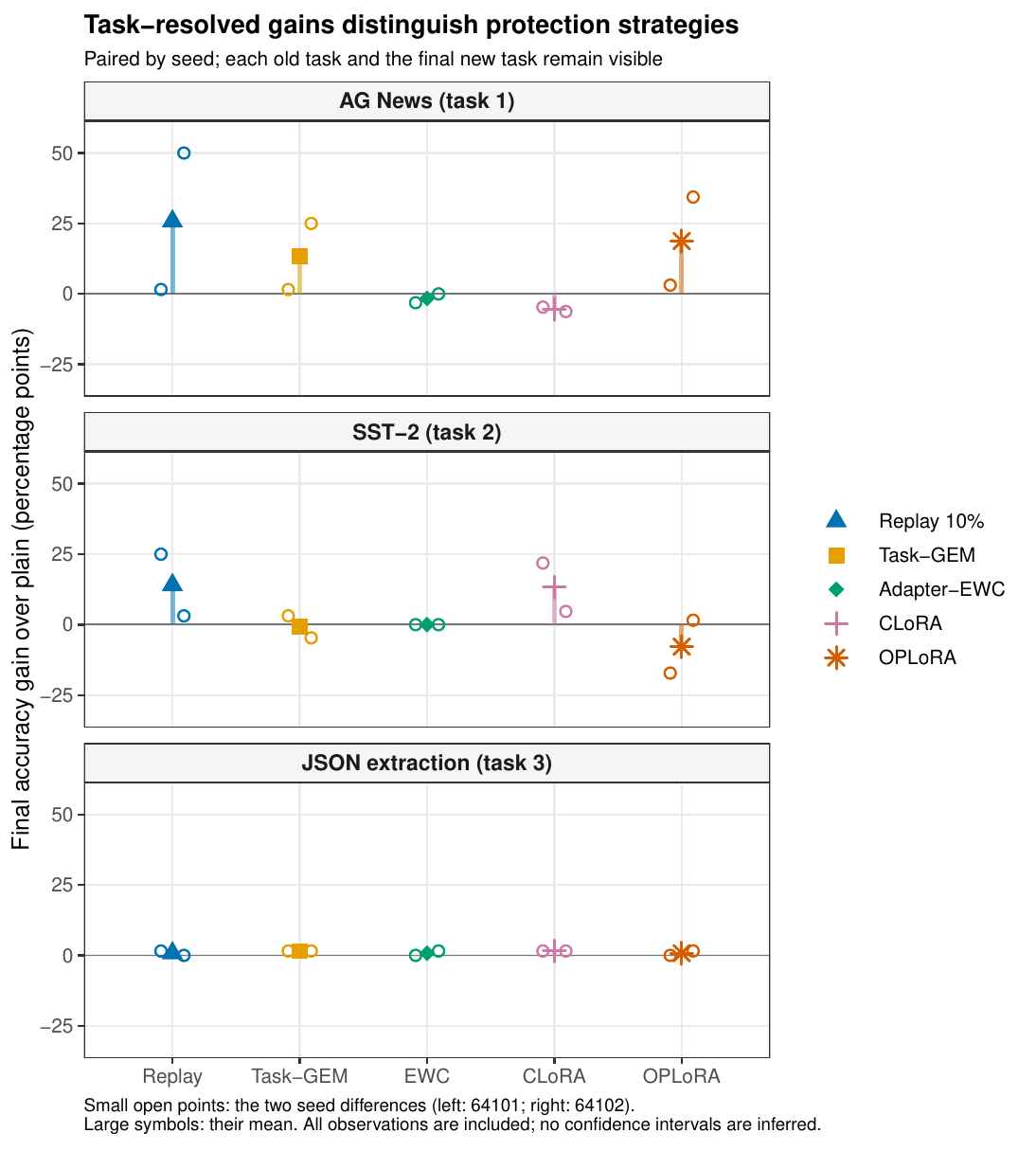}
\caption{Final accuracy gains over Plain on each task, including the last New task. Task-specific outcomes distinguish retention trade-offs that aggregate ACC and BWT can conceal. Small open points show paired seed differences; large symbols show their means. These descriptive comparisons include gains, ties, and regressions.}
\label{fig:continual-bridge-tradeoffs}
\end{figure}

\paragraph{Content loss, not an EOS-only explanation.}
Content and strict-EOS scores agree on all 9,216 archived generation records, and every one of the 2,304 final outputs contains EOS. A post-hoc inspection of Plain's 85 final old-task errors finds 83 valid but incorrect labels: 56 incorrect news categories and 27 incorrect sentiment polarities. The remaining outputs are \texttt{Travel} and a sentence using \texttt{informative}. Across the two seeds, Plain loses 61 previously correct old-task observations and gains five; replay loses 14 and gains 11. These are seed--item observations, not independent training replicates. They establish loss of learned task answers beyond a failure to stop or match punctuation.

\subsection{Numerical checks, archived controls, and interpretation}
\label{app:continual-bridge-audit}

All 128 task-GEM solves with preceding tasks converge; 73 have nonempty active sets, and all 192 stored constraints pass tolerance (maximum positive dot $2.28\times10^{-8}$). EWC penalties are nonzero. CLoRA applies nonzero regularization at all 96 updates in each chain. Across OPLoRA's 224 layers, maximum QR-basis orthogonality error is $7.15\times10^{-7}$; the largest final protected-leakage energy ratio is $1.91\times10^{-14}$ for FP32 effective factors and $3.72\times10^{-7}$ after factor rounding to BF16. The latter is a factor-rounding diagnostic, not an end-to-end kernel error bound. These checks establish that the declared protections execute; they do not equate a protected weight subspace with all subsequently acquired task functions.

The eight Plain/Replay/GEM/EWC chains are explicitly archived controls; the four CLoRA/OPLoRA chains are the extension, not concurrent reruns. Common input, configuration, execution-environment, and exposure checks pass. Zero-strength tests recover the Plain path. All 768 new-chain baseline records match their corresponding archived Plain records in generated tokens, scores, and NLL. Independent re-scoring checks 36,864 Boolean fields and reproduces the matrices and metrics. Server-side completion checks validate saved adapters, basis references, and artifact hashes; mechanical save/reload tests reproduce outputs and NLL within tolerance. The lightweight local archive does not independently rerun large-tensor inference or SVD. Historical model-weight hashes were not recorded and are not retroactively inferred from baseline agreement.

This is a descriptive bridge with one model, one fixed order, two paired seeds, and fixed QLoRA adaptations without method-specific hyperparameter search. Its empirical contribution is the conjunction of acquired-task forgetting, functional replay benefits, and task-dependent subspace protection under a shared sequence, not a universal method ranking, a new-task noninferiority test, or an equal-FLOPs comparison. All 12 complete chains and all boundary evaluations enter the results; mechanical completion and scientific improvement remain distinct.

\end{document}